\documentclass[11 pt]{article}
\usepackage{style}

\usepackage{cleveref}

\newcommand{\bP}[2][]{\Pr\ifthenelse{\isempty{#1}}{}{_{#1}}\left[#2\right]}
\newcommand{\bE}[2][]{\mathop\mathbb{E}\ifthenelse{\isempty{#1}}{}{_{#1}}\left[#2\right]}
\newcommand{\bI}[2][]{\mathop\mathbb{I}\ifthenelse{\isempty{#1}}{}{_{#1}}\left[#2\right]}
\newcommand{\Var}[2][]{\mathbf{Var}\ifthenelse{\isempty{#1}}{}{_{#1}}\left[#2\right]}

\DeclareMathOperator*{\argmax}{arg\,max}
\DeclareMathOperator*{\argmin}{arg\,min}
\title{\LARGE\bfseries A Non-Asymptotic Theory of Seminorm Lyapunov Stability: From Deterministic to Stochastic Iterative Algorithms}
\author{
	Zaiwei Chen\textsuperscript{$1$}, 
	Sheng Zhang\textsuperscript{$2$}, 
	Zhe Zhang\textsuperscript{$1$}, 
	Shaan Ul Haque\textsuperscript{$3$}, and
	Siva Theja Maguluri\textsuperscript{$3$}\\
	{\small\textsuperscript{$1$}\textit{Purdue IE.} \href{mailto:chen5252@purdue.edu}{\textit{chen5252@purdue.edu}}, \href{mailto:zhan5111@purdue.edu}{\textit{zhan5111@purdue.edu}}}\\
	{\small\textsuperscript{$2$}\textit{Amazon Web Services, Inc.,} \href{mailto:szhang405@outlook.com}{\textit{szhang405@outlook.com}}}\\
	{\small\textsuperscript{$3$}\textit{Georgia Tech ISyE.} \href{mailto:shaque49@gatech.edu}{\textit{shaque49@gatech.edu}}, \href{mailto:siva.theja@gatech.edu}{\textit{siva.theja@gatech.edu}}}
}
\date{\vspace{-0.4 in}}
\begin{document}
	\maketitle
	
	\begin{abstract}
		We study the problem of solving fixed-point equations for seminorm-contractive operators and establish foundational results on the non-asymptotic behavior of iterative algorithms in both deterministic and stochastic settings. Specifically, in the deterministic setting, we prove a fixed-point theorem for seminorm-contractive operators, showing that iterates converge geometrically to the kernel of the seminorm. In the stochastic setting, we analyze the corresponding stochastic approximation (SA) algorithm under seminorm-contractive operators and Markovian noise, providing a finite-sample analysis for various stepsize choices. 
		
		A benchmark for equation solving is linear systems of equations, where the convergence behavior of fixed-point iteration is closely tied to the stability of linear dynamical systems. In this special case, our results provide a complete characterization of system stability with respect to a seminorm, linking it to the solution of a Lyapunov equation in terms of positive semi-definite matrices. In the stochastic setting, we establish a finite-sample analysis for linear Markovian SA without requiring the Hurwitzness assumption.  
		
		Our theoretical results offer a unified framework for deriving finite-sample bounds for various reinforcement learning algorithms in the average reward setting, including TD($\lambda$) for policy evaluation (which is a special case of solving a Poisson equation) and Q-learning for control.
	\end{abstract}

	\section{Introduction}\label{sec:intro} 
	Fixed-point equations are fundamental in diverse fields such as optimization \cite{boyd2004convex}, game theory \cite{fudenberg1991game}, dynamical systems \cite{khalil2002nonlinear}, and reinforcement learning (RL) \cite{sutton2018reinforcement}, where solving the Bellman equation often lies at the core of the problem. When the operator defining the fixed-point equation is a norm-contractive mapping, the development of algorithms and the analysis of their convergence are relatively well understood. This understanding extends to both deterministic settings—leveraging the Banach fixed-point theorem \cite{banach1922operations}—and stochastic settings, through stochastic approximation methods designed for norm-contractive operators \cite{robbins1951stochastic,borkar2009stochastic,bertsekas1996neuro}.

	While norm-contractive operators provide a well-established framework for analyzing fixed-point equations, many real-world problems involve operators that are only seminorm-contractive. A representative example is RL in the average reward setting \cite{mahadevan1996average,puterman2014markov}, where the absence of a discount factor makes the associated Bellman operator contractive only with respect to a seminorm. Another common example is solving least squares problems involving positive semi-definite but not positive definite matrices, which arise in applications such as over-parameterized models \cite{li2018algorithmic} and low-rank matrix completion \cite{shah2020sample}. Even in cases where the operator is norm-contractive, solving for the exact fixed point may not always be necessary; instead, it may suffice to find an approximate fixed point with an error confined to a suitable subspace. For instance, in solving RL problems in the discounted setting, the Bellman operator is a norm-contractive mapping. However, by leveraging its seminorm-contractive property, one can achieve significantly better sample complexity guarantees \cite{devraj2021q}. Therefore, developing provably efficient algorithms for solving seminorm fixed-point equations, in both deterministic and stochastic settings, is of significant theoretical and practical importance.

	In this paper, we consider solving seminorm fixed-point equations of the form $ p(\bar{F}(x) - x) = 0 $, where $ p(\cdot) $ is an arbitrary seminorm (see Section \ref{subsec: seminorm} for the formal definition) and $ \bar{F}: \mathbb{R}^d \to \mathbb{R}^d $ is an operator that is contractive with respect to $ p(\cdot) $. In the special case where $ p(\cdot) $ is a norm, this equation reduces to the standard fixed-point equation $ \bar{F}(x) = x $. A natural approach to solve $ p(\bar{F}(x) - x) = 0 $ is through the fixed-point iteration
	\begin{align}\label{eq:FPI_intro}
		x_{k+1} = \Bar{F}(x_k).
	\end{align} 
	However, unlike the norm-contractive setting \cite{banach1922operations}, the convergence properties of such fixed-point iterations are, to the best of our knowledge, not completely understood in the literature. To make matters worse, we often lack sufficient information and/or computational power to perform the desired fixed-point iteration and must instead work with its noisy variant:
	\begin{align}\label{saeq:intro}
		x_{k+1} = x_k + \alpha_k (F(x_k, Y_k) - x_k + w_k),
	\end{align}
	where $ F(\cdot, \cdot) $ is another operator, $ \alpha_k $ is the stepsize, and $ \{Y_k\} $ and $ \{\omega_k\} $ are two stochastic processes representing noise. Here, $F(\cdot,Y_k)+w_k$ is understood as a noisy estimate of $\bar{F}(\cdot)$. Motivated by applications in average reward RL, the stochastic process $\{Y_k\}$ is a Markov chain, and the stochastic process $\{w_k\}$ is a martingale difference sequence. See Section \ref{sec:SA} for more details.  Algorithms of the form (\ref{saeq:intro}) are known as Markovian stochastic approximation (SA) algorithms, which are central to modern large-scale machine learning. As we will see in Section \ref{sec:average-reward-RL}, the update equation (\ref{saeq:intro}) captures a variety of average reward RL algorithms, such as TD($\lambda$) for policy evaluation and Q-learning for policy optimization.
	
	In this work, we present a unified theory for solving fixed-point equations with seminorm-contractive operators, starting with a fixed-point theorem that characterizes the convergence behavior of the fixed-point iteration (\ref{eq:FPI_intro}) and eventually leading to the finite-sample analysis of the SA algorithm (\ref{saeq:intro}). Moreover, in the special case where the target equation is linear, we extend the classical Lyapunov stability theorem to characterize the stability of the fixed-point iteration (which can also be viewed as a linear dynamical system) with respect to seminorms. This result is later used to analyze Markovian linear SA without requiring the Hurwitzness assumption. The non-asymptotic results in this paper are essential for evaluating algorithmic efficiency under real-world computational and time constraints. Unlike asymptotic analysis, which guarantees convergence only in the limit, non-asymptotic analysis provides explicit performance bounds that guide algorithm design and implementation.

	\subsection{Main Contributions}
	The main contributions of this paper can be summarized in three key points.
	
	\begin{itemize}
		\item \textbf{Seminorm Fixed-Point Theorem.} 
		We start by analyzing the convergence behavior of fixed-point iterations for operators that are contractive with respect to seminorms. By establishing fundamental properties of seminorms, we demonstrate that the quotient space induced by a seminorm is a Banach space and that a seminorm-contractive operator becomes norm-contractive when restricted to this quotient space. Leveraging these insights and the Banach fixed-point theorem, we prove that the fixed-point iteration for a seminorm-contractive operator converges geometrically to the kernel of the seminorm.

		A notable application of our results is the iteration scheme $ x_{k+1} = Ax_k $, where the matrix $ A $ may have a spectral radius greater than one. Understanding the convergence behavior of $ x_k $ in such discrete linear dynamical systems is crucial due to the widespread applications of linear system theory. Building on the seminorm fixed-point theorem, we establish a Lyapunov stability theorem for these linear systems, linking their stability (defined with respect to an appropriate seminorm) to the solutions of associated Lyapunov equations. For completeness, we also provide a characterization of the stability of the continuous-time linear dynamical system $ \dot{x}(t) = Ax(t) $ by the solution to its corresponding Lyapunov equation.

		\item \textbf{Markovian Stochastic Approximation.} 
		The most notable contribution of this work is the establishment of finite-sample bounds for the Markovian SA algorithm presented in Eq. (\ref{saeq:intro}), measured by the expectation of the square of the seminorm. Specifically, we show that the expected squared error decays at a rate of $\tilde{\mathcal{O}}(1/k)$ with properly chosen diminishing stepsizes, and converges geometrically fast to an asymptotic error of $\tilde{\mathcal{O}}(\alpha)$ when constant stepsizes $\alpha_k \equiv \alpha$ are used.
		
		A significant consequence of these results is the application to Markovian linear SA algorithms of the form $x_{k+1} = x_k + \alpha_k(A(Y_k)x_k + b(Y_k))$, where $A(\cdot)$ and $b(\cdot)$ are matrix-valued and vector-valued functions, respectively, and $\{Y_k\}$ is a Markov chain. Crucially, our framework does not require the expectation of $A(Y_k)$ to be Hurwitz, a common assumption in the existing literature. Despite this relaxation, the Markovian linear SA algorithm still converges, albeit in a properly defined seminorm.
		
		Methodologically, our analysis of Markovian SA under seminorm-contractive mappings is built on a Lyapunov-based approach. Inspired by \cite{chen2021lyapunov}, which studied the norm-contractive setting of the SA algorithm in Eq. (\ref{saeq:intro}), we develop a novel Lyapunov function tailored for seminorm-contractive operators. This function is constructed using the infimal convolution of an indicator function with the generalized Moreau envelope. Further details on the proof techniques can be found in Section \ref{subsec:proofsketch}.

		\item \textbf{Average Reward Reinforcement Learning.} We model popular average reward RL algorithms, such as TD($\lambda$) with linear function approximation and Q-learning, in the form of Eq. (\ref{saeq:intro}) and provide their finite-sample guarantees. These results imply a sample complexity of $\Tilde{\mathcal{O}}(\epsilon^{-2})$ for both algorithms. The average reward setting is notably more challenging than the discounted setting due to the Bellman operator being non-contractive with respect to any norm. Our analysis of SA with seminorm-contractive operators plays a key role in establishing these guarantees. 
		
		As a side note, the Bellman equation for policy evaluation in average reward RL takes the form of a Poisson equation. Therefore, our results can potentially be applied to finding solutions to general Poisson equations (which have wide applications) in both deterministic and stochastic settings.
	\end{itemize}

	\subsection{Related Literature}
	\label{subsec:literature}
	In this section, we discuss related work on seminorm contractive operators, Lyapunov stability of dynamical systems, SA, and average reward RL (focusing on value-based algorithms).
	
	\textbf{Seminorm Contractive Operators.} Norm-contractive operators have been extensively studied in the literature since the introduction of the contraction principle in \cite{banach1922operations}. They have found numerous applications, including the study of dynamical systems \cite{lohmiller1998contraction}, control system design \cite{Lohmiller2000}, and robotics \cite{Manchester2018}. These applications have motivated a series of works examining contractive operators in various settings \cite{pham2009contraction, tsukamoto2021contraction, Bullo2022}. Seminorms have been a key area of study in functional analysis \cite{conway2019course} and topological vector spaces \cite{bourbaki2013topological}. However, to the best of our knowledge, the literature lacks a comprehensive theory of seminorm-contractive operators. Notable exceptions that focus on special cases within this paradigm include the study of span seminorm contraction for the average reward Bellman operator in MDPs \cite{puterman2014markov}, the ergodic coefficient of Markov chains \cite{Bullo2023}, and network systems analysis as a seminorm contraction \cite{jafarpour2022}. The objective of this paper is to establish a unifying framework for the study of seminorm-contractive operators in full generality.

	\textbf{Lyapunov Stability.} Lyapunov stability has long been a central focus in the control theory of dynamical systems \cite{khalil2002nonlinear, khalil2009lyapunov}. One especially well-studied area within this field is the stability analysis of linear systems, which finds applications in a variety of settings and often serves as an effective approximation for more complex nonlinear systems. Specifically, the stability of linear systems has been shown to be closely linked to the existence of a unique solution to the associated Lyapunov equation \cite{lyapunov, bartels1972algorithm}. In discrete- and continuous-time settings, this translates respectively to Schur stability and Hurwitz stability of the corresponding linear operators \cite{hinrichsen2005mathematical}. However, to the best of our knowledge, the study of linear system stability within the context of seminorms, along with the corresponding formulation of Lyapunov equations, has not yet been addressed in the existing literature.
	
	\textbf{Stochastic Approximation.} SA has been a powerful tool for solving root-finding problems under noisy observations. The early literature on SA focused on establishing its asymptotic convergence \cite{robbins1951stochastic,kushner2003stochastic,borkar2009stochastic,benveniste2012adaptive}, while the more recent ones have shared the interest in providing the finite-sample guarantees. Specifically, for linear SA, finite-sample analysis has been performed in \cite{srikant2019finite,mou2020linear,lakshminarayanan2018linear}. For nonlinear SA, finite-sample bounds were provided when there is a contractive operator \cite{chen2019finite,chen2021lyapunov,qu2020finite,mou2022optimal}, or when the algorithm is an SGD variant for minimizing some objective function \cite{bottou2010large,lan2020first,beck2017first}. Most existing studies on finite-time bounds assume that the algorithm converges to a unique point. However, a broader perspective includes situations in which the SA algorithm can converge to any point in a subspace or even wander within that subspace without exhibiting clear convergence behavior. In this work, we focus on two such scenarios—namely, SA for seminorm contractive operators and SA for linear operators—and provide non-asymptotic results.

	\textbf{Value-Based Methods in Average Reward RL.} 
	In the next few paragraphs, we discuss related work on average reward RL, focusing on last-iterate convergence guarantees of value-based algorithms such as TD-learning for policy evaluation and Q-learning for policy optimization.
	
	The policy evaluation problem in RL is usually solved using TD-learning variants, such as least-squares policy evaluation (LSPE), $ n $-step TD, and TD($\lambda$). The asymptotic convergence of TD($\lambda$) (with linear function approximation) was established in \cite{tsitsiklis1999average}. The asymptotic convergence and the rate of convergence of LSPE were provided in \cite{yu2009convergence}. However, both results require that the all-ones vector does not belong to the span of the basis vectors used for linear function approximation, which is relatively restrictive (as it is not even satisfied in the tabular setting) and is not required in this work. Moreover, we focus on finite-sample guarantees rather than asymptotic convergence, thereby offering a more refined characterization of the algorithm’s convergence behavior. 
	
	Q-learning is one of the most well-known model-free algorithms for policy optimization. The first provably convergent algorithms include relative value iteration (RVI) Q-learning, stochastic shortest path (SSP) Q-learning \cite{abounadi2001learning}, and their variants \cite{gosavi2004reinforcement}. More recently, the authors of \cite{wan2020learning} introduced a variant of Q-learning that does not require a reference function (which is needed in RVI Q-learning). However, most existing results focus on establishing asymptotic convergence, whereas we provide non-asymptotic guarantees.
	
	A conference version of this paper \cite{zhang2021finite} specifically focused on average reward RL. However, in this work, we develop a general theory for seminorm contraction SA (and linear SA without the Hurwitzness assumption) and obtain results for average reward TD-learning and Q-learning as special cases.
	
	\textbf{Solving the Poisson Equation.} The Poisson equation characterizes the long-term behavior of functionals of a Markov chain \cite{meyn2012markov}. In the context of the policy evaluation problem in RL, since the policy is fixed, the Poisson equation naturally arises as the average reward Bellman equation. It has been extensively studied in Markov chain theory, where estimating its solutions serves as a subroutine in various problems, including asymptotic variance estimation \cite{agrawal2024markov} and the construction of variance-reduced unbiased estimators for the stationary expectation of a function of the Markov chain \cite{jure2018, douc2024solvingpoissonequationusing, henderson2002approximating}. In \cite{agrawal2024markov}, the authors establish finite-time mean square error bounds for the value function estimate, while \cite{jure2018, douc2024solvingpoissonequationusing, henderson2002approximating} focus on the asymptotic optimality of the proposed unbiased estimators of the stationary expectation. The results presented in this work can be leveraged to obtain finite-time bounds for iterative methods used to solve the Poisson equation.
	
	\section{Preliminary Results on the Seminorm Fixed-Point Theorem}
	\label{sec:seminorm contractive operator}
	Seminorms, as generalizations of the concept of norms, have been utilized in various areas of research, including Lasso regression in convex optimization \cite{boyd2004convex} and average reward RL \cite{puterman2014markov}. In this section, we develop preliminary results for seminorm-contractive operators, starting with several useful properties of seminorms and ultimately leading to a seminorm fixed-point theorem.

	\subsection{Seminorms}
	\label{subsec: seminorm}
	Recall that a function $\|\cdot\| : \mathbb{R}^d \to \mathbb{R}$ is a norm if it satisfies the following three properties:
	\begin{enumerate}[(1)]
		\item Triangle Inequality: $\|x + y\| \leq \|x\| + \|y\|$ for any $x,y\in\mathbb{R}^d$; 
		\item Absolute Homogeneity: $\|\alpha x\| = |\alpha| \|x\|$ for all $x\in\mathbb{R}^d$ and $\alpha\in\mathbb{R}$; 
		\item Positive Definiteness: $\|x\| \geq 0$ for all $x\in\mathbb{R}$, with $\|x\| = 0$ only if $x = 0$.
	\end{enumerate}
	Although norms are useful metrics to measure the size of vectors, sometimes we care only about certain components of a vector instead of the whole vector, which motivates the definition of seminorms.
	
	\begin{definition}
		\label{def:seminorm}
		A non-negative real-valued function $p: \mathbb{R}^d  \to \mathbb{R}$ is called a \textit{seminorm} if 
		\begin{enumerate}[(1)]
			\item  $p(\alpha x) = | \alpha | p(x)$ for all $x \in \mathbb{R}^d$ and $\alpha \in \mathbb{R}$,
			\item $p(x+y) \leq p(x) + p(y)$ for all $x, y \in \mathbb{R}^d$.
		\end{enumerate} 
	\end{definition}
	
	A seminorm can be viewed as a relaxation of a norm because the seminorm does not require the positive definiteness property, i.e., $p(x) = 0$ does not imply $x = 0$. A representative example of seminorms is the span seminorm 
	\begin{align*}
		p_\text{span}(x)= \max_{1\leq i\leq d} x_i - \min_{1\leq j\leq d} x_j,\quad \forall\,x \in \mathbb{R}^d,
	\end{align*}
	which is used in studying average reward MDPs \cite{puterman2014markov}. As another example, the seminorm $p(x)=(x^\top Px)^{1/2}$ defined by a positively semi-definite matrix $P\in\mathbb{R}^{d\times d}$ often appears as the objective function in linear regression problems.

	\vspace{3 mm}
	\noindent\textbf{Properties of Seminorms.} The following proposition summarizes several fundamental properties of seminorms, all of which will be frequently used in our study. See Appendix \ref{pf:prop:seminorm_properties} for the proof.
	
	\begin{proposition}\label{prop:seminorm_properties}
		Let $p(\cdot)$ be a seminorm on $\mathbb{R}^d$.
		\begin{enumerate}[(1)]
			\item The kernel of $p(\cdot)$, defined as $\text{ker}(p) = \{x \in \mathbb{R}^d \mid p(x)=0\}$, is a linear subspace of $\mathbb{R}^d$.
			\item There exists a norm $\| \cdot \|$ on $\mathbb{R}^d$ such that $p(x) = \min_{y \in \text{ker}(p)} \| x - y \|$ for all $x \in \mathbb{R}^d$.
			\item Let $q(\cdot)$ be another seminorm on $\mathbb{R}^d$ such that $\text{ker}(p)=\text{ker}(q)$. Then, there exist constants $C_1,C_2>0$ such that $C_1 q(x) \leq p(x)  \leq C_2 q(x)$ for all $x \in \mathbb{R}^d$.
		\end{enumerate}
	\end{proposition}
	
	Several remarks are in order. First, given a seminorm $p(\cdot)$, the set $\{x\in \mathbb{R}^d \mid p(x)=0\}$ is, in general, not a singleton. For example, consider the seminorm $p(x)=(x^\top Px)^{1/2}$, where $P\in\mathbb{R}^{d\times d}$ is a positive semi-definite matrix. Then, it is easy to see that the kernel space of $p(\cdot)$ is the kernel space of the matrix $P$, which is a linear subspace of $\mathbb{R}^d$. Proposition \ref{prop:seminorm_properties} (1) states that this is true for any seminorm.
	Second, there is a close connection between a seminorm and a norm on $\mathbb{R}^d$. Suppose that we are given a norm $\| \cdot \|$ and a linear subspace $E$, then $p(x):=\min_{y \in E} \| x - y \|$ is, by Definition \ref{def:seminorm}, a seminorm.
	Proposition \ref{prop:seminorm_properties} (2) states that the converse is also true, i.e., any seminorm $p(\cdot)$ can be equivalently written as the distance to its kernel with respect to some norm $\| \cdot \|$. Finally, given two arbitrary norms, denoted as $\|\cdot\|_a$ and $\|\cdot\|_b$, defined in a finite-dimensional Euclidean space, it is well known that they are equivalent in the sense that there exist $c_1,c_2>0$ such that $c_1\|x\|_a\leq \|x\|_b\leq c_2\|x\|_a$ for any $x\in\mathbb{R}^d$. This is, in general, not true for two arbitrary seminorms. However, Proposition \ref{prop:seminorm_properties} (3) states that we still have equivalence between seminorms as long as their kernels coincide.
	
	To further understand Proposition \ref{prop:seminorm_properties}, let us use the span seminorm as an illustrative example, which is of particular interest to us due to its relevance in average reward RL. 
	Recall that the span seminorm $p_\text{span}(\cdot)$ is defined as $p_\text{span}(x)= \max_{1\leq i\leq d} x_i - \min_{1\leq j\leq d} x_j$ for all $x \in \mathbb{R}^d$. In this case, it is easy to see that the kernel of $p_\text{span}(\cdot)$ is $\{c  e_d \mid c \in \mathbb{R}\}$, where $e_d$ is the all-ones vector in $\mathbb{R}^d$. In addition, we have $p_{\text{span}}(x) = \min_{y \in \{c  e_d \mid c \in \mathbb{R}\}} 2\| x - y\|_\infty$ for all $x \in \mathbb{R}^d$,
	which follows by observing that $\arg\min_{y \in \{c \cdot e_d \mid c \in \mathbb{R}\}} \|x - y\|_\infty=(\max_{1\leq i\leq d} x_i + \min_{1\leq j\leq d} x_j)e_d/2$.
	
	\vspace{3 mm}
	\noindent\textbf{The Quotient Space Defined by a Seminorm.} 
	Next, we introduce the quotient space associated with a seminorm and show that seminorms are actually norms when restricted to their corresponding quotient spaces. This result will be particularly useful in establishing the seminorm fixed-point theorem.
	
	For any $x, y \in \mathbb{R}^d$ with $x - y \in \text{ker}(p)$, we have $p(x) = p(y)$. This motivates the definition of the following equivalence relation $\sim$ on $\mathbb{R}^d$: $x \sim y \text{ if and only if } x - y \in \text{ker}(p)$. Given $x \in \mathbb{R}^d$, the equivalence class of $x$ under $\sim$, denoted $[x]$, is defined as $[x] = x + \text{ker}(p)$. With this notation, we define the quotient of $\mathbb{R}^d$ by $\text{ker}(p)$ as $\mathbb{R}^d / \text{ker}(p) = \{[x] \mid x \in \mathbb{R}^d\}$, which is the set of all equivalence classes induced by $\sim$ on $\mathbb{R}^d$. Alternatively, the quotient space $\mathbb{R}^d / \text{ker}(p)$ can be interpreted as the set of all affine subspaces of $\mathbb{R}^d$ that are parallel to $\text{ker}(p)$. Scalar multiplication and addition are defined on the equivalence classes by $\alpha [x] = [\alpha x]$ and $[x] + [y] = [x + y]$ for all $\alpha \in \mathbb{R}$ and $x, y \in \mathbb{R}^d$. These operations are well-defined (i.e., they do not depend on the choice of representatives), making the quotient space $\mathbb{R}^d / \text{ker}(p)$ a vector space with $\text{ker}(p)$ being its zero vector.
	
	We conclude this section with the following result, which states that the vector space $\mathbb{R}^d / \text{ker}(p)$ equipped with $p(\cdot)$ is actually a Banach space.  See Appendix \ref{ap:quotient_space} for the proof.
	
	\begin{lemma}
		\label{lemma:p is norm on quotient space}
		Let $p(\cdot)$ be a seminorm on $\mathbb{R}^d$. Then, $p(\cdot)$ is a norm on $\mathbb{R}^d / \text{ker}(p)$. Furthermore, $(\mathbb{R}^d / \text{ker}(p), ~p)$ forms a Banach space.
	\end{lemma}

	\subsection{A Seminorm Fixed-Point Theorem}
	In this section, we present a fixed-point theorem that extends the Banach fixed-point theorem for norm-contractive operators to the seminorm setting. The results in this section are fundamental for establishing finite-sample guarantees for SA algorithms, as will be discussed in Section \ref{sec:SA}.
	
	The following definition formalizes the concept of seminorm-contractive operators.
	\begin{definition}
		\label{def:seminorm contraction mapping}
		An operator $T: \mathbb{R}^d \to \mathbb{R}^d$ is a contraction mapping with respect to a seminorm $p(\cdot)$ if there exists $\gamma\in [0,1)$ such that 
		\begin{align*}
			p(T(x)-T(y)) \leq \gamma p(x-y),\quad \forall\,x, y \in \mathbb{R}^d.
		\end{align*}
		Moreover, we call $x^* \in \mathbb{R}^d$ a fixed point of $T: \mathbb{R}^d \to \mathbb{R}^d$ with respect to $p(\cdot)$  if $p(T(x^*) - x^*) = 0$.
	\end{definition}
	
	Note that unlike in the norm-contractive setting, $p(T(x^*) - x^*) = 0$, in general, does not imply $T(x^*)=x^*$.
	Next, we present a natural generalization of the Banach fixed-point theorem to the seminorm setting. 
	
	\begin{theorem}
		\label{theorem:Banach fixed-point theorem}
		Let $T: \mathbb{R}^d \to \mathbb{R}^d$ be a $\gamma$-contraction mapping with respect to a seminorm $p(\cdot)$. Then, there exists $x^* \in \mathbb{R}^d$ such that $p(T(x^*)-x^*) = 0$. In addition, for any $x_0 \in \mathbb{R}^d$, the sequence $\{x_k\}$ generated by $x_{k+1} = T(x_k)$ satisfies
		\begin{align*}
			p(x_k - x^*) \leq \gamma^k p(x_0 - x^*), \quad \forall\, k \geq 0.
		\end{align*}
	\end{theorem}
	
	The complete proof of Theorem \ref{theorem:Banach fixed-point theorem} is presented in Appendix \ref{pf:theorem:Banach fixed-point theorem}. The high-level idea is to show that $p(\cdot)$ is a norm-contractive mapping in the Banach space $(\mathbb{R}^d / \text{ker}(p), ~p)$ (cf. Lemma \ref{lemma:p is norm on quotient space}), which enables us to leverage the Banach fixed-point theorem. 
	
	The above theorem guarantees the existence of a unique affine subspace of fixed points $[x^*]=x^*+\ker(p)$.  Moreover, since $p(x_k-x^*)$ can be viewed as the distance from $x_k$ to the set $[x^*]$ with respect to some norm $\| \cdot \|$ (cf. Proposition \ref{prop:seminorm_properties}), fixed-point iteration guarantees the geometric convergence of the sequence $\{x_k\}$ to the affine subspace $[x^*]$. One should be careful in interpreting the convergence results in Theorem \ref{theorem:Banach fixed-point theorem}. Unlike the classical definition of convergence, even if $x_k$ converges to $x^*$ in some seminorm, certain components of $x_k$ may still diverge to infinity. As a clear example, consider $T: \mathbb{R}^2 \to \mathbb{R}^2$ defined as $T(x(1), x(2)) = (x(1)/2, 2x(2))$, which is a contraction mapping with respect to the seminorm $\hat{p}(x(1), x(2)) = |x(1)|$, with a contraction factor of $1/2$. In addition, the set of solutions to the seminorm fixed-point equation $\hat{p}(T(x(1), x(2)) - (x(1), x(2))) = 0$ is given by $E = \{(x(1), x(2)) \in \mathbb{R}^2 \mid x(1) = 0, x(2) \in \mathbb{R}\}$. As a result of Theorem \ref{theorem:Banach fixed-point theorem}, the sequence $\{x_k\}$ generated by $x_{k+1} = T(x_k)$ converges to $E$ geometrically fast. This implies that $x_k(1)$ converges to zero. However, the results do not imply any kind of convergence for $x_k(2)$. In fact, since $x_k(2) = 2^k x_0(2)$, unless $x_0(2) = 0$, we have $\lim_{k \to \infty} |x_k(2)| = \infty$.

	To further illustrate the use of Theorem \ref{theorem:Banach fixed-point theorem}, we have provided two more examples of using Theorem \ref{theorem:Banach fixed-point theorem} to study the convergence behavior of optimization algorithms in Appendix \ref{ap:example:optimization}. 
	
	\section{Seminorm Lyapunov Stability Theorems}\label{sec:Lyapunov}
	A notable special case of fixed-point equations is the class of linear systems of equations, which serves as a representative example in the analytical study of root-finding problems and has wide-ranging applications. In this context, fixed-point iteration can be interpreted as a discrete linear dynamical system, and Theorem \ref{theorem:Banach fixed-point theorem} provides the foundation for establishing a \textit{seminorm Lyapunov stability theorem}. Given the significance of linear systems, we devote this section to exploring their properties and implications. Furthermore, for completeness, we present Lyapunov stability theorems for both discrete-time and continuous-time linear dynamical systems.

	\subsection{Discrete-Time Linear Dynamical Systems}
	When $T(\cdot)$ is a linear operator, the fixed-point iteration (\ref{eq:FPI_intro}) takes the form  
	\begin{align}\label{algo:deterministic_linear}
		x_{k+1} = Ax_k, \quad \forall\, k \geq 0,
	\end{align}  
	where $A\in\mathbb{R}^{d\times d}$ is a real-valued matrix. Understanding the convergence behavior of $\{x_k\}$ generated by Eq. (\ref{algo:deterministic_linear}) has been a central topic in control theory. Specifically, the classical Lyapunov stability theorem \cite{khalil2002nonlinear} states that the linear dynamical system (\ref{algo:deterministic_linear}) is globally geometrically stable (equivalently, the spectral radius of $A$ is strictly less than one) if and only if, for any positive definite matrix $Q\in\mathbb{R}^{d\times d}$, there exists a unique positive definite matrix $P\in\mathbb{R}^{d\times d}$ that solves the Lyapunov equation 
	\begin{align}\label{eq:discrete Lyapunov equation}
		A^\top P A - P + Q = 0.
	\end{align}
	The Lyapunov equation is of significant importance because it provides an explicit way to verify the stability of the linear system (\ref{algo:deterministic_linear}).
	
	The next theorem can be viewed as a generalization of the classical Lyapunov stability theorem to the case where $A$ may have eigenvalues with modulus greater than one. In this case, the stability (in terms of a properly defined seminorm) of $\{x_k\}$ is also characterized by the same Lyapunov equation.
	
	\begin{theorem}\label{theorem:Seminorm GAS and Lyapunov Equation}
		Consider the sequence $\{x_k\}$ generated by Eq. (\ref{algo:deterministic_linear}), and let $E$ be a linear subspace of $\mathbb{R}^d$. Then, the following statements are equivalent:  
		\begin{enumerate}[(1)]
			\item  The subspace $E$ is invariant under $A$, i.e., $x \in E$ implies $Ax \in E$. Furthermore, $E$ contains $E_{A,\geq 1}$, where $E_{A,\geq 1}$ denotes the subspace spanned by all generalized eigenvectors of $A$ corresponding to eigenvalues with moduli greater than or equal to one.
			
			\item There exists a seminorm $p(\cdot)$ with $\ker(p) = E$ and constants $\alpha, \beta > 0$ such that $p(x_k) \leq \alpha p(x_0) e^{-\beta k}$ for all $x_0 \in \mathbb{R}^d$.
			
			\item For any seminorm $p(\cdot)$ with $\ker(p) = E$, there exist constants $\alpha, \beta > 0$ such that $p(x_k) \leq \alpha p(x_0) e^{-\beta k}$ for all $x_0 \in \mathbb{R}^d$.
			
			\item There exist matrices $P, Q \in \mathcal{S}^{d}_{+,E}$ satisfying Eq. (\ref{eq:discrete Lyapunov equation}), where $\mathcal{S}^{d}_{+,E} = \{B \succeq 0 \mid \ker(B) = E\}$ denotes the set of positive semi-definite matrices with kernel $E$.
			
			\item For any $Q \in \mathcal{S}^{d}_{+,E}$, there exists a unique $P \in \mathcal{S}^{d}_{+,E}$ satisfying Eq. (\ref{eq:discrete Lyapunov equation}).
		\end{enumerate}
		Consequently, the smallest subspace $E$ for which any of these statements hold is $E = E_{A,\geq 1}$.
	\end{theorem}
	
	To understand Theorem \ref{theorem:Seminorm GAS and Lyapunov Equation}, consider the special case where $E=\{0\}$. In this case, for Statement (1) in Theorem \ref{theorem:Seminorm GAS and Lyapunov Equation} to hold, we must have $E_{A,\geq 1}=\{0\}$, which implies that the spectral radius of $A$ is strictly less than one. Then, statements (2)–(5) of Theorem \ref{theorem:Seminorm GAS and Lyapunov Equation} reduce to the classical Lyapunov stability theorem \cite{khalil2002nonlinear}. Note that, even when $E_{A,\geq 1}=\{0\}$, one has the flexibility to choose the linear subspace $E$ to apply Theorem \ref{theorem:Seminorm GAS and Lyapunov Equation}, which is not captured by the classical Lyapunov stability theorem. More generally, Theorem \ref{theorem:Seminorm GAS and Lyapunov Equation} states that as long as the kernel space of a seminorm $p(\cdot)$ contains the ``unstable region'' of $A$, i.e., $E_{A,\geq 1}$, the discrete linear dynamical system converges geometrically in $p(\cdot)$. Moreover, this condition can be verified by finding a solution to the Lyapunov equation (\ref{eq:discrete Lyapunov equation}).

	The complete proof of Theorem \ref{theorem:Seminorm GAS and Lyapunov Equation} is provided in Appendix \ref{pf:theorem:Seminorm GAS and Lyapunov Equation}. Here, we provide a proof sketch. Our plan to prove Theorem \ref{theorem:Seminorm GAS and Lyapunov Equation} is to first prove the equivalence among Statements (2) -- (5) and then prove the equivalence between Statements (1) and (2). The equivalence among (2) -- (5) extends the proof of the classical Lyapunov stability theory to the seminorm setting. Specifically, the implication (2) $\Rightarrow$ (3) follows from the equivalence of seminorms that share the same kernel space (cf. Proposition \ref{prop:seminorm_properties}). The implication (3) $\Rightarrow$ (5) is established by constructing $P = \sum_{k=0}^\infty (A^\top)^k Q A^k$ and verifying that it is the unique solution to the Lyapunov equation (\ref{eq:discrete Lyapunov equation}). The implication (5) $\Rightarrow$ (4) is trivial, and (4) $\Rightarrow$ (2) follows by defining $p(x) = \sqrt{x^\top P x}$ and using the Lyapunov equation to verify that the discrete linear system (\ref{algo:deterministic_linear}) is contracting with respect to $p(\cdot)$. The most challenging part is proving (1) $\Leftrightarrow$ (2). Specifically, to show (1) $\Rightarrow$ (2), we use the Jordan normal form of $A$ to explicitly construct a seminorm $p(\cdot)$ with $\ker(p) = E$ such that the discrete linear system (\ref{algo:deterministic_linear}) is contracting with respect to $p(\cdot)$. To prove (2) $\Rightarrow$ (1), we use an induction argument to show that all generalized eigenvectors associated with eigenvalues of $A$ with moduli greater than or equal to one are contained in $E$. Importantly, although the proof of Theorem \ref{theorem:Seminorm GAS and Lyapunov Equation} utilizes the properties of seminorms stated in Proposition \ref{prop:seminorm_properties}, it does not follow from the general seminorm fixed-point theorem (cf. Theorem \ref{theorem:Banach fixed-point theorem}). For more details, refer to Appendix \ref{pf:theorem:Seminorm GAS and Lyapunov Equation}. 
	
	As a side note, the globally exponential stability with respect to seminorms for discrete-time linear dynamical systems \eqref{algo:deterministic_linear} has been studied in \cite[Lemma 28]{Bullo2023}. However, Theorem \ref{theorem:Seminorm GAS and Lyapunov Equation} is different from their results in that it further provides a Lyapunov equation that can be used in the stability analysis of discrete linear dynamical systems.

	\subsection{Continuous-Time Linear Dynamical Systems} Although the primary focus of this work is on the convergence behavior of discrete iterative algorithms, our results for seminorms can also be used to extend the classical Lyapunov stability theorem for continuous-time linear dynamical systems beyond the traditional Hurwitz setting. We include the result here for completeness. 
	
	Consider the ordinary differential equation (ODE):  
	\begin{align}\label{eq:ODE}  
		\dot{x}(t) = Ax(t), \quad x(0) \in \mathbb{R}^d,  
	\end{align}  
	where $A \in \mathbb{R}^{d \times d}$ is a real-valued matrix. A key question for ODE (\ref{eq:ODE}) is regarding the stability of its equilibrium points, i.e., solutions to $Ax=0$. The classical Lyapunov stability theorem states that the origin is a globally exponentially stable equilibrium point of ODE (\ref{eq:ODE}) (equivalently, the matrix $A$ is Hurwitz) if and only if, for any positive definite matrix $Q \in \mathbb{R}^{d \times d}$, there exists a unique positive definite matrix $P\in \mathbb{R}^{d \times d}$ that satisfies the Lyapunov equation: 
	\begin{align}\label{eq:Lyapunov_Continuous}  
		A^\top P + PA + Q = 0.  
	\end{align}  
	See \cite{khalil2002nonlinear,haddad2008nonlinear} for further details.
	
	In the following theorem, we extend this result to the case where $A$ is not necessarily Hurwitz. However, in this case, the stability is defined in terms of convergence in seminorms. Specifically, we say that ODE (\ref{eq:ODE}) is globally exponentially stable with respect to some seminorm $p(\cdot)$ if and only if there exist constants $\alpha,\beta>0$ such that $p(x(t))\leq \alpha p(x(0))e^{-\beta t}$ for all $t\geq 0$ and any initialization $x(0)\in\mathbb{R}^d$.
	
	\begin{theorem}\label{thm:Lyapunov_Continuous}
		Let $A \in \mathbb{R}^{d \times d}$ be a real-valued matrix, and let $E$ be a linear subspace of $\mathbb{R}^d$. Then, the following statements are equivalent.
		\begin{enumerate}[(1)]
			\item  The linear subspace $E$ is invariant under $A$, i.e., $x\in E$ $\Rightarrow$ $Ax\in E$. Moreover, $E\supseteq E_{A,\geq 0}$, where $E_{A,\geq 0}$ denotes the linear subspace spanned by all generalized eigenvectors associated with the eigenvalues of $ A $ whose real parts are greater than or equal to zero.
			\item There exists a seminorm $p(\cdot)$ with $\text{Ker}(p) = E$ such that the ODE (\ref{eq:ODE}) is globally exponentially stable with respect to $p(\cdot)$.
			
			\item For any seminorm $p(\cdot)$ with $\text{Ker}(p) = E$, the ODE (\ref{eq:ODE}) is globally exponentially stable with respect to $p(\cdot)$.
			
			\item There exists a pair of $P, Q \in \mathcal{S}_{+,E}^d$ satisfying Eq. (\ref{eq:Lyapunov_Continuous}).
			
			\item For any $Q \in \mathcal{S}_{+,E}^d$, there exists a unique $P \in \mathcal{S}_{+,E}^d$ satisfying Eq. (\ref{eq:Lyapunov_Continuous}).
		\end{enumerate}
		Consequently, the smallest subspace $E$ for which any of these statements hold is $E = E_{A,\geq 0}$.
	\end{theorem}
	
	To connect Theorem \ref{thm:Lyapunov_Continuous} with the classical Lyapunov theorem, consider again the special case where $E=\{0\}$. In this case, Statement (1) of Theorem \ref{thm:Lyapunov_Continuous} holds only if $E_{A,\geq 0}=\{0\}$, which implies that the matrix $A$ is Hurwitz. Under this condition, Theorem \ref{thm:Lyapunov_Continuous} recovers the classical Lyapunov stability theory for continuous-time linear dynamical systems.  More generally, Theorem \ref{thm:Lyapunov_Continuous} allows for matrices $A$ that are not Hurwitz. In this case, as long as the unstable region, i.e., $E_{A,\geq 0}$, is contained within the kernel space of the seminorm, the system remains stable with respect to the seminorm. Moreover, the same Lyapunov equation can be used to verify this property.
	
	The proof of Theorem \ref{thm:Lyapunov_Continuous} is presented in Appendix \ref{pf:thm:Lyapunov_Continuous}. The high-level idea behind the proof is similar to that of Theorem \ref{theorem:Seminorm GAS and Lyapunov Equation}, except that we work with matrix exponentials.
	
	\section{Markovian Stochastic Approximation under Seminorm Contraction}
	\label{sec:SA}
	Although fixed-point iteration provides a promising approach for finding a fixed point of a seminorm-contractive operator, its practical application is often constrained by computational limitations or incomplete knowledge of the operator itself. These challenges motivate the use of SA, a data-driven small-stepsize variant of the fixed-point iteration. In this section, we develop a finite-sample analysis of SA under a seminorm contraction setting with Markovian noise.
	
	\subsection{Problem Setting}
	Let $\bar{F}:\mathbb{R}^d \to \mathbb{R}^d$ be an operator defined as $\bar{F}(x) := \mathbb{E}_{Y \sim \mu}[F(x, Y)]$, where $Y \in \mathcal{Y}$ is a random variable with distribution $\mu$, and $F:\mathbb{R}^d \times \mathcal{Y} \to \mathbb{R}^d$ is another operator. We assume that $\bar{F}(\cdot)$ is a seminorm-contractive operator, as stated in the following.
	
	\begin{assumption}\label{as:seminorm_contraction}
		There exists a seminorm $p_{c,E}(\cdot)$ on $\mathbb{R}^d$ with kernel denoted by $E$ such that 
		\begin{align*}
			p_{c,E}(\bar{F}(x_1)-\bar{F}(x_2))\leq \gamma p_{c,E}(x_1-x_2),\quad \forall\,x_1,x_2\in\mathbb{R}^d,
		\end{align*}
		where $\gamma \in [0,1)$ is the contraction factor.
	\end{assumption}
	
	Our goal is to find a fixed point $x^* \in \mathbb{R}^d$ of $\bar{F}(\cdot)$ in the seminorm sense:
	\begin{align}\label{eq:system of equations}
		p_{c,E}\left(\bar{F}(x^*) - x^*\right) = 0.
	\end{align}
	Under Assumption \ref{as:seminorm_contraction}, a solution to Eq. (\ref{eq:system of equations}) must exist, as guaranteed by Theorem \ref{theorem:Banach fixed-point theorem}.
	
	In SA, starting with an arbitrary initialization $x_0 \in \mathbb{R}^d$, we iteratively update $x_k$ as follows:
	\begin{align}\label{algo:SA}
		x_{k+1} = x_k + \alpha_k \left(F(x_k, Y_k) - x_k + \omega_k\right),
	\end{align}
	where $\{\alpha_k\}$ is a sequence of stepsizes, and $\{Y_k\}$ and $\{\omega_k\}$ are two stochastic processes. The algorithm described in Eq. \eqref{algo:SA} extends the SA algorithms studied in \cite{chen2021lyapunov, chandak2022concentration} to the seminorm-contractive setting. To analyze this algorithm, we introduce the following assumption regarding the operator $F(\cdot, \cdot)$ and the stochastic processes $\{Y_k\}$ and $\{\omega_k\}$.

	\begin{assumption}\label{as:sa_all}
		The following properties hold:
		\begin{enumerate}[(1)]
			\item  There exist $A_1,B_1>0$ such that $p_{c,E}(F(x_1,y)-F(x_2,y))\leq A_1p_{c,E}(x_1-x_2)$ for any $x_1,x_2\in\mathbb{R}^d$ and $y\in\mathcal{Y}$ and $p_{c,E}(F(0,y))\leq B_1$ for any $y\in\mathcal{Y}$.
			\item The stochastic process $\{Y_k\}$ is a uniformly ergodic Markov chain with unique stationary distribution $\mu$. 
			\item The stochastic process $\{\omega_k\}$ satisfies $\mathbb{E}[\omega_k \mid \mathcal{F}_k]=0$ and $p_{c,E}(\omega_k)\leq A_2p_{c,E}(x_k)+B_2$ almost surely for all $k \geq 0$, where $\mathcal{F}_k$ is the $\sigma$-algebra generated by $\{(x_i,Y_i,\omega_i)\}_{0\leq i\leq k-1}\cup \{x_k\}$ and $A_2,B_2>0$ are constants.
		\end{enumerate}
	\end{assumption}
	
	Assumption \ref{as:sa_all} (1) is a natural generalization of the standard Lipschitz continuity assumption for studying SA under norm-contractive operators. Assumption \ref{as:sa_all} (2) is motivated by applications in RL, where the environment is modeled as an MDP. Under uniform ergodicity, there exist $C > 0$ and $\rho \in (0,1)$ such that $\sup_{y\in\mathcal{Y}}~ d_{\text{TV}}\left(\mathbb{P}\left(Y_k | Y_0=y\right), \mu\right) \leq C\rho^k$ for all $k\geq 0$, where $d_{\text{TV}}\left(\nu_1, \nu_2 \right)$ stands for the total variation distance between probability measures $\nu_1$ and $\nu_2$ \cite{levin2017markov}. For a Markov chain with a finite state space, Assumption \ref{as:sa_all} (2) is satisfied when the Markov chain is irreducible and aperiodic \cite{levin2017markov}. Assumption \ref{as:sa_all} (3) states that the additive noise is a martingale difference sequence, which is also standard in the SA literature \cite{borkar2000ode}.

	Next, we introduce the mixing time of a Markov chain, which is a useful quantity for our analysis.
	\begin{definition}\label{definition:Markov chain mixing time}
		Given a Markov chain $\mathcal{M}_Y=\{Y_k\}$ with unique stationary distribution $\mu_Y$, for any $\delta>0$, the mixing time $t_\delta(\mathcal{M}_Y)$ of the Markov chain $\mathcal{M}_Y$ with accuracy $\delta$ is defined as
		\begin{align}\label{eq:mixing time definition}
			t_\delta(\mathcal{M}_Y) = \min\left\{k \geq 0 \;\middle|\; \sup_{y\in\mathcal{Y}}~ d_{\text{TV}}\left(\mathbb{P}\left(Y_k | Y_0=y\right), \mu_Y\right) \leq \delta \right\}.
		\end{align}
	\end{definition}
	
	Throughout this section, for simplicity of notation, we will use $t_\delta$ for the mixing time of the Markov chain $\{Y_k\}$ with accuracy $\delta$. In addition, when $\delta=\alpha_k$, where $\alpha_k$ is the stepsize in Eq. (\ref{algo:SA}), we denote $t_{k}=t_{\alpha_k}$. Under Assumption \ref{as:sa_all} (2), the Markov chain $\{Y_k\}$ mixes at a geometric rate. Therefore, we have $t_\delta
	=\mathcal{O}( \log \left(1/\delta \right))$,
	which implies that $\lim_{\delta \rightarrow 0} \delta t_\delta = 0$. This property is critical in our analysis for controlling the Markovian noise.

	\subsection{Finite-Sample Analysis}
	In this section, we present the finite-sample bounds for the SA algorithm described in Eq. \eqref{algo:SA}. To state the result, we first specify the requirements for selecting the stepsize sequence $\{\alpha_k\}$. For simplicity, let $ A = A_1 + A_2 + 1 $ and $ B = B_1 + B_2 $, where the constants $ A_1 $, $ A_2 $, $ B_1 $, and $ B_2 $ are defined in Assumption \ref{as:sa_all}. Additionally, the theorem depends on three problem-specific constants: $ \varphi_1 > 0 $, $ \varphi_2 \in (0, 1) $, and $ \varphi_3 > 0 $. The explicit expressions for these constants will be provided in the proof section (see Section \ref{subsec:proofsketch}, Eq. (\ref{eq:def:constants})).  
	
	\begin{condition}
		\label{condition: stepsizes requirement}
		The stepsize sequence $\{\alpha_k \}$ is positive, non-increasing, and satisfies $\alpha_{k-t_k, k-1} \leq \min\{\varphi_2/\left(\varphi_3 A^2\right), 1/\left(4A\right)\}$ for all $k \geq t_k$, where $\alpha_{i, j}:= \sum^j_{k=i} \alpha_k$.
	\end{condition}
	
	Now, we are ready to state our main theorem, the proof of which will be discussed in detail in Section \ref{subsec:proofsketch}.
	
	\begin{theorem}\label{thm:SA_finite}
		Consider $\{x_k\}$ generated by the SA algorithm presented in Eq. \eqref{algo:SA}. Suppose that Assumptions \ref{as:seminorm_contraction} and \ref{as:sa_all} are satisfied and the stepsize sequence $\{\alpha_k\}$ verifies Condition \ref{condition: stepsizes requirement}. Then, we have for all $k\geq K:= \min\{k\geq 0 ~|~ k \geq t_k\}$ (which is finite under Assumption \ref{as:sa_all} (2)) that
		\begin{align}
			\label{eq:general finite-sample bound}
			\mathbb{E}\left[p_{c,E}(x_k-x^*)^2\right]\leq   \varphi_1c_1\prod_{j=K}^{k-1}(1-\varphi_2\alpha_j)+\varphi_3c_2\sum_{i=K}^{k-1}\alpha_i\alpha_{i-t_i,i-1}\prod_{j=i+1}^{k-1}(1-\varphi_2\alpha_j),
		\end{align}
		where $c_1=\left(p_{c,E}(x_0-x^*)+p_{c,E}(x_0)+B/A\right)^2$ and $c_2=\left(Ap_{c,E}(x^*)+B\right)^2$.
		In particular, we have the following convergence bounds for three common choices of stepsizes (as long as they satisfy Condition \ref{condition: stepsizes requirement}):
		\begin{enumerate}[(1)]
			\item When $\alpha_k \equiv \alpha\in (0,1)$, we have for all $k \geq t_\alpha$:
			\begin{equation*}
				\begin{split}
					\mathbb{E}[p_{c,E}(x_{k}-x^*)^2] \leq \varphi_1 c_1 (1 - \varphi_2 \alpha)^{k-K} + \frac{\varphi_3c_2 }{\varphi_2} \alpha t_\alpha.
				\end{split}
			\end{equation*}
			\item When $\alpha_k = \alpha/(k + h)$, where $\alpha > 1/\varphi_2$ and $h > 0$, we have for all $k \geq K$:
			\begin{equation*}
				\begin{split}
					\mathbb{E}[p_{c,E}(x_{k}-x^*)^2] \leq \varphi_1 c_1 \left(\frac{K+h}{k + h}\right)^{\alpha \varphi_2} + \frac{8e\alpha^2\varphi_3c_2}{\varphi_2\alpha - 1} \frac{t_k}{k + h}.
				\end{split}
			\end{equation*}
			\item When $\alpha_k = \alpha/(k + h)^\xi$, where $\xi \in (0, 1)$ and $\alpha,h > 0$, we have for all $k \geq K$:
			\begin{equation*}
				\begin{split}
					\mathbb{E}[p_{c,E}(x_{k}-x^*)^2] \leq \varphi_1 c_1 e^{-\frac{\varphi_2 \alpha}{1-\xi}\left[(k+h)^{1-\xi} - (K+h)^{1-\xi}\right]} + \frac{4\varphi_3c_2 \alpha}{\varphi_2} \frac{t_k}{(k + h)^\xi}.
				\end{split}
			\end{equation*} 
		\end{enumerate}
	\end{theorem}

	In each case of Theorem \ref{thm:SA_finite}, the right-hand side of the bound is a combination of the ``bias'' term and the ``variance'' term. Using constant stepsizes is very efficient in driving the bias to zero, but cannot eliminate the variance. When using the $\mathcal{O}(1/k)$ stepsizes, the convergence rate is roughly $\Tilde{\mathcal{O}}\left(1/k\right)$ because $t_k=\mathcal{O}(\log(k))$, which is orderwise optimal. When using the $\mathcal{O}(1/k^\xi)$ stepsizes, the resulting convergence rate is $\Tilde{\mathcal{O}}\left(1/k^{\xi}\right)$, which is sub-optimal but more robust with respect to $\alpha$. To compute the sample complexity, we use Theorem \ref{thm:SA_finite} Part (2). Given $\epsilon>0$, to make
	$\mathbb{E}[p_{c,E}(x_{k}-x^*)]\leq \epsilon$, the sample complexity is $\tilde{\mathcal{O}}(\epsilon^{-2})$. These results are qualitatively similar to those for SA under norm-contractive operators \cite{chen2021lyapunov}. In fact, when $E=\{0\}$ (in which case $p_{c,E}(\cdot)$ becomes a norm), Theorem \ref{thm:SA_finite} recovers existing results for norm-contractive SA algorithms as its special case. 
	
	As a final remark, although Theorem \ref{thm:SA_finite} establishes mean-square convergence of the seminorm $p_{c,E}(\cdot)$, it is important to note that convergence in a seminorm does not guarantee convergence (or even boundedness) of all components of $x_k$. Certain components may still diverge to infinity. To ensure the empirical stability of the algorithm, one can incorporate a projection step into the algorithm. Specifically, according to Proposition \ref{prop:seminorm_properties}, there exists a norm $\|\cdot\|_c$ such that $p_{c,E}(x) = \min_{y \in E = \text{ker}(p)} \|x - y\|_c$ for all $x \in \mathbb{R}^d$.
	Therefore, instead of directly implementing the SA algorithm as presented in Eq.~(\ref{algo:SA}), we can modify it as follows:
	\[
	x_{k+1} = \argmin_{x \in E} \|x_k + \alpha_k \left(F(x_k, Y_k) - x_k + \omega_k\right) - x\|_c, \quad \forall\, k \geq 0.
	\]
	Solving for the \textit{argmin} is not always straightforward, depending on the seminorm and the kernel space. In the context of Q-learning in average reward RL, since the corresponding seminorm is the span seminorm, the optimization problem admits a closed-form solution. See the paragraph after Definition \ref{def:seminorm} for more details.
	
	\subsection{Linear Stochastic Approximation}\label{sec:linear_SA}
	A special case of SA is linear SA, which deserves particular attention due to its wide applicability in RL and control. A linear SA takes the form  
	\begin{align}\label{algo:linear_SA}  
		x_{k+1} = x_k + \alpha_k \left(A(Y_k)x_k + b(Y_k)\right),  
	\end{align}  
	where $A:\mathcal{Y} \to \mathbb{R}^{d \times d}$ is a matrix-valued function, $b:\mathcal{Y} \to \mathbb{R}^d$ is a vector-valued function, and $\{Y_k\}$ is a Markov chain with a unique stationary distribution $\mu$. Let $\bar{A} \in \mathbb{R}^{d \times d}$ and $\bar{b} \in \mathbb{R}^d$ be defined as $\bar{A} = \mathbb{E}_{Y \sim \mu}[A(Y)]$ and $\bar{b} = \mathbb{E}_{Y \sim \mu}[b(Y)]$, respectively.

	We impose the following assumptions to study the linear SA described in Eq. (\ref{algo:linear_SA}). Let $E$ be a linear subspace of $\mathbb{R}^d$.

	\begin{assumption}\label{as:lsa_Lyapunov}
		There exists a pair of matrices $ P,Q \in \mathcal{S}_{+,E}^d $ such that $ \bar{A}^\top P + P\bar{A} + Q = 0 $.
	\end{assumption}

	In the special case where $ E = \{0\} $, Assumption \ref{as:lsa_Lyapunov} is equivalent to stating that the matrix $ \bar{A} $ is Hurwitz, i.e., all eigenvalues have strictly negative real parts \cite{khalil2002nonlinear,haddad2008nonlinear}. The Hurwitzness property has been a standard assumption in the literature for studying the convergence behavior of linear SA algorithms \cite{bertsekas1996neuro,srikant2019finite}. Assumption \ref{as:lsa_Lyapunov} is weaker, as it allows flexibility in choosing the subspace $ E $. In fact, as long as $ E $ is invariant under $A$ and contains $E_{A,\geq 0}$, which is the linear subspace spanned by the generalized eigenvectors of $ \bar{A} $ corresponding to eigenvalues with non-negative real parts, Assumption \ref{as:lsa_Lyapunov} holds automatically (cf. Theorem \ref{thm:Lyapunov_Continuous}).
	
	Let $p:\mathbb{R}^d\to \mathbb{R}$ be defined as $p(x)=\sqrt{x^\top Px}$, where the matrix $P$ is given in Assumption \ref{as:lsa_Lyapunov}. It is clear that $p(\cdot)$ is a seminorm with its kernel space being $E$. Moreover, according to Proposition \ref{prop:seminorm_properties}, there exists a norm, denoted by $\|\cdot\|_c$, such that $p(x)=\min_{x'\in E}\|x-x'\|_c$.

	\begin{assumption}\label{as:lsa_all}
		The following properties hold.
		\begin{enumerate}[(1)]
			\item For any $y\in\mathcal{Y}$, $x\in E$ implies $A(y)x\in E$.
			\item There exist $L_1,L_2>0$ such that $\sup_{y\in\mathcal{Y}}\|A(y)\|_c\leq L_1$ and $\sup_{y\in\mathcal{Y}}\|b(y)\|_c\leq L_2$.
			\item There exists $C \geq 1$ such that $t_\delta \leq C \log(1/\delta)$, where
			\begin{align*}  
				t_\delta := \min\left\{k \geq 0 \,\middle|\, \|\bar{A} - \mathbb{E}[A(Y_k) \mid Y_0=y]\|_c \leq L_1\delta, \|\bar{b}-\mathbb{E}[b(Y_k)\mid Y_0=y]\|_c \leq L_2\delta,~\forall\,y \in \mathcal{Y}\right\}.
			\end{align*}
		\end{enumerate}
	\end{assumption}
	\begin{remark}
		The specific norm used to state Assumption \ref{as:lsa_all} (2) and (3) is not important because all norms are equivalent in finite-dimensional vector spaces. We choose to state the assumption using $\|\cdot\|_c$ for simplicity of presentation.
	\end{remark}

	Assumption \ref{as:lsa_all} (1) states that $E$ is an invariant subspace under $A(y)$ for any $y\in\mathcal{Y}$. Assumption \ref{as:lsa_all} (2) is similar to the Lipschitz continuity assumption for general seminom-contractive SA (cf.\ Assumption \ref{as:sa_all} (1)). Assumption \ref{as:lsa_all} (3) is comparable to, yet slightly weaker than, requiring that the Markov chain $\{Y_k\}$ be uniformly ergodic. 
	
	To state the finite-sample bounds of the linear SA algorithm described in Eq. (\ref{algo:linear_SA}), we need the following notation. Let $c_1'=(p(x_0)+p(x_0-x^*)+L_2/L_1)^2$,
	$c_2'>0$ be such that $Q\geq c_2'P$ (such a $c_2'$ is guaranteed to exist because $P$ and $Q$ are both positive semi-definite matrices sharing the same kernel space), and $c_3'=114(L_1p(x^*)+L_2)^2$. The proof of the following theorem is presented in Appendix \ref{sec:lsa_pf}. 
	
	\begin{theorem}\label{thm:linear_SA}
		Consider $\{x_k\}$ generated by the linear SA algorithm presented in Eq. \eqref{algo:linear_SA}. Suppose that Assumptions \ref{as:lsa_Lyapunov} and \ref{as:lsa_all} are satisfied and the stepsize sequence $\{\alpha_k\}$ is positive, non-increasing, and satisfies $\alpha_{k-t_k, k-1} \leq \min\{1/(4L_1),c_2'/(228L_1^2)\}$ for all $k \geq t_k:=t_{\alpha_k}$.
		Then, we have for all $k\geq K:=\min\{k\geq 0 ~|~ k \geq t_k\}$ that 
		\begin{align*}
			\mathbb{E}[p^2(x_k-x^*)]
			\leq c_1'\prod_{j=K}^{k-1}\left(1-c_2'\alpha_j/2\right)+c_3'\sum_{i=K}^{k-1}\alpha_i\alpha_{i-t_i,i-1}\prod_{j=i+1}^{k-1}\left(1-c_2'\alpha_j/2\right).
		\end{align*}
		In particular, we have the following convergence bounds for three common choices of stepsizes:
		\begin{enumerate}[(1)]
			\item When $\alpha_k \equiv \alpha$, we have for all $k \geq K$:
			\begin{equation*}
				\begin{split}
					\mathbb{E}[p(x_{k}-x^*)^2] \leq c_1' \left(1-c_2'\alpha /2\right)^{k-t_\alpha} + \frac{2c_3'}{c_2'} \alpha t_\alpha.
				\end{split}
			\end{equation*}
			\item When $\alpha_k = \alpha/(k + h)$, where $\alpha > 2/c_2'$, we have for all $k \geq K$:
			\begin{align*}
				\mathbb{E}[p(x_{k}-x^*)^2] \leq c_1' \left(\frac{K+h}{k + h}\right)^{\alpha c_2'/2} + \frac{16e\alpha^2c_3'}{c_2'\alpha - 2} \frac{t_k}{k + h}.
			\end{align*}
			\item When $\alpha_k = \alpha/(k + h)^\xi$, where $\xi \in (0, 1)$, we have for all $k \geq K$:
			\begin{align*}
				\mathbb{E}[p(x_{k}-x^*)^2] \leq c_1' e^{-\frac{c_2' \alpha}{2(1-\xi)}\left[(k+h)^{1-\xi} - (K+h)^{1-\xi}\right]} + \frac{8c_3' \alpha}{c_2'} \frac{t_k}{(k + h)^\xi}.
			\end{align*} 
		\end{enumerate}
		
	\end{theorem}
	
	Theorem \ref{thm:linear_SA} is qualitatively similar to Theorem \ref{thm:SA_finite} for general seminorm-contractive SA. Moreover, when $\bar{A}$ is Hurwitz, i.e., Assumption \ref{as:lsa_Lyapunov} is satisfied with $E=\{0\}$, we recover existing finite-sample bounds of linear SA in \cite{srikant2019finite}.
	
	As a final remark, suppose that the expected operator $\bar{F}(\cdot)$ is not linear but satisfies $(\bar{F}(x)-x)^\top P(x-x^*)\leq -c_0(x-x^*)^\top P(x-x^*)$ (where $c_0>0$) for some positive semi-definite matrix $P$, where $x^*$ is a particular solution to $(\bar{F}(x)-x)^\top P(\bar{F}(x)-x)=0$. Then, our results can be easily reproduced to provide the finite-sample bounds of the corresponding SA algorithm. This case can be viewed as an extension of SA with a dissipative operator \cite{chen2019finite}.
	
	\subsection{Proof of Theorem \ref{thm:SA_finite}}\label{subsec:proofsketch}
	Our approach for proving Theorem \ref{thm:SA_finite} is a Lyapunov-based approach. Specifically, we construct a suitable Lyapunov function $M_E(\cdot)$ that is a smooth approximation of the seminorm-square function $p^2_{c,E}(\cdot)/2$. Then, we use the Lyapunov function to derive the following one-step recursive bound:
	\begin{equation}\label{eq:recursive contractive bound}
		\mathbb{E}\left[M_E(x_{k+1} - x^*)\right] \leq \left(1 - \mathcal{O}(\alpha_k)\right) \mathbb{E}\left[M_E(x_{k} - x^*)\right] + o(\alpha_k),
	\end{equation}
	which will then be repeatedly used to obtain the desired finite-sample bound.
	
	To begin with the construction of the Lyapunov function, the following definition is needed. 
	
	\begin{definition}\label{def:infimal_convolution}
		Let $f_1,f_2:\mathcal{D}\to (-\infty,\infty]$ be two proper functions, where $\mathcal{D}$ is a subset of $\mathbb{R}^d$. The infimal convolution of $f_1(\cdot)$ and $f_2(\cdot)$ is defined as
		\begin{align*}
			(f_1\Box f_2)(x)=\inf_{x\in\mathcal{D}}\{f_1(u)+f_2(x-u)\}.
		\end{align*}
	\end{definition}
	
	The infimal convolution is an important concept in optimization \cite{beck2017first}. Many popular algorithms, such as the proximal point method \cite{parikh2014proximal}, were developed based on the infimal convolution.
	
	Let $\|\cdot\|_s$ be a properly chosen norm such that the function $\|x\|_s^2/2$ is $L$-smooth with respect to $\|\cdot\|_s$ for some $L>0$. Our Lyapunov function is defined as
	$$M_E(x):=\left(\frac{1}{2}p^2_{c,E} ~\Box~ \frac{1}{2\theta} \| \cdot \|^2_s \right)(x) = \min_{y \in \mathbb{R}^d} \left\{\frac{1}{2} p^2_{c,E}(y) + \frac{1}{2\theta} \| x - y \|^2_s \right\},$$
	where $\theta > 0$ is a tunable parameter.

	Before moving forward, we make the following observation regarding the connection between the seminorm and infimal convolution. Recall that by Proposition \ref{prop:seminorm_properties} (2), there exists a norm, denoted by $\| \cdot \|_c$, such that $p_{c,E}(x) = \min_{y \in E} \| x - y \|_c$ for all $x \in \mathbb{R}^d$.  Therefore, using the definition of the infimal convolution, we can write the seminorm $p_{c,E}(\cdot)$ equivalently as
	\begin{equation}
		\label{eq:seminorm infimal convolution}
		p_{c,E}(x) = \left(\| \cdot \|_c ~\Box~ \delta_E \right)(x)= \min_{y \in \mathbb{R}^d} \left\{\| y \|_c + \delta_{E}(x - y) \right\},
	\end{equation}
	where $\delta_E(x)=0$ if $x\in E$ and $\delta_E(x)=+\infty$ otherwise. Similarly, the function $p_{s,E}:\mathbb{R}^d\to \mathbb{R}$ defined as $p_{s,E}(x)= \left(\| \cdot \|_s ~\Box~ \delta_E \right)(x)= \min_{y \in E} \| x - y \|_s $ is also a seminorm. A sequence of properties regarding infimal convolutions and indicator functions are summarized in the following lemma, the proof of which is presented in Appendix \ref{pf:le:box}.
	
	\begin{lemma}\label{le:box}
		Let $f, g, h:\mathbb{R}^d\to \mathbb{R}$ be three functions, and let $E$ be a linear subspace of $\mathbb{R}^d$. Then the following properties hold:
		\begin{enumerate}[(1)]
			\item  Monotonicity: If $f(\cdot) \geq g(\cdot)$, then $(f\,\Box\, \delta_E)\geq  (g \,\Box\, \delta_E)$.
			\item Scaling Invariance: $\beta f ~\square~ \delta_E = \beta (f ~\square~ \delta_E)$ for any $\beta \geq 0$.
			\item Commutativity: $f ~\square~ g = g ~\square~ f$.
			\item Associativity: $\left(f ~\square~ g\right) ~\square~ h = f ~\square~ \left(g ~\square~ h\right)$.
			\item $\delta_E ~\square~ \delta_E = \delta_E$
			\item If $f(\cdot)$ is convex and $L$-smooth with respect to some norm $\|\cdot\|$, then $f ~\square~ \delta_E$ is also convex and $L$-smooth with respect to $\|\cdot\|$.
		\end{enumerate}
	\end{lemma}
	
	Next, we present in the following proposition several important properties of our Lyapunov function $M_E(\cdot)$. See Appendix \ref{pf:prop:Moreau} for the proof.

	\begin{proposition}\label{prop:Moreau}
		The function $M_E(\cdot)$ has the following properties.
		\begin{enumerate}[(1)]
			\item  $M_E(\cdot)$ is convex and satisfies 
			\begin{align*}
				M_E(y)\leq M_E(x)+\langle \nabla M_E(x),y-x\rangle+\frac{L}{2\theta}p_{s,E}^2(y-x),\quad \forall\,x,y\in\mathbb{R}^d.
			\end{align*}
			\item Let  $\ell_{cs},u_{cs}$ be such that $\ell_{cs}\|\cdot\|_s\leq \|\cdot\|_c\leq u_{cs}\|\cdot\|_s$, and let $\ell_{cm}=\sqrt{1+\theta \ell_{cs}^2}$ and $u_{cm}=\sqrt{1+\theta u_{cs}^2}$. Then, we have
			\begin{align*}
				\ell^2_{cm}M_E(x) \leq \frac{1}{2}p^2_{c,E}(x)\leq u^2_{cm}M_E(x),\quad \forall\,x \in \mathbb{R}^d.
			\end{align*}
			\item It holds for all $x,y\in\mathbb{R}^d$ that  
			\begin{align*}
				\| \nabla M_E(x) - \nabla M_E(y) \|_{s,*} \leq \frac{L}{\theta} p_{s,E}(x-y),
			\end{align*}
			where $\| \cdot \|_{s,*}$ is the dual norm of $\| \cdot \|_s$.
		\end{enumerate}
	\end{proposition}
	
	Note that the first two properties together imply that $M_E(\cdot)$ is a smooth approximation of the function $p^2_{c,E}(\cdot)/2$, and the last property is a stronger version of the smoothness property in convex optimization because $p_{s,E}(x-y) \leq \| x-y \|_s$ for all $x, y \in \mathbb{R}^d$. 
	
	The smooth approximation property of $M_E(\cdot)$ enables us to establish a one-step negative drift inequality of $x_k$ with respect to $M_E(\cdot)$. To present the result, we need the following notation. Let $\theta > 0$ be chosen such that $\gamma^2 < (1+\theta \ell^2_{cs})/(1+\theta u^2_{cs})$, which is always possible since $\gamma \in [0,1)$ and $\lim_{\theta \rightarrow 0}(1+\theta \ell^2_{cs})/(1+\theta u^2_{cs})=1$. Furthermore, we define constants
	\begin{align}\label{eq:def:constants}
		\varphi_1:=\frac{1+\theta u_{cs}^2}{1+\theta \ell_{cs}^2},\quad  \varphi_2:=\frac{1}{2}\left(1-\gamma^2 \varphi_1 \right), \quad \text{and} \quad \varphi_3:=\frac{82L (1+\theta u_{cs}^2)}{\theta\ell_{cs}^2},
	\end{align}
	which were used in stating Theorem \ref{thm:SA_finite}.
	Next, we present the one-step drift inequality.
	\begin{proposition}\label{le:recursion}
		It holds  for all $k\geq K$ that
		\begin{align}\label{eq:recursion 1}
			\mathbb{E}\left[M_E(x_{k+1}-x^*)\right]\leq
			\left(1-\varphi_2 \alpha_k\right)\mathbb{E}[M_E(x_k-x^*)] +\frac{\varphi_3\alpha_k\alpha_{k-t_k,k-1}}{2u_{cm}^2}\left(Ap_{c,E}(x^*)+B\right)^2,
		\end{align}
		where we recall that $t_k$ is the mixing time of the Markov chain $\{Y_k\}$ with accuracy $\alpha_k$.
	\end{proposition}
	
	\begin{proof}[Proof of Theorem \ref{thm:SA_finite}]
		Open the recursion in Eq. (\ref{eq:recursion 1}), and we obtain an outright bound on norm-square error (cf. Eq. (\ref{eq:general finite-sample bound})), which in turn gives us the convergence bounds when using different stepsizes. The details are presented in Appendix \ref{proof: general finite-sample bound}.\end{proof}

	\section{Applications in Average Reward Reinforcement Learning}
	\label{sec:average-reward-RL}
	In this section, we showcase the applicability of our SA results in the context of average reward RL. Average reward RL is crucial for optimizing long-term performance in continuing environments where episodic resets are impractical or artificial. Unlike the more commonly studied discounted setting, which prioritizes short-term gains, the average reward framework focuses on steady-state behavior, making it well-suited for problems in operations research, control, and multi-agent systems. Due to the absence of discounting, theoretical guarantees on the finite-sample convergence of popular algorithms such as Q-learning and TD-learning are much more limited. As we will see soon, our results on seminorm-contractive SA with Markovian noise provide a universal framework to study these algorithms.

	\subsection{Background}
	Consider an infinite-horizon average reward MDP described by $(\mathcal{S}, \mathcal{A}, \mathcal{R}, p)$, where $\mathcal{S}$ is a finite state space, $\mathcal{A}$ is a finite action space, $\mathcal{R}: \mathcal{S}\times \mathcal{A} \rightarrow [0, 1]$ is a reward function, and $p$ is the transition kernel, i.e., $p(s'\mid s,a)$ is the transition probability of going to state $s'$ from state $s$ under action $a$. Throughout, we use $|\mathcal{S}|$ (respectively, $|\mathcal{A}|$) to denote the cardinality of the state (respectively, action) space. 
	
	An agent interacts with the environment according to the following protocol: at each time step $k=0,1,2,\cdots$, the agent is in a state $S_k \in \mathcal{S}$ and selects an action $A_k \in \mathcal{A}$, then receives from the environment an immediate reward $\mathcal{R}(S_k, A_k)$ and the next state $S_{k+1}\sim p(\cdot\mid  S_k, A_k)$. Given a stationary deterministic policy $\pi:\mathcal{S}\to\mathcal{A}$, the long-term average reward with initial state $s \in \mathcal{S}$ is defined as
	\begin{align}\label{average reward}
		r^{\pi}(s) = \mathbb{E}_\pi\left[ \liminf_{K \rightarrow \infty} \frac{1}{K}\sum^{K-1}_{k=0} \mathcal{R}(S_k, \pi(S_k)) \,\middle|\, S_0=s \right].
	\end{align}
	Let $\Pi$ be the set of deterministic stationary policies. A policy $\pi^* \in \Pi$ is said to be optimal if it satisfies $r^{\pi^*}(s) \geq  r^\pi(s)$ for all $s\in \mathcal{S}$ and $\pi\in \Pi$. Note that restricting ourselves to deterministic policies is actually without loss of generality because there always exists a deterministic optimal policy \cite{puterman2014markov}.

	In average reward RL, there are two main problems: the policy evaluation problem and the policy optimization problem. The policy evaluation problem involves estimating the quality of a given policy $\pi \in \Pi$, which is often used as a subroutine in the search for an optimal policy in the actor-critic framework \cite{ganesh2024accelerated,li2024stochastic}. A widely used method for solving this problem is average reward TD-learning \cite{tsitsiklis1999average}. Alternatively, one can directly find an optimal policy using Q-learning \cite{mahadevan1996average}. In the following, we show that our SA results can be applied to establish finite-sample bounds for average reward TD-learning with linear function approximation and for average reward Q-learning in the synchronous setting.
	
	\subsection{TD-Learning for Policy Evaluation}
	Consider the problem of evaluating a given policy $\pi \in \Pi$ using a sample trajectory generated by applying this policy to the MDP. Since the underlying model is an induced Markov reward process (MRP), for simplicity, we employ the notation $\mathcal{R}^\pi(s)= \mathcal{R}(s, \pi(s))$ for all $s$, and $P^\pi(s,s'):= p(s'|s,\pi(s))$ for all $(s,s')$. We make the following standard assumption on the induced MRP.
	\begin{assumption}
		\label{assumption:ergodic MC}
		The Markov chain $\{S_k\}$ associated with $P^\pi$ is irreducible and aperiodic.
	\end{assumption}
	
	Assumption \ref{assumption:ergodic MC} is standard in studying the policy evaluation problem \cite{tsitsiklis1997analysis,tsitsiklis1999average,bhandari2018finite} and guarantees that all states can be visited infinitely often in the sample trajectory. In fact, under Assumption \ref{assumption:ergodic MC}, the induced Markov chain $\{S_k\}$ with transition matrix $P^\pi$ has a unique stationary distribution $\mu^\pi\in\Delta(\mathcal{S})$ (where $\Delta(\mathcal{S})$ denotes the set of probability distributions supported on $\mathcal{S}$), and enjoys the geometric mixing property \cite{levin2017markov}.

	\vspace{3 mm}
	\noindent\textbf{The Bellman Equation.} Under Assumption \ref{assumption:ergodic MC}, the average reward $r^{\pi}\in\mathbb{R}^{|\mathcal{S}|}$ defined in Eq. \eqref{average reward} is independent of the initial state $s$. Specifically, we have $r^\pi(s) = r(\pi) := (\mu^{\pi})^\top\mathcal{R}^\pi$ for all $s \in \mathcal{S}$ \cite{puterman2014markov}. Let $V^\pi\in\mathbb{R}^{|\mathcal{S}|}$ be the differential value function of policy $\pi$ defined as $V^{\pi} = \sum^\infty_{k=0} \left(P^\pi \right)^k (\mathcal{R}^\pi - r(\pi) e)$,
	where $e \in \mathbb{R}^{|\mathcal{S}|}$ is the all-ones vector.
	Then, it is known that $V^\pi$ satisfies the Bellman equation \cite{puterman2014markov}, which is given as
	\begin{align}\label{eq:poisson_eq}
		v=\mathcal{R}^\pi + P^\pi v - r(\pi)e
	\end{align}
	Define the Bellman operator $\mathcal{T}^\pi(v) = \mathcal{R}^\pi - r(\pi)e + P^\pi v$
	for all $v\in\mathbb{R}^{|\mathcal{S}|}$.
	Then, unlike in the discounted setting, the solution to the Bellman equation $\mathcal{T}^\pi(v) = v$ is, in general, not unique. In fact, one can easily observe that any element from the set $\{V^{\pi} + ce \mid c \in \mathbb{R}\}$ is a solution of the Bellman equation. 
	
	As a side note, the average reward Bellman equation \eqref{eq:poisson_eq} is identical to the Poisson equation in Markov chain theory \cite{meyn2012markov, douc2018markov}. Therefore, our results on Markovian SA with a seminorm-contractive operator offer a direct pathway to solving a broader class of problems that involve solving the Poisson equation as a subroutine \cite{agrawal2024markov, jure2018, douc2024solvingpoissonequationusing, henderson2002approximating}.
	
	\vspace{3 mm}
	\noindent\textbf{Variants of the Bellman Operator.} Using the one-step Bellman operator $\mathcal{T}^\pi(\cdot)$, we define the multi-step Bellman operator and the $\lambda$-averaged Bellman operator in the following. The latter is important to introduce the TD$(\lambda)$ algorithm for policy evaluation. For any $m=0,1,2,\cdots$, let $\mathcal{T}^{m}(\cdot)$ be the $m$-step Bellman operator defined as
	\begin{align*}
		\mathcal{T}^m(v) :=\underbrace{\mathcal{T}\circ \mathcal{T} \circ\cdots \circ\mathcal{T}}_{m \text{ times}} (v)
		=  \sum^{m-1}_{k=0} (P^\pi )^k (\mathcal{R}^\pi -r({\pi}) e ) + (P^\pi )^m v,\quad \forall\,v\in\mathbb{R}^{|\mathcal{S}|}.
	\end{align*}
	For any $\lambda\in [0,1)$, let $\mathcal{T}^\lambda(\cdot)$ be the $\lambda$-averaged version of the $m$-step Bellman operator defined as 
	\begin{align}\label{def:T_lambda}
		\mathcal{T}^{(\lambda)}(v) := (1-\lambda)\sum^{\infty}_{m=0}\lambda^{m} \mathcal{T}^{m+1} (v),\quad \forall\,v\in\mathbb{R}^{|\mathcal{S}|}.
	\end{align}
	Note that the regular Bellman operator $\mathcal{T}(\cdot)$ corresponds to $\mathcal{T}^m(\cdot)$ with $m=1$ and $\mathcal{T}^{(\lambda)}(\cdot)$ with $\lambda=0$.
	It is easy to check that the set of fixed points of $\mathcal{T}^m(\cdot)$ (for any $m\geq 0$) and $\mathcal{T}^\lambda(\cdot)$ (for any $\lambda\in [0,1)$) is also $\{V^{\pi} + ce \mid c \in \mathbb{R}\}$.

	\vspace{3 mm}
	\noindent\textbf{Linear Function Approximation.} Observe that $V^\pi$ lives in the $|\mathcal{S}|$-dimensional Euclidean space. In most modern applications of RL, the size of the state space can be prohibitively large, making exact value function learning intractable \cite{sutton2018reinforcement}. To overcome this challenge, we consider employing linear function approximation, where the main idea is to approximate $V^\pi$ from a linear subspace of $\mathbb{R}^{|\mathcal{S}|}$. Specifically, let $\Phi\in\mathbb{R}^{|\mathcal{S}|\times d}$ (where $d \leq |\mathcal{S}|$) be a chosen feature matrix, where the columns (denoted by $\{\phi_i\in\mathbb{R}^{|\mathcal{S}|}\}_{1\leq i\leq d}$) are called basis vectors, and the rows (denoted by $\{\phi(s)\in\mathbb{R}^d\}_{s\in\mathcal{S}}$) are called features associated with the states. Our goal is to approximate $V^\pi$ from the linear subspace $W_\Phi:=\{V_\theta=\Phi\theta\mid \theta\in\mathbb{R}^d\}$, where $\theta\in\mathbb{R}^d$ is the weight vector. We assume that the matrix $\Phi$ has full column rank, which is without loss of generality because if some basis vector $\phi_i$ is a linear combination of the others, it can be disregarded without changing the approximation power of the function class. Additionally, we assume that $\|\phi(s)\|_2 \leq 1$ for all $s \in \mathcal{S}$, which can be ensured through feature normalization.

	\subsubsection{TD(\texorpdfstring{$\lambda$}{lambda}) with Linear Function Approximation}\label{subsubsec:TD_Algorithm}
	
	We consider the average reward TD($\lambda$) algorithm proposed in \cite{tsitsiklis1999average}, which is presented in Algorithm \ref{algo:TD}.

	\begin{algorithm}[ht]
		\begin{algorithmic}[1]
			\STATE \textbf{Input:} Integer $K\geq 0$, $\lambda \in [0,1)$,  initializations $z_{-1} = 0^d$, $r_0\in\mathbb{R}$, and $\theta_0\in\mathbb{R}^d$, and stepsizes $\{\alpha_k\}$ and $\{\beta_k\}$.
			\FOR{$k=0,1,\cdots,K-1$}
			\STATE Observe the tuple  $O_k = (S_k,\mathcal{R}^\pi(S_k),S_{k+1})$
			\STATE Update the eligibility trace: $z_k = \lambda z_{k-1} + \phi(S_k)$
			\STATE Update the average reward estimate: $r_{k+1} = r_k + \alpha_k (\mathcal{R}^\pi(S_k) - r_k)$
			\STATE Update the weight vector: $\theta_{k+1} = \theta_k + \beta_k z_k(\mathcal{R}^\pi(S_k) - r_k + \phi(S_{k+1})^\top \theta_k - \phi(S_k)^\top \theta_k)$
			\ENDFOR
			\STATE \textbf{Output:} $\theta_K$ and $r_K$.
		\end{algorithmic}
		\caption{Average Reward TD($\lambda$) with Linear Function Approximation} 
		\label{algo:TD}
	\end{algorithm} 
	
	Note that Algorithm \ref{algo:TD} is implemented based on a single trajectory of samples $\{O_k\}_{k \geq 0}$. Consequently, there is no need to repeatedly reset the system. Line 4 of Algorithm \ref{algo:TD} updates the eligibility traces, which is, in fact, a recursive method for computing $z_k = \sum_{t=0}^{k} \lambda^{k-t} \phi(S_t)$ (this expression coincides with \cite[Eq. (4)]{tsitsiklis1999average}) to facilitate implementation. In our analysis, we can disregard this recursive update equation and instead directly use $z_k = \sum_{t=0}^{k} \lambda^{k-t} \phi(S_t)$. The update equations for the reward estimate and the weight vector are standard for TD-learning in the average reward setting. Finally, while Algorithm \ref{algo:TD} involves two sequences of stepsizes $\{\alpha_k\}$ and $\{\beta_k\}$, we will choose $\alpha_k=c_\alpha \beta_k$ for some fixed $c_\alpha > 0$ to ensure that it is a single-timescale algorithm.
	
	\subsubsection{Finite-Sample Analysis}
	To study Algorithm~\ref{algo:TD}, we first reformulate it as a Markovian linear SA in the form of Eq. (\ref{algo:linear_SA}), then identify the appropriate seminorm and verify the corresponding assumptions, and finally apply Theorem \ref{thm:linear_SA}.

	\vspace{3 mm}
	\noindent\textbf{Markovian Linear SA Reformulation.}
	Let $Y_k = (S_k, S_{k+1}, z_k)$ for all $k \geq 0$. It is easily seen that $\{Y_k\}$ is a Markov chain, whose state space is denoted by $\mathcal{Y}$. Under Assumption \ref{assumption:ergodic MC}, the Markov chain $\{Y_k\}$ has a unique stationary distribution $\mu$, as shown in \cite[Section 6.3.3]{bertsekas1996neuro}. Define the functions $A:\mathcal{Y} \to\mathbb{R}^{(d+1)\times (d+1)}$ and $b:\mathcal{Y}\to\mathbb{R}^{d+1}$ as 
	\begin{align*}
		A(y) = A(s, s', z) = \begin{bmatrix}
			-c_\alpha & 0_{1\times d} \\
			-z & z \left(\phi(s')^\top - \phi(s)^\top \right)
		\end{bmatrix}, \quad \text{and}\quad 
		b(y) = b(s, s', z) = \mathcal{R}^\pi(s)\begin{bmatrix}c_\alpha \\
			z^\top\end{bmatrix},
	\end{align*}
	for all $y = (s, s', z) \in \mathcal{Y}$. Let $\Theta_k = [r_k ,
	\theta_k^\top]^\top \in \mathbb{R}^{d+1}$. The update equations for $r_k$ and $\theta_k$ (cf. Algorithm \ref{algo:TD} Lines $5$ and $6$) can be jointly written as
	\begin{align}\label{TD(lambda) matrix form}
		\Theta_{k+1} = \Theta_k + \beta_k\left(A(Y_k) \Theta_k + b(Y_k)\right).
	\end{align}
	The update equation above can be viewed as a Markovian linear SA algorithm for solving the equation $\bar{A}\Theta + \bar{b} = 0$, where
	\begin{align*}
		\bar{A}:=\,&\mathbb{E}_{Y\sim \mu} [A(Y)]= \begin{bmatrix}
			-c_\alpha & 0_{1 \times d}\\
			-\frac{1}{1-\lambda}\Phi^\top D e &\quad  \Phi^\top D \left(P^{(\lambda)}- I \right)\Phi
		\end{bmatrix},\\
		\bar{b}:=\,&\mathbb{E}_{Y\sim \mu} [b(Y)]=\begin{bmatrix}
			c_\alpha r(\pi)\\
			\Phi^\top D\sum_{m=0}^\infty(\lambda P^\pi)^m\mathcal{R}^\pi
		\end{bmatrix}.
	\end{align*}
	In the definition of $\bar{A}$, we denoted $P^{(\lambda)}=(1-\lambda)\sum_{m=0}^\infty\lambda^mP^{m+1}$ for simplicity of notation.
	See \cite{tsitsiklis1999average} for more details on computing the explicit expressions of $\bar{A}$ and $\bar{b}$.

	\vspace{3 mm}
	\noindent\textbf{Characterizing the Solution Set.} 
	In view of $\bar{A}$, $\bar{b}$, and the definition of $\mathcal{T}^{(\lambda)}(\cdot)$ from Eq. (\ref{def:T_lambda}), a solution $\Theta^*=[r_\infty,\theta^\top_\infty]^\top \in\mathbb{R}^{d+1}$ to the equation $\bar{A}\Theta+\bar{b}=0$ must satisfy $r_\infty=r(\pi)$ and 
	\begin{align*}
		\Phi^\top D \left(\mathcal{T}^{(\lambda)} \left(\Phi \theta_\infty \right)-\Phi\theta_\infty\right)=0,
	\end{align*}
	which can be equivalently written in the form of a projected Bellman equation
	\begin{align}\label{projected Bellman equation}
		\Phi \theta_\infty = \Pi_{D,W_{\Phi}} \mathcal{T}^{(\lambda)} \left(\Phi \theta_\infty \right).
	\end{align}
	Here in Eq. (\ref{projected Bellman equation}), $\Pi_{D,W_{\Phi}} = \Phi \left(\Phi^\top D \Phi \right)^{-1} \Phi^\top D$ denotes the projection operator onto the linear subspace $ W_{\Phi} = \{V_\theta = \Phi \theta \mid \theta \in \mathbb{R}^d\} $ with respect to the weighted $\ell_2$-norm $\| x \|_{D} = \sqrt{x^\top D x}$, where $ D = \text{diag}(\mu^\pi) \in \mathbb{R}^{|\mathcal{S}| \times |\mathcal{S}|} $.

	To characterize the solution set of Eq.~(\ref{projected Bellman equation}), we begin by noting that if the all-ones vector $ e $ belongs to the linear subspace $ W_{\Phi} $, then the fixed point of $ \Pi_{D,W_{\Phi}} \mathcal{T}^{(\lambda)}(\cdot) $ is not unique. This follows directly from the fact that the set of fixed points of $ \mathcal{T}^{(\lambda)}(\cdot) $ is $ \{V^{\pi} + c e \mid c \in \mathbb{R}\} $. The above observation motivates us to define the linear subspace 
	\begin{align}\label{def:theta_e}
		S_{\Phi, e} := \text{span}(\{\theta \in \mathbb{R}^d \mid \Phi \theta = e \})=
		\begin{cases}
			\{c \cdot \theta_e \mid c \in \mathbb{R}\}, & \exists\; \theta_e \in \mathbb{R}^d \text{ such that } \Phi \theta_e = e, \\
			\{0\}, & \text{otherwise}.
		\end{cases}
	\end{align}
	Define $E_{\Phi, e}$ as the orthogonal complement of $S_{\Phi, e}$ and let $ \Pi_{D,W_{E_{\Phi, e}}}(\cdot) $ be the projection operator onto the linear subspace $ W_{E_{\Phi, e}} := \{\Phi \theta \mid \theta \in E_{\Phi, e}\} \subseteq W_{\Phi} $ with respect to the norm $ \| \cdot \|_{D} $. Using the above notation, we characterize the set of solutions to the projected Bellman equation~(\ref{projected Bellman equation}) in the following lemma. See Appendix~\ref{pf:le:TD fixed points} for the proof.
	
	\begin{lemma}
		\label{le:TD fixed points}
		Under Assumption~\ref{assumption:ergodic MC}, the solution set of the projected Bellman equation~(\ref{projected Bellman equation}) is $ \mathcal{L}_{\Phi, e} := \theta^* + S_{\Phi, e} $, where $ \theta^* \in E_{\Phi, e} $ is the unique solution to the equation $ \Phi \theta = \Pi_{D,W_{E_{\Phi, e}}} \mathcal{T}^{(\lambda)} \left(\Phi \theta\right) $.
	\end{lemma}
	
	\begin{remark}
		Lemma \ref{le:TD fixed points} shows that the projected Bellman equation~(\ref{projected Bellman equation}) has a unique fixed point $ \theta^* $ when $ e \not \in W_{\Phi} $. In prior work \cite{tsitsiklis1999average,zhang2021average}, it was assumed that the columns of $ \Phi $ are independent of the all-ones vector $ e $. While this assumption guarantees $ e \not \in W_{\Phi} $, it is relatively restrictive and does not hold even in the tabular setting. In this paper, by working with seminorms, we eliminate the need for such a restrictive assumption.
	\end{remark}
	
	With the solution set $S_{\Phi, e}$ to the projected Bellman equation (\ref{projected Bellman equation}) specified, the overall solution set to the linear system of equations $\bar{A}\Theta+\bar{b}=0$ can be represented as $[r(\pi),{\theta^*}^\top ]^\top +E$, where $E=\{0\} \times S_{\Phi, e}:= \{[0, \theta^\top] \in \mathbb{R}^{d+1} \mid \theta \in S_{\Phi, e}\}$ is the kernel space of $\bar{A}$.
	
	\vspace{3 mm}
	\noindent\textbf{Verifying the Assumptions.} Our next step is to verify Assumptions \ref{as:lsa_Lyapunov} and \ref{as:lsa_all} needed to apply Theorem \ref{thm:linear_SA}. We start with Assumption \ref{as:lsa_Lyapunov}. Let $P\in\mathbb{R}^{(d+1)\times (d+1)}$ be the projection matrix onto the linear subspace $E^\perp$ (which is the orthogonal complement of $E$) with respect to the $\ell_2$-norm, i.e., $\arg\min_{\Theta'\in E^\perp}\|\Theta-\Theta'\|_2=P\Theta$. It is clear that the matrix $P$ is symmetric, idempotent, and positive semi-definite, with its kernel space being $E$. In addition, we have the following result.
	
	\begin{lemma}\label{le:TDLFA_Lyapunov}
		Suppose that Assumption \ref{assumption:ergodic MC} is satisfied. Then, we have
		\begin{align*}
			\bar{A}^\top P+P\bar{A}+\Delta P\leq 0,
		\end{align*}
		where $\Delta:=\min_{\|\theta\|_2=1, \theta\in E_{\Phi, e}} \theta^\top\Phi^\top D ( I - P^{(\lambda)} )\Phi\theta>0$.
	\end{lemma}
	
	\begin{remark}
		Although Lemma \ref{le:TDLFA_Lyapunov} is stated as a Lyapunov inequality, it is sufficient for our analysis of linear SA in Section \ref{sec:linear_SA} to go through.
	\end{remark}

	Let $p(\Theta)=\sqrt{\Theta^\top P\Theta}$, which is a seminorm with $\text{ker}(p)=E$. Since $P$ is a projection matrix, we have $p(\Theta)=\sqrt{\Theta^\top P\Theta}=\sqrt{\Theta^\top P^\top P\Theta}=\|P\Theta\|_2=\min_{\Theta'\in E}\|\Theta-\Theta'\|_2$. Therefore, the standard $\ell_2$-norm satisfies $p(\Theta)=\min_{\Theta'\in E}\|\Theta-\Theta'\|_2$. 
	
	Next, we verify Assumption \ref{as:lsa_all} of linear SA in the following lemma. The proof is presented in Appendix \ref{pf:le:TDlambda_assumptions}. 
	
	\begin{lemma}\label{le:TDlambda_assumptions}
		Suppose that Assumption \ref{assumption:ergodic MC} is satisfied and $c_\alpha\geq \Delta+1/(\Delta(1-\lambda)^2)$. Then, we have the following results.
		\begin{enumerate}[(1)]
			\item  For any $y\in\mathcal{Y}$, we have $\Theta\in E$ $\Rightarrow$ $A(y)\Theta \in E$.
			\item For $y \in\mathcal{Y}$, we have $\|A(y)\|\leq 2c_\alpha$ and $\|b(y)\|_2\leq 2c_\alpha$.
			\item There exists $C \geq 1$ such that $t_\delta \leq C \log(1/\delta)$, where
			\begin{align*}  
				t_\delta:=\min\left\{k\geq 0\,\middle|\,\sup_{y\in\mathcal{Y}}\|\bar{A} - \mathbb{E}[A(Y_k) \mid Y_0=y]\|_2 \leq 2c_\alpha\delta, \; \sup_{y\in\mathcal{Y}}\|\bar{b}-\mathbb{E}[b(Y_k)\mid Y_0=y]\|_2 \leq 2c_\alpha\delta\right\}. 
			\end{align*}
		\end{enumerate}
	\end{lemma}
	
	Now that all assumptions are verified, we are ready state the finite-sample bounds of the average reward TD$(\lambda)$ algorithm. The proof of the following theorem is omitted, as it follows directly by applying Theorem \ref{thm:linear_SA} to Algorithm \ref{algo:TD}.
	
	\begin{theorem}
		\label{theorem:TD finite time bounds}
		Consider $\{\Theta_k\}$ generated by Algorithm \ref{algo:TD}. Suppose that Assumption \ref{assumption:ergodic MC} is satisfied and $c_\alpha\geq \Delta+1/(\Delta(1-\lambda)^2)$. Then, we have the following convergence bounds for all $k \geq K:=\min\{k:k\geq t_k:=t_{\beta_k}\}$.
		\begin{itemize}[(1)]
			\item When $\beta_k \equiv \beta$ with properly chosen $\beta$, we have
			\begin{align*}
				\mathbb{E}[p(\Theta_{k}-\Theta^*)^2] \leq  \varrho_1 (1 - \Delta\beta/2)^{k-t_\beta} + \frac{2\varrho_2}{\Delta}\beta t_\beta,
			\end{align*}
			where $\varrho_1=(p(\Theta_0)+p(\Theta_0-\Theta^*)+1)^2$ and $\varrho_2=456c_\alpha^2( p(\Theta^*)+1)^2$.
			\item When $\beta_k = \beta/(k+h)$ with $\beta>2/\Delta$ and properly chosen $h$, we have
			\begin{align*}
				\mathbb{E}[p(\Theta_{k}-\Theta^*)^2] \leq  \varrho_1 \left(\frac{K+h}{k + h}\right)^{\beta \Delta/2} + \frac{16e\beta^2\varrho_2}{\Delta\beta - 2} \frac{t_k}{k + h}.
			\end{align*}
		\end{itemize}
	\end{theorem}
	
	Theorem \ref{theorem:TD finite time bounds} (1) analyzes Algorithm \ref{algo:TD} with properly chosen constant stepsizes. In this case, the iterates $\Theta_k$ do not converge to any solution of Eq.~\eqref{projected Bellman equation} due to the presence of noise variance. However, the convergence bound shows that the expected distance of $\Theta_k$ to the solution set $\mathcal{L}_{\Phi, e}$ decreases exponentially fast until it reaches a level that depends on the chosen stepsize. Theorem \ref{theorem:TD finite time bounds} (2) studies Algorithm \ref{algo:TD} with a carefully selected decaying stepsize sequence. Under this setting, the iterates $\Theta_k$ achieve an $\tilde{\mathcal{O}}\left(1 / k\right)$ convergence rate to the solution set $\mathcal{L}_{\Phi, e}$.

	\subsection{Q-Learning for Policy Optimization}
	We now consider finding an optimal policy for the average reward RL problem through Q-learning. The following assumption is standard in this setting \cite[Section 8.4]{puterman2014markov}.
	
	\begin{assumption}
		\label{assumption:unichain mdp}
		For any deterministic policy $\pi$, the induced Markov chain $\{S_k\}$ is a unichain, i.e., it consists of a single recurrent class plus a set of transient states.
	\end{assumption}

	Under this assumption, all stationary policies have constant average reward function \cite{bertsekas2011dynamic,puterman2014markov}. As a result, the optimal value $r^*\in\mathbb{R}^{|\mathcal{S}|}$ is also a constant function. Let $Q^*\in\mathbb{R}^{|\mathcal{S}||\mathcal{A}|}$ be the optimal differential Q-function defined as 
	\begin{align*}
		Q^*(s,a)=\mathbb{E}\left[\sum^\infty_{k=0}\mathcal{R}(S_k,\pi^*(S_k))-r^*e\,\middle|\,S_0=s,A_0=a\right].
	\end{align*}
	Then, it is known that $Q^*$ solves the Bellman optimality equation:
	\begin{align}\label{eq:Q fixed point equation}
		\mathcal{H}(Q) - r^* e = Q,
	\end{align}
	where $\mathcal{H}: \mathbb{R}^{|\mathcal{S}| |\mathcal{A}|} \to\mathbb{R}^{|\mathcal{S}||\mathcal{A}|}$ is defined as
	\begin{align}\label{eq:Q_Bellman}
		[\mathcal{H}(Q)](s,a) = \mathcal{R}(s,a) + \sum_{s' \in \mathcal{S}} p(s'|s,a) \max_{a' \in \mathcal{A}} Q(s', a'), \quad \forall\, (s,a)\in \mathcal{S}\times \mathcal{A}.
	\end{align}
	In addition, any policy $\pi$ satisfies  $\pi(s)\in\arg\max_{a\in\mathcal{A}}Q^*(s,a)$ for all $s\in\mathcal{S}$ is an optimal policy. Therefore, the problem reduces to solving the Bellman optimality equation (\ref{eq:Q fixed point equation}) to find the optimal differential Q-function $Q^*$. However, unlike in the discounted setting, since the operator $\mathcal{H}(\cdot)$ is, in general, not a norm-contraction mapping, the solution to Eq. (\ref{eq:Q fixed point equation}) is not unique. In fact, any point from $\{Q^*+ce | ~c \in \mathbb{R}\}$ is a solution to Eq. (\ref{eq:Q fixed point equation}) \cite{puterman2014markov}. Fortunately, since $\arg\max_{a\in\mathcal{A}}Q^*(s,a)=\arg\max_{a\in\mathcal{A}}\{Q^*(s,a)+c\}$ for any $c\in\mathbb{R}$, to find an optimal policy, it is enough to find $Q^*$ up to an additive constant. 
	
	Without additional assumptions, the operator defined in Eq. (\ref{eq:Q_Bellman}) is, in general, not even a seminorm contraction mapping. However, a multi-step variant of $\mathcal{H}(\cdot)$ is shown to be a span-seminorm contraction mapping \cite{puterman2014markov}. Specifically, for any $J\geq 1$, let $\mathcal{H}^{(J)}$ $:\mathbb{R}^{|\mathcal{S}| |\mathcal{A}|}\to\mathbb{R}^{|\mathcal{S}||\mathcal{A}|}$ be defined as
	\begin{align}\label{J-step Bellman}
		&[\mathcal{H}^{(J)}(Q)](s,a) := \mathbb{E}\left[\mathcal{R}(S_0,A_0)+\sum_{k=1}^{J-1}\mathcal{R}(S_k,\mu_Q(S_k))+Q(S_J,\mu_Q(S_J))\,\middle|\,S_0=s,A_0=a\right]
	\end{align}
	for all $Q\in\mathbb{R}^{|\mathcal{S}||\mathcal{A}|}$,
	where $\mu_{Q}(s)\in\argmax_{a\in\mathcal{A}}Q(s,a)$ for all $s\in\mathcal{S}$. Note that the $J$-step operator $\mathcal{H}^{J}(\cdot)$ is \textit{not} equal to the operator $\mathcal{H}(\cdot)$ being repeatedly applied for $J$ times because the actions are always chosen according to $\mu_Q(\cdot)$.
	Observe that any solution for Eq. (\ref{eq:Q fixed point equation}) is also a solution for the fixed-point equation:
	\begin{align}\label{J-step fixed point eq}
		\mathcal{H}^{(J)}(Q)-r^*e=Q,
	\end{align}
	and vice versa. 
	
	\subsubsection{Average Reward Q-Learning}
	In this section, we present the $J$-step synchronous Q-learning algorithm in Algorithm \ref{algo:Q-learning}, which is developed as an SA algorithm for solving the fixed-point equation (\ref{J-step fixed point eq}).

	\begin{algorithm}[ht]
		\begin{algorithmic}[1]
			\STATE \textbf{Input:} Integer $K\geq 0$, initialization $Q_0\in\mathbb{R}^{|\mathcal{S}||\mathcal{A}|}$, and stepsizes $\{\alpha_k\}$.
			\FOR{$k=0,1,\cdots,K-1$}
			\STATE Compute $\mu_{k}(s) \in\argmax_{a\in\mathcal{A}} Q_k(s, a)$ for all $s \in \mathcal{S}$.
			\FOR{$(s,a) \in \mathcal{S} \times \mathcal{A}$}
			\STATE Sample $S^1 \sim p(\cdot|s,a), S^2 \sim p(\cdot|S^1,\mu_k(S^1)), \dots, S^J \sim p(\cdot|S^{J-1},\mu_k(S^{J-1})).$
			\STATE $Q_{k+1}(s,a)=Q_{k}(s,a) + \alpha_k \left(\mathcal{R}(s, a) +\sum_{j=1}^{J-1} \mathcal{R}(S^j, \mu_k(S^j)) + Q_{k}(S^J, \mu_k(S^J)) - Q_k(s,a)\right).$
			\ENDFOR
			\ENDFOR
			\STATE \textbf{Output:} $Q_K$.
		\end{algorithmic}
		\caption{$J$-Step Synchronous Q-Learning} 
		\label{algo:Q-learning}
	\end{algorithm}

	One can also implement Q-learning asynchronously using samples from a single trajectory of the Markov chain generated by applying a behavior policy to the MDP. However, due to the nonlinear nature of the Bellman optimality equation and the lack of norm-contractive mapping, theoretically characterizing the convergence behavior of asynchronous Q-learning in the average reward setting is significantly more challenging, and remains a direction for future research.

	In many other variants of Q-learning, such as RVI Q-Learning \cite{abounadi2001learning}, the update involves subtracting a Lipschitz function, $ f(Q_k) $, from the temporal difference in all components of $ Q_k $. This ensures the almost sure convergence of the iterates to a special point satisfying $ f(Q^*) = r^* $. In our proposed Algorithm~\ref{algo:Q-learning}, we demonstrate convergence in the seminorm sense to a fixed point in the kernel space. However, a similar subtraction of $ f(Q_k) $ can also be incorporated into Algorithm~\ref{algo:Q-learning} to guarantee convergence to the special point within the kernel space.

	\subsubsection{Finite-Sample Analysis}
	To formulate Algorithm \ref{algo:Q-learning} Line $6$ in the form of the SA algorithm presented in Eq. (\ref{algo:SA}), for any state-action pair $(s,a)$ and $k\geq 0$, let $w_k(s,a)$ be defined as
	\begin{align*}
		\omega_k(s,a)=[\mathcal{H}^J(Q_k)](s,a)-\mathcal{R}(s,a)-\sum_{i=1}^{J-1}\mathcal{R}(S^i,\mu_{Q_k}(S^i)) - Q_k(S^J, \mu_{Q_k}(S^J)).
	\end{align*}
	Then, the update equation for $Q_{k+1}$ can be written compactly as
	\begin{align}\label{eq:Q_update}
		Q_{k+1}
		=Q_k+\alpha_k(\mathcal{H}^{(J)}(Q_k)-Q_k+\omega_k),\quad\forall\,k\geq 0.
	\end{align}
	Note that there is no Markovian noise in the previous equation as we are performing synchronous updates. To apply our results on seminorm-contractive SA to Algorithm \ref{algo:Q-learning}, we next verify in the following lemma that Assumptions \ref{as:seminorm_contraction} and \ref{as:sa_all} are satisfied in the context of Q-learning.

	\begin{lemma}\label{lem:ass_ver_operator}
		Suppose that Assumption \ref{assumption:unichain mdp} is satisfied. Then, the following statements hold.
		\begin{enumerate}[(1)]
			\item There exist $J\geq 1$ and a constant $\gamma\in [0,1)$ such that
			\begin{align*}
				p_{\text{span}}(\mathcal{H}^{(J)}(Q_1)-\mathcal{H}^{(J)}(Q_2))\leq \gamma p_{\text{span}}(Q_1-Q _2),\quad \forall\,Q_1,Q_2\in\mathbb{R}^{|\mathcal{S}||\mathcal{A}|}.
			\end{align*}
			\item The operator $\mathcal{H}^{(J)}(\cdot)$ satisfies $p_{\text{span}}(\mathcal{H}^{(J)}(0))\leq Jp_{\text{span}}(\mathcal{R})$.
			\item The stochastic process $\{\omega_k\}$ satisfies $\mathbb{E}[\omega_{k}|\mathcal{F}_k]=0$ and $p_{\text{span}}(\omega_k)\leq 2p_{\text{span}}(Q_k)+2Jp_{\text{span}}(\mathcal{R})$ for all $k\geq 0$, where $\mathcal{F}_k$ is the $\sigma$-algebra generated by $\{Q_0,\dots,Q_k,\omega_0,\dots,\omega_{k-1}\}$.
		\end{enumerate}
	\end{lemma}
	
	Lemma \ref{lem:ass_ver_operator} (1) is restated from \cite[Section 8.5.4]{puterman2014markov}. The proof of Lemma \ref{lem:ass_ver_operator} (2) and (3) trivially follows from the definitions of $\mathcal{H}^J(\cdot)$ and $w_k$, and the triangle inequality, hence is omitted. Now that all the assumptions needed to apply Theorem \ref{thm:SA_finite} are verified, we have the following finite-sample bounds for $J$-step synchronous Q-learning.
	
	\begin{theorem}\label{thm:Q_learning_ROC}
		Consider $\{Q_k\}$ generated by Algorithm \ref{algo:Q-learning}. 
		\begin{enumerate}[(1)]
			\item When $\alpha_k \equiv \alpha \leq \frac{(1 - \gamma)^2}{640e\log\left(|\mathcal{S}| |\mathcal{A}|\right)}$, we have for all $k\geq 0$ that
			\begin{equation*}
				\begin{split}
					\mathbb{E}[p_{\text{span}}(Q_{k}-Q^*)^2] \leq c_{Q,1} \left(1 -  \frac{1-\gamma}{2}\alpha\right)^{k} + c_{Q,2}\frac{\log\left(|\mathcal{S}| |\mathcal{A}|\right)}{(1-\gamma)^2} \alpha,
				\end{split}
			\end{equation*}
			where $c_{Q,1} = 3(p_{\text{span}}(Q_0-Q^*)+p_{\text{span}}(Q_0)+1)^2$ and $c_{Q,2} = 912e(p_{\text{span}}(Q^*)+ J)^2$.
			\item When $\alpha_k =\alpha/(k + h)$ with $\alpha = 4/(1-\gamma)$ and $h = 640e\log\left(|\mathcal{S}| |\mathcal{A}|\right)/(1 - \gamma)^3$, we have
			\begin{equation*}
				\begin{split}
					\mathbb{E}[p_{\text{span}}(Q_{k}-Q^*)^2] \leq 8192e^2\left(p_{\text{span}}(Q_0-Q^*)+2p_{\text{span}}(Q_0)+J^2\right)\frac{\log\left(|\mathcal{S}| |\mathcal{A}|\right)}{(1-\gamma)^3 k}.
				\end{split}
			\end{equation*}
		\end{enumerate}
		
	\end{theorem}
	
	As observed in Theorem \ref{thm:SA_finite}, when constant stepsizes are used, the error converges exponentially fast to a ball (measured by the span seminorm) centered at $Q^*$ with radius proportional to the stepsize $\alpha$. When using $\mathcal{O}(1/k)$ diminishing stepsizes, the error converges at a rate of $\mathcal{O}(1/k)$. Recently, the authors of \cite{bravo2024stochastic} studied synchronous RVI Q-learning as a stochastically perturbed version of the Krasnoselski–Mann iteration for solving fixed points of non-expansive operators. They achieved a polynomial rate of convergence under a different error metric. While their convergence rate seems to be worse compared to that in Theorem \ref{thm:Q_learning_ROC}, since the algorithm and the convergence metric are different, it is hard to conduct a quantitative comparison. There is also another line of work establishing regret bounds for variants of average-reward Q-learning (see \cite{agrawal2024optimistic} and the references therein), which differs in focus from our work on providing last-iterate convergence rates.
	
	We remark that our results can be applied to Q-learning for discounted MDPs to obtain sample complexity guarantees that remain uniformly bounded for any discount factor \cite{devraj2021q}. To illustrate, in the discounted setting, it has been shown in the literature that the sample complexity depends polynomially on the effective horizon $1/(1-\gamma')$, where $\gamma'$ is the discount factor \cite{chen2021lyapunov,li2024q}. Therefore, as $\gamma'$ approaches 1, the sample complexity goes to infinity. However, note that the motivation for Q-learning (in both the average-reward and discounted settings) is that, once we find $Q^*$, we can compute an optimal policy via $\pi^*(s) \in \arg\max_{a \in \mathcal{A}} Q^*(s,a)$ for all $s\in\mathcal{S}$.
	From this formula, to obtain an optimal policy, it is not necessary to find the exact $Q^*$. Indeed, any $Q$ that differs from $Q^*$ by a constant multiple of the all-ones vector induces the same optimal policy. Hence, instead of aiming to make $\|Q - Q^*\|_\infty$ (the $\ell_\infty$-norm is a standard metric for Q-learning) small, it suffices to make $p_{\text{span}}(Q - Q^*)$ small (recall that the kernel of the span seminorm is exactly the space spanned by the all-ones vector). By leveraging its seminorm-contractive property in conjunction with the discount factor, we can obtain sample complexity guarantees that are uniformly bounded for any discount factor. This finding is consistent with the recent result in \cite{devraj2021q}.
	
	\section{Conclusion}\label{sec:conclusion}
	In this work, we focus on solving seminorm fixed-point equations. Assuming that the operator is seminorm-contractive, we first establish a fixed-point theorem and then present the finite-sample analysis of the associated Markovian SA algorithm. An extensive case study is provided when the operator is linear, which leads to seminorm Lyapunov stability theorems (in the deterministic setting) and Markovian linear SA without the Hurwitzness assumption (in the stochastic setting). We demonstrate our theoretical findings in the context of average reward RL—a more challenging setting than the discounted setting due to the absence of a discount factor—and provide finite-sample guarantees for TD($\lambda$) with linear function approximation and synchronous Q-learning.

	\bibliographystyle{apalike}
	\bibliography{references}

\begin{thebibliography}{}

\bibitem[Abounadi et~al., 2001]{abounadi2001learning}
Abounadi, J., Bertsekas, D., and Borkar, V.~S. (2001).
\newblock {Learning algorithms for Markov decision processes with average
  cost}.
\newblock {\em SIAM Journal on Control and Optimization}, 40(3):681--698.

\bibitem[Agrawal and Agrawal, 2024]{agrawal2024optimistic}
Agrawal, P. and Agrawal, S. (2024).
\newblock {Optimistic Q-learning for average reward and episodic reinforcement
  learning}.
\newblock {\em Preprint arXiv:2407.13743}.

\bibitem[Agrawal et~al., 2024]{agrawal2024markov}
Agrawal, S., Maguluri, S.~T., et~al. (2024).
\newblock Markov chain variance estimation: A stochastic approximation
  approach.
\newblock {\em Preprint arXiv:2409.05733}.

\bibitem[Banach, 1922]{banach1922operations}
Banach, S. (1922).
\newblock Sur les op{\'e}rations dans les ensembles abstraits et leur
  application aux {\'e}quations int{\'e}grales.
\newblock {\em Fund. math}, 3(1):133--181.

\bibitem[Bartels and Stewart, 1972]{bartels1972algorithm}
Bartels, R.~H. and Stewart, G.~W. (1972).
\newblock Solution of the matrix equation {$AX+ XB= C$}.
\newblock {\em Communications of the ACM}, 15(9):820--826.

\bibitem[Beck, 2017]{beck2017first}
Beck, A. (2017).
\newblock {\em First-order methods in optimization}.
\newblock SIAM.

\bibitem[Benveniste et~al., 2012]{benveniste2012adaptive}
Benveniste, A., M{\'e}tivier, M., and Priouret, P. (2012).
\newblock {\em Adaptive algorithms and stochastic approximations}, volume~22.
\newblock Springer Science \& Business Media.

\bibitem[Bertsekas, 2007]{bertsekas2011dynamic}
Bertsekas, D.~P. (2007).
\newblock {\em Dynamic Programming and Optimal Control, Vol. II}.
\newblock Athena Scientific, 3rd edition.

\bibitem[Bertsekas and Tsitsiklis, 1996]{bertsekas1996neuro}
Bertsekas, D.~P. and Tsitsiklis, J.~N. (1996).
\newblock {\em Neuro-dynamic programming}.
\newblock Athena Scientific.

\bibitem[Bhandari et~al., 2018]{bhandari2018finite}
Bhandari, J., Russo, D., and Singal, R. (2018).
\newblock A finite time analysis of temporal difference learning with linear
  function approximation.
\newblock In {\em Conference on learning theory}, pages 1691--1692. PMLR.

\bibitem[Borkar, 2009]{borkar2009stochastic}
Borkar, V.~S. (2009).
\newblock {\em Stochastic approximation: A dynamical systems viewpoint},
  volume~48.
\newblock Springer.

\bibitem[Borkar and Meyn, 2000]{borkar2000ode}
Borkar, V.~S. and Meyn, S.~P. (2000).
\newblock The {ODE} method for convergence of stochastic approximation and
  reinforcement learning.
\newblock {\em SIAM Journal on Control and Optimization}, 38(2):447--469.

\bibitem[Bottou, 2010]{bottou2010large}
Bottou, L. (2010).
\newblock Large-scale machine learning with stochastic gradient descent.
\newblock In {\em Proceedings of COMPSTAT'2010: 19th International Conference
  on Computational StatisticsParis France, August 22-27, 2010 Keynote, Invited
  and Contributed Papers}, pages 177--186. Springer.

\bibitem[Bourbaki, 2013]{bourbaki2013topological}
Bourbaki, N. (2013).
\newblock {\em Topological vector spaces: Chapters 1--5}.
\newblock Springer Science \& Business Media.

\bibitem[Boyd and Vandenberghe, 2004]{boyd2004convex}
Boyd, S.~P. and Vandenberghe, L. (2004).
\newblock {\em Convex optimization}.
\newblock Cambridge university press.

\bibitem[Bravo and Cominetti, 2024]{bravo2024stochastic}
Bravo, M. and Cominetti, R. (2024).
\newblock {Stochastic fixed-point iterations for nonexpansive maps: Convergence
  and error bounds}.
\newblock {\em SIAM Journal on Control and Optimization}, 62(1):191--219.

\bibitem[Chandak et~al., 2022]{chandak2022concentration}
Chandak, S., Borkar, V.~S., and Dodhia, P. (2022).
\newblock Concentration of contractive stochastic approximation and
  reinforcement learning.
\newblock {\em Stochastic Systems}, 12(4):411--430.

\bibitem[Chen et~al., 2024]{chen2021lyapunov}
Chen, Z., Maguluri, S.~T., Shakkottai, S., and Shanmugam, K. (2024).
\newblock {A Lyapunov theory for finite-sample guarantees of Markovian
  stochastic approximation}.
\newblock {\em Operations Research}, 72(4):1352--1367.

\bibitem[Chen et~al., 2022]{chen2019finite}
Chen, Z., Zhang, S., Doan, T.~T., Clarke, J.-P., and Maguluri, S.~T. (2022).
\newblock Finite-sample analysis of nonlinear stochastic approximation with
  applications in reinforcement learning.
\newblock {\em Automatica}, 146:110623.

\bibitem[Conway, 2019]{conway2019course}
Conway, J.~B. (2019).
\newblock {\em A course in functional analysis}, volume~96.
\newblock Springer.

\bibitem[Davydov et~al., 2022]{Bullo2022}
Davydov, A., Jafarpour, S., and Bullo, F. (2022).
\newblock {Non-Euclidean} contraction theory for robust nonlinear stability.
\newblock {\em IEEE Transactions on Automatic Control}, 67(12):6667--6681.

\bibitem[De~Pasquale et~al., 2023]{Bullo2023}
De~Pasquale, G., Smith, K.~D., Bullo, F., and Valcher, M.~E. (2023).
\newblock Dual seminorms, ergodic coefficients and semicontraction theory.
\newblock {\em IEEE Transactions on Automatic Control}, 69(5):3040--3053.

\bibitem[Devraj and Meyn, 2021]{devraj2021q}
Devraj, A.~M. and Meyn, S.~P. (2021).
\newblock Q-learning with uniformly bounded variance.
\newblock {\em IEEE Transactions on Automatic Control}, 67(11):5948--5963.

\bibitem[Douc et~al., 2022]{douc2024solvingpoissonequationusing}
Douc, R., Jacob, P.~E., Lee, A., and Vats, D. (2022).
\newblock {Solving the Poisson equation using coupled Markov chains}.
\newblock {\em Preprint arXiv:2206.05691}.

\bibitem[Douc et~al., 2018]{douc2018markov}
Douc, R., Moulines, E., Priouret, P., Soulier, P., Douc, R., Moulines, E.,
  Priouret, P., and Soulier, P. (2018).
\newblock {\em Markov chains: Basic definitions}.
\newblock Springer.

\bibitem[Fudenberg and Tirole, 1991]{fudenberg1991game}
Fudenberg, D. and Tirole, J. (1991).
\newblock {\em Game theory}.
\newblock MIT press.

\bibitem[Ganesh and Aggarwal, 2024]{ganesh2024accelerated}
Ganesh, S. and Aggarwal, V. (2024).
\newblock An accelerated multi-level {Monte Carlo} approach for average reward
  reinforcement learning with general policy parametrization.
\newblock {\em arXiv e-prints}, pages arXiv--2407.

\bibitem[Gosavi, 2004]{gosavi2004reinforcement}
Gosavi, A. (2004).
\newblock Reinforcement learning for long-run average cost.
\newblock {\em European Journal of Operational Research}, 155(3):654--674.

\bibitem[Haddad and Chellaboina, 2008]{haddad2008nonlinear}
Haddad, W.~M. and Chellaboina, V. (2008).
\newblock {\em Nonlinear dynamical systems and control: A Lyapunov-based
  approach}.
\newblock Princeton University Press.

\bibitem[Halmos, 2017]{halmos2017finite}
Halmos, P.~R. (2017).
\newblock {\em Finite-dimensional vector spaces}.
\newblock Courier Dover Publications.

\bibitem[Henderson and Glynn, 2002]{henderson2002approximating}
Henderson, S.~G. and Glynn, P.~W. (2002).
\newblock Approximating martingales for variance reduction in {Markov} process
  simulation.
\newblock {\em Mathematics of Operations Research}, 27(2):253--271.

\bibitem[Hinrichsen and Pritchard, 2005]{hinrichsen2005mathematical}
Hinrichsen, D. and Pritchard, A.~J. (2005).
\newblock {\em Mathematical systems theory I: Modeling, state space analysis,
  stability and robustness}, volume~48.
\newblock Springer Science \& Business Media.

\bibitem[Jafarpour et~al., 2021]{jafarpour2022}
Jafarpour, S., Cisneros-Velarde, P., and Bullo, F. (2021).
\newblock Weak and semi-contraction for network systems and diffusively coupled
  oscillators.
\newblock {\em IEEE Transactions on Automatic Control}, 67(3):1285--1300.

\bibitem[Khalil, 2002]{khalil2002nonlinear}
Khalil, H. (2002).
\newblock {\em {Nonlinear Systems}}.
\newblock Pearson Education. Prentice Hall.

\bibitem[Khalil, 2009]{khalil2009lyapunov}
Khalil, H.~K. (2009).
\newblock Lyapunov stability.
\newblock {\em Control systems, robotics and automation}, 12:115.

\bibitem[Kushner and Yin, 2003]{kushner2003stochastic}
Kushner, H. and Yin, G.~G. (2003).
\newblock {\em Stochastic approximation and recursive algorithms and
  applications}, volume~35.
\newblock Springer Science \& Business Media.

\bibitem[Lakshminarayanan and Szepesvari, 2018]{lakshminarayanan2018linear}
Lakshminarayanan, C. and Szepesvari, C. (2018).
\newblock {Linear stochastic approximation: How far does constant step-size and
  iterate averaging go?}
\newblock In {\em International Conference on Artificial Intelligence and
  Statistics}, pages 1347--1355. PMLR.

\bibitem[Lan, 2020]{lan2020first}
Lan, G. (2020).
\newblock {\em First-order and stochastic optimization methods for machine
  learning}, volume~1.
\newblock Springer.

\bibitem[Levin and Peres, 2017]{levin2017markov}
Levin, D.~A. and Peres, Y. (2017).
\newblock {\em Markov chains and mixing times}, volume 107.
\newblock American Mathematical Soc.

\bibitem[Li et~al., 2024a]{li2024q}
Li, G., Cai, C., Chen, Y., Wei, Y., and Chi, Y. (2024a).
\newblock {Is Q-learning minimax optimal? A tight sample complexity analysis}.
\newblock {\em Operations Research}, 72(1):222--236.

\bibitem[Li et~al., 2024b]{li2024stochastic}
Li, T., Wu, F., and Lan, G. (2024b).
\newblock {Stochastic first-order methods for average-reward Markov decision
  processes}.
\newblock {\em Mathematics of Operations Research}.

\bibitem[Li et~al., 2018]{li2018algorithmic}
Li, Y., Ma, T., and Zhang, H. (2018).
\newblock Algorithmic regularization in over-parameterized matrix sensing and
  neural networks with quadratic activations.
\newblock In {\em Conference On Learning Theory}, pages 2--47. PMLR.

\bibitem[Lohmiller and Slotine, 2000]{Lohmiller2000}
Lohmiller, W. and Slotine, J.-J. (2000).
\newblock Control system design for mechanical systems using contraction
  theory.
\newblock {\em IEEE Transactions on Automatic Control}, 45(5):984--989.

\bibitem[Lohmiller and Slotine, 1998]{lohmiller1998contraction}
Lohmiller, W. and Slotine, J.-J.~E. (1998).
\newblock On contraction analysis for non-linear systems.
\newblock {\em Automatica}, 34(6):683--696.

\bibitem[Mahadevan, 1996]{mahadevan1996average}
Mahadevan, S. (1996).
\newblock Average reward reinforcement learning: Foundations, algorithms, and
  empirical results.
\newblock {\em Machine learning}, 22(1):159--195.

\bibitem[Manchester et~al., 2018]{Manchester2018}
Manchester, I.~R., Tang, J.~Z., and Slotine, J.-J.~E. (2018).
\newblock Unifying robot trajectory tracking with control contraction metrics.
\newblock {\em Robotics Research: Volume 2}, pages 403--418.

\bibitem[Meyn and Tweedie, 2012]{meyn2012markov}
Meyn, S.~P. and Tweedie, R.~L. (2012).
\newblock {\em Markov chains and stochastic stability}.
\newblock Springer Science \& Business Media.

\bibitem[Mijatović and Vogrinc, 2018]{jure2018}
Mijatović, A. and Vogrinc, J. (2018).
\newblock {On the Poisson equation for Metropolis-Hastings chains}.
\newblock {\em Bernoulli}, 24(3):2401 -- 2428.

\bibitem[Mou et~al., 2022]{mou2022optimal}
Mou, W., Khamaru, K., Wainwright, M.~J., Bartlett, P.~L., and Jordan, M.~I.
  (2022).
\newblock Optimal variance-reduced stochastic approximation in {Banach} spaces.
\newblock {\em Preprint arXiv:2201.08518}.

\bibitem[Mou et~al., 2020]{mou2020linear}
Mou, W., Li, C.~J., Wainwright, M.~J., Bartlett, P.~L., and Jordan, M.~I.
  (2020).
\newblock {On linear stochastic approximation: Fine-grained Polyak-Ruppert and
  non-asymptotic concentration}.
\newblock In {\em Conference on Learning Theory}, pages 2947--2997. PMLR.

\bibitem[Narici and Beckenstein, 2010]{narici2010topological}
Narici, L. and Beckenstein, E. (2010).
\newblock {\em Topological vector spaces}.
\newblock Chapman and Hall/CRC.

\bibitem[Parikh et~al., 2014]{parikh2014proximal}
Parikh, N., Boyd, S., et~al. (2014).
\newblock Proximal algorithms.
\newblock {\em Foundations and trends{\textregistered} in Optimization},
  1(3):127--239.

\bibitem[PARKS, 1992]{lyapunov}
PARKS, P.~C. (1992).
\newblock {A. M. Lyapunov's stability theory—100 years on}.
\newblock {\em IMA Journal of Mathematical Control and Information},
  9(4):275--303.

\bibitem[Pham et~al., 2009]{pham2009contraction}
Pham, Q.-C., Tabareau, N., and Slotine, J.-J. (2009).
\newblock A contraction theory approach to stochastic incremental stability.
\newblock {\em IEEE Transactions on Automatic Control}, 54(4):816--820.

\bibitem[Puterman, 2014]{puterman2014markov}
Puterman, M.~L. (2014).
\newblock {\em Markov decision processes: Discrete stochastic dynamic
  programming}.
\newblock John Wiley \& Sons.

\bibitem[Qu and Wierman, 2020]{qu2020finite}
Qu, G. and Wierman, A. (2020).
\newblock Finite-time analysis of asynchronous stochastic approximation and
  {Q}-learning.
\newblock In {\em Conference on Learning Theory}, pages 3185--3205. PMLR.

\bibitem[Robbins and Monro, 1951]{robbins1951stochastic}
Robbins, H. and Monro, S. (1951).
\newblock A stochastic approximation method.
\newblock {\em The annals of mathematical statistics}, pages 400--407.

\bibitem[Schaefer, 1971]{schaefer1971locally}
Schaefer, H.~H. (1971).
\newblock Locally convex topological vector spaces.
\newblock In {\em Topological Vector Spaces}, pages 36--72. Springer.

\bibitem[Shah et~al., 2020]{shah2020sample}
Shah, D., Song, D., Xu, Z., and Yang, Y. (2020).
\newblock Sample efficient reinforcement learning via low-rank matrix
  estimation.
\newblock {\em Advances in Neural Information Processing Systems},
  33:12092--12103.

\bibitem[Srikant and Ying, 2019]{srikant2019finite}
Srikant, R. and Ying, L. (2019).
\newblock Finite-time error bounds for linear stochastic approximation and
  {TD}-learning.
\newblock In {\em Conference on Learning Theory}, pages 2803--2830. PMLR.

\bibitem[Sutton and Barto, 2018]{sutton2018reinforcement}
Sutton, R. and Barto, A. (2018).
\newblock {\em Reinforcement Learning: An Introduction}.
\newblock Adaptive Computation and Machine Learning series. MIT Press.

\bibitem[Tsitsiklis and Van~Roy, 1997]{tsitsiklis1997analysis}
Tsitsiklis, J.~N. and Van~Roy, B. (1997).
\newblock An analysis of temporal-difference learning with function
  approximation.
\newblock {\em IEEE transactions on automatic control}, 42(5):674--690.

\bibitem[Tsitsiklis and Van~Roy, 1999]{tsitsiklis1999average}
Tsitsiklis, J.~N. and Van~Roy, B. (1999).
\newblock Average cost temporal-difference learning.
\newblock {\em Automatica}, 35(11):1799--1808.

\bibitem[Tsukamoto et~al., 2021]{tsukamoto2021contraction}
Tsukamoto, H., Chung, S.-J., and Slotine, J.-J.~E. (2021).
\newblock Contraction theory for nonlinear stability analysis and
  learning-based control: A tutorial overview.
\newblock {\em Annual Reviews in Control}, 52:135--169.

\bibitem[Wan et~al., 2021]{wan2020learning}
Wan, Y., Naik, A., and Sutton, R.~S. (2021).
\newblock Learning and planning in average-reward {Markov} decision processes.
\newblock In {\em International Conference on Machine Learning}, pages
  10653--10662. PMLR.

\bibitem[Yu and Bertsekas, 2009]{yu2009convergence}
Yu, H. and Bertsekas, D.~P. (2009).
\newblock Convergence results for some temporal difference methods based on
  least squares.
\newblock {\em IEEE Transactions on Automatic Control}, 54(7):1515--1531.

\bibitem[Zhang et~al., 2021a]{zhang2021average}
Zhang, S., Wan, Y., Sutton, R.~S., and Whiteson, S. (2021a).
\newblock Average-reward off-policy policy evaluation with function
  approximation.
\newblock In {\em international conference on machine learning}, pages
  12578--12588. PMLR.

\bibitem[Zhang et~al., 2021b]{zhang2021finite}
Zhang, S., Zhang, Z., and Maguluri, S.~T. (2021b).
\newblock Finite sample analysis of average-reward {TD}-learning and
  {Q}-learning.
\newblock {\em Advances in Neural Information Processing Systems},
  34:1230--1242.

\end{thebibliography}
	
	\newpage
	
	\begin{center}
		{\LARGE\bfseries Appendices}
	\end{center}
	
	\appendix
	
	\section{Supplementary Results for Section  \ref{sec:seminorm contractive operator}}
	\subsection{Proof of Proposition \ref{prop:seminorm_properties}}\label{pf:prop:seminorm_properties}
	\begin{enumerate}[(1)]
		\item  It is clear that $0\in \text{ker}(p)$. In addition, for any $x, y \in \text{ker}(p)$ and $\alpha\in\mathbb{R}$, we have $p(x + y) \leq p(x) + p(y) = 0$ and $p(\alpha x) = \| \alpha \| p(x) = 0$. As a result,
		we have $x + y \in \text{ker}(p)$ and $\alpha x \in \text{ker}(p)$. Therefore, $\text{ker}(p)$ is a linear subspace of $\mathbb{R}^d$.
		\item Let $U_p:=\{x \in \mathbb{R}^d \mid p(x) < 1\}$ be the open unit ball in $\mathbb{R}^d$ with respect to the seminorm $p(\cdot)$. The Minkowski functional of $U_p$ is defined as
		\begin{equation*}
			q_{U_p}(x) := \inf\{r > 0\mid x \in r U_p\} \quad \forall\,x \in \mathbb{R}^d.
		\end{equation*}
		It is well known that $U_p$ is an absorbing absolutely convex set and $p \equiv q_{U_p}$ \cite{narici2010topological,schaefer1971locally}. Suppose that there exists a bounded absorbing absolutely convex set $V$ such that $V + \text{ker}(p) = U_p$. Then, the function $\| \cdot \|$ defined as 
		\begin{align}\label{def:prop:seminorm_properties:norm}
			\| x \| := \inf\{r > 0\mid x \in r V\} \quad \forall\, x \in \mathbb{R}^d,
		\end{align}
		is a norm on $\mathbb{R}^d$ \cite{schaefer1971locally}.  Next, we present one way to construct such a set $V$. Note that the bounded absorbing absolutely convex set $V$ might not be unique, which is the reason why the norm $\| \cdot \|$ associated with $p(\cdot)$ need not be unique.

		Let $\text{ker}(p)^{\bot}$ be the orthogonal complement of $\text{ker}(p)$, and let $U_{\text{ker}(p)}=\{x \in \text{ker}(p) \mid \| x \|' < 1\}$ be the bounded open unit ball within $\text{ker}(p)$ with respect to some norm $\| \cdot \|'$ on $\mathbb{R}^d$. We define $V:= U_p \cap M$, where $M:= \text{ker}(p)^{\bot} + U_{\text{ker}(p)} = \{x+y \mid x \in \text{ker}(p)^{\bot}, y \in U_{\text{ker}(p)}\}$. Next, we show that $V$ is a bounded absorbing absolutely convex set.
		\begin{itemize}
			\item \textbf{Boundedness of $V$}: Suppose that $V$ is not bounded. Then, there must exist an $x \in V$ with unique orthogonal decomposition $x=x_{\text{ker}(p)} + x_{\text{ker}(p)^\bot}$ such that $x_{\text{ker}(p)}$ or $x_{\text{ker}(p)^\bot}$ is unbounded. Since $x \in M$, we have $x_{\text{ker}(p)} \in U_{\text{ker}(p)}$, which implies that $x_{\text{ker}(p)}$ is bounded. So, $x_{\text{ker}(p)^\bot}$ must be unbounded. However, as $x \in U_p$, this is impossible. Therefore, $V$ is a bounded set. 
			\item \textbf{$V$ Being Absorbing}: Since $U_p$ and $M$ are absorbing, for any $x \in \mathbb{R}^d$, we know that there exist positive scalars $c_{U_p}$ and $c_{M}$ such that $x/\max (c_{U_p}, c_M) \in U_p$ and $x/\max (c_{U_p}, c_M) \in M$, which implies that $x/\max (c_{U_p}, c_M) \in U_p \cap M =V$. Hence, $V$ is absorbing. 
			\item \textbf{$V$ Being Absolutely Convex}: Since $U_{p}$ and $M$ are absolutely convex, and the intersection of two absolutely convex sets is absolutely convex, $V$ is absolutely convex \cite{narici2010topological}. 
		\end{itemize}

		Finally, we show that $V + \text{ker}(p) = U_p$. Let $x \in V$ and $y \in \text{ker}(p)$. Since $x \in U_p$, we have $p(x+y) = p(x) < 1$, which implies that $x+y \in U_p$. For the other direction, let $x \in U_p$, then, there exist $y \in \text{ker}(p)$ and $z \in \text{ker}(p)^\bot$ such that $x = y + z$. It follows that $z \in M$. In addition, as $p(z) = p(y + z) = p(x) < 1$, we have $z \in U_p$. Thus, $x \in V + \text{ker}(p)$.

		Now that we have shown that $\| \cdot \|$ defined in Eq. (\ref{def:prop:seminorm_properties:norm}) is a norm, it remains to show that $p(x) = \min_{y \in \text{ker}(p)} \| x - y \|$ for all $x \in \mathbb{R}^d$. We first note that
		\begin{align*}
			\min_{y \in \text{ker}(p)} \| x - y \|
			= \min_{y \in \text{ker}(p)} \inf\{r > 0\mid x-y \in r V\} = \inf\{r > 0 \mid \exists~ y \in \text{ker}(p): x-y \in r V \}.
		\end{align*}
		Suppose that $x = 0$. Then it is clear  that $q_U(0) = \inf\{r > 0\mid \exists~ y \in \text{ker}(p): -y \in r V \} = 0$. Suppose that $x \not = 0$. Then, there exist $r > 0$ and $u \in U$ such that $x = r u$. As $V + \text{ker}(p) = U_p$, there exist $v \in V$ and $z \in \text{ker}(p)$ such that $u = v + z$, and thus $x = rv + y$, where $y := rz \in \text{ker}(p)$, which implies that $x-y \in rV$. Conversely, suppose that $x\not = 0$. Then, there exist $r> 0$, $v \in V$ and $y \in \text{ker}(p)$ such that $x-y=rv$. It follows that $x = r u$, where $u:= v + y/r \in U$, which implies that $x \in rU$. Therefore, we have
		\begin{align*}
			p(x) = q_{U}(x)
			= \inf\{r > 0\mid \exists y \in \text{ker}(p): x-y \in r V \} = \min_{y \in \text{ker}(p)} \| x - y \|.
		\end{align*}
		\item By the second part of this proposition, there exist two norms $\| \cdot \|_p$ and $\| \cdot \|_q$ on $\mathbb{R}^d$ such that
		\begin{align*}
			p(x) = \min_{y \in V} \| x - y \|_p ~\text{ and }~ q(x) = \min_{z \in V} \| x - z \|_q, \quad \forall\, x \in \mathbb{R}^d.
		\end{align*}
		Let $y^*(x) \in \arg \min_{y \in V} \| x - y \|_p$ and $z^*(x) \in \arg \min_{z \in V} \| x - z \|_q$ for all $x \in \mathbb{R}^d$. Since any two norms on a finite-dimensional space are equivalent \cite{halmos2017finite}, there exist $C_1,C_2>0$ such that $C_1\| x \|_q \leq \| x \|_p \leq C_2 \| x \|_q$ for all $x\in\mathbb{R}^d$. Therefore, for all $x \in \mathbb{R}^d$, we have
		\begin{align*}
			p(x) = \| x - y^*(x) \|_p \leq \| x - z^*(x) \|_p \leq C_2 \| x - z^*(x) \|_q = C_2 q(x),
		\end{align*}
		and
		\begin{align*}
			p(x) = \| x - y^*(x) \|_p \geq C_1 \| x - y^*(x) \|_q \geq C_1 \| x - z^*(x) \|_q = C_1 q(x).
		\end{align*}
		It follows that $C_1 q(x) \leq p(x)  \leq C_2 q(x)$ for all $x\in\mathbb{R}^d$.
	\end{enumerate}
	
	\subsection{Proof of Lemma \ref{lemma:p is norm on quotient space}}\label{ap:quotient_space}
	We first prove that $p(\cdot)$ is a norm on $\mathbb{R}^d/\text{ker}(p)$ by verifying the definition of norms.
	\begin{enumerate}[(1)]
		\item  Triangle inequality: for any $[x], [y] \in \mathbb{R}^d/\text{ker}(p)$, we have
		\begin{align*}
			p([x] + [y]) = p([x+y]) = p(x + y) \leq p(x) + p(y) = p([x]) + p([y]).    
		\end{align*}
		\item Absolute homogeneity: for any $[x] \in \mathbb{R}^d/\text{ker}(p)$ and $\alpha \in \mathbb{R}$, we have
		\begin{align*}
			p(\alpha [x]) = p([\alpha x]) = p(\alpha x) = |\alpha| p(x) = |\alpha| p([x]).    
		\end{align*}
		\item Positive definiteness: for all $[x] \in \mathbb{R}^d$, if $p([x]) = 0$, then $[x] = \text{ker}(p) = [0]$.    
	\end{enumerate}
	
	To further verify that $\left(\mathbb{R}^d/\text{ker}(p), ~p \right)$ is a Banach space, we only need to show the completeness of $\left(\mathbb{R}^d/\text{ker}(p), ~p \right)$. Let $\{[x_n]\}_{n \geq 1}$ be an arbitrary Cauchy sequence in the quotient space $\mathbb{R}^d/\text{ker}(p)$. Then, there must exist a subsequence $\{[x_{n_k}]\}_{k \geq 1}$ such that
	\begin{align*}
		p([x_{n_{k+1}} - x_{n_{k}}]) &= p([x_{n_{k+1}}] - [x_{n_{k}}]) < \frac{1}{2^k} \quad \text{for all } k \geq 1.
	\end{align*}
	Let $y_1 = 0 \in \text{ker}(p)$, then by Proposition \ref{prop:seminorm_properties}, there exists $y_2 \in \text{ker}(p)$ such that
	\begin{align*}
		\| (x_{n_2} - y_2) - (x_{n_1} - y_1) \| = p([x_{n_2} - x_{n_1}]) < \frac{1}{2}.
	\end{align*}
	Suppose that for any given $k \geq 2$, there exist $y_k, y_{k-1} \in \text{ker}(p)$ such that
	\begin{align*}
		\| (x_{n_k} - y_k) - (x_{n_{k-1}} - y_{k-1}) \| &< \frac{1}{2^{k-1}}.
	\end{align*}
	Then, again by Proposition \ref{prop:seminorm_properties}, there exists $y_{k+1} \in \text{ker}(p)$ such that
	\begin{align*}
		\| (x_{n_{k+1}} - y_{k+1}) - (x_{n_k} - y_k) \| = p([x_{n_{k+1}} - x_{n_{k}}]) < \frac{1}{2^k}.
	\end{align*}
	Therefore, by induction, there exists a sequence $\{y_k\} \in \text{ker}(p)$ such that
	\begin{align*}
		\| (x_{n_{k+1}} - y_{k+1}) - (x_{n_k} - y_k) \| < \frac{1}{2^{k}} \quad \text{for all } k \geq 1.
	\end{align*}
	Fix an arbitrary $\epsilon > 0$. Since $\sum^\infty_{k=1} 1/2^{k}$ is convergent, there exists an integer $K > 0$ such that $\sum^\infty_{k=K} 1/2^{k} < \epsilon$. Thus, for any integers $k_1, k_2 \geq K$, we have
	\begin{align*}
		\| (x_{n_{k_1}} - y_{k_1}) - (x_{n_{k_2}} - y_{k_2}) \| \leq \sum^\infty_{k=K} \| (x_{n_{k+1}} - y_{k+1}) - (x_{n_k} - y_k) \| < \sum^\infty_{k=K} \frac{1}{2^{k}} < \epsilon.
	\end{align*}
	Therefore, the sequence $\{x_{n_k} - y_k\}$ is a Cauchy sequence in $\mathbb{R}^d$. Since $\mathbb{R}^d$ is complete, there must exists an $x^* \in \mathbb{R}^d$ such that the sequence $\{x_{n_k} - y_k\}$ converges to $x^*$. So, there exists an integer $N > 0$, such that $\| (x_{n_k}-y_k) - x^*\| < \epsilon$ for all $k \geq N$, and thus
	\begin{align*}
		p([x_{n_k}] - [x^*]) &= p([x_{n_k} - x^*]) = p([x_{n_k} - y_k - x^*]) \leq \| (x_{n_k} - y_k) - x^* \| < \epsilon \quad \text{for all } k \geq N.
	\end{align*}
	Therefore, the subsequence $\{[x_{n_k}]\}$ converges to $[x^*]$, which implies that the Cauchy sequence $\{[x_n]\}_{n \geq 1}$ converges to $[x^*] \in \mathbb{R}^d/\text{ker}(p)$.

	\subsection{Proof of Theorem \ref{theorem:Banach fixed-point theorem}}\label{pf:theorem:Banach fixed-point theorem}
	We start by showing that the operator $H: \mathbb{R}^d/\text{ker}(p) \to \mathbb{R}^d/\text{ker}(p)$ defined as $H([x]) := [T(x)]$ for all $x\in\mathbb{R}^d$ is also a $\gamma$-contraction with respect to $p(\cdot)$. The first step is to verify that the mapping $H(\cdot)$ is well-defined. For any $x, y \in \mathbb{R}^d$ with $x-y \in \text{ker}(p)$, we have $p(T(x)-T(y)) \leq \gamma p(x-y) = 0$,
	which implies that $H([x])=[T(x)]=[T(y)]=H([y])$. Thus, it does not matter which representative
	element of $[x]$ we pick for computing $H([x])$. Next, for any $[x], [y] \in \mathbb{R}^d/\text{ker}(p)$, we have
	\begin{align*}
		p(H([x]) - H([y])) &= p([T(x)] - [T(y)]) \\
		&= p([T(x)-T(y)]) \\
		&= p(T(x)-T(y)) \\
		&\leq \gamma p(x-y)\\
		&= \gamma p([x-y])\\
		&= \gamma p([x]-[y]).
	\end{align*}
	As a result, together with Lemma \ref{lemma:p is norm on quotient space}, $H(\cdot)$ is a $\gamma$-contraction mapping with respect to $p(\cdot)$.
	
	Now, we are ready to prove the theorem. 
	
	\begin{enumerate}[(1)]
		\item \textbf{Existence of $x^*$:} By Lemma \ref{lemma:p is norm on quotient space}, we know that $\left(\mathbb{R}^d/\text{ker}(p), p \right)$ is a Banach space, and the mapping $H(\cdot)$ defined by $H([x]) := [T(x)]$
		is a $\gamma$-contraction with respect to $p(\cdot)$. Therefore, the Banach fixed-point theorem implies that there exists a unique $x^* \in \mathbb{R}^d$ up to equivalence class for which $H([x^*]):= [T(x^*)] = [x^*]$ and thus $p(T(x^*)-x^*) = 0$. 
		\item \textbf{Geometric Convergence:} Since the sequence $\{[x_k]\}_{k \geq 1}$ satisfies
		\begin{align*}
			[x_k] = [T(x_{k-1})] =: H([x_{k-1}]) \quad \text{for all } k \geq 1,
		\end{align*}
		again, by the Banach fixed-point theorem, we have
		\begin{align*}
			p(x_k - x^*) = p([x_k - x^*]) &= p([x_k] - [x^*]) \\
			&= p\left(H([x_{k-1}]) - H([x^*])\right)\\
			&\leq \gamma p\left([x_{k-1}] - [x^*]\right)\\
			&\vdots\\ 
			&\leq \gamma^k p([x_0] - [x^*])= \gamma^k p([x_0 - x^*]) = \gamma^k p(x_0 - x^*). 
		\end{align*}
	\end{enumerate}

	\subsection{Illustrative Examples of Seminorm Fixed-Point Theorem}\label{ap:example:optimization}
	\subsubsection{Strongly Convex and Smooth Function with respect to a Seminorm}
	Consider an optimization problem $\min_{x \in \mathbb{R}^d} f(x)$. Classical results have shown that when the objective function $f(\cdot)$ is smooth and strongly convex, using gradient descent $x_{k+1} = x_k - \alpha \nabla f(x_k)$
	to solve the optimization problem leads to geometric convergence \cite{lan2020first,beck2017first}. The gradient descent update can be equivalently written as a fixed-point iteration
	\begin{align}\label{eq:GD}
		x_{k+1}= T(x_k), \quad \forall\, k\geq 0,
	\end{align}
	where $T: \mathbb{R}^d \to \mathbb{R}^d$ is defined as $T(x)= x - \alpha \nabla f(x)$ for all $x\in\mathbb{R}^d$. In this section, we generalize the concepts of smoothness and strong convexity to the seminorm case and provide the convergence analysis of the update in Eq. (\ref{eq:GD}).
	
	Suppose that the function $f(\cdot)$ is convex and the set of global minimizers $\mathcal{X}^*$ of $f(\cdot)$ is an affine subspace of $\mathbb{R}^d$, i.e., $\mathcal{X}^* = \{x^* + y \mid y \in V\}$,
	where $x^*$ is a particular global minimizer of $f(\cdot)$ and $V$ is a linear subspace of $\mathbb{R}^d$. Let $p(\cdot)$ be a seminorm defined as $p(x) = \min_{y \in V} \| x - y \|_2$. Then, the function $f(\cdot)$ is said to be $\mu$-strong convex with respect to $p(\cdot)$ if 
	\begin{align*}
		\left(\nabla f(x) - \nabla f(y)\right)^\top \left(x-y\right)\geq \mu p\left(x-y \right)^2,\quad \forall\, x, y \in \mathbb{R}^d,
	\end{align*}
	and is said to be $L$-smooth with respect to $p(\cdot)$ if 
	\begin{align*}
		p^*\left(\nabla f(x) - \nabla f(y) \right) \leq L p\left(x-y \right), \quad \forall\, x, y \in \mathbb{R}^d,
	\end{align*}
	where $p^*(x):=\sup_{y\in\mathbb{R}^d:p(y)\leq 1}x^\top y$. In general, $p^*(x)$ can be infinity because $p(y)\leq 1$ is not necessarily a compact set. However, we will show in the following lemma that all level sets of $f(\cdot)$ are parallel to the subspace. This result ensures that $\nabla f (x) \in V^\bot$ for all $x \in 
	\mathbb{R}^d$, which guarantees that $p^*\left(\nabla f(x) - \nabla f(y) \right)$ is well-defined and finite for any $x, y \in \mathbb{R}^d$. 
	
	\begin{lemma}\label{le:all level sets are parallel}
		For any $x, y \in \mathbb{R}^d$ such that $x - y \in V$, we have $f(x) = f(y)$.
	\end{lemma}
	\begin{proof}[Proof of Lemma \ref{le:all level sets are parallel}]
		We can construct two sequences $\{z_n \in \mathcal{X}^* \mid n \geq 1\}$ and $\{\lambda_n \in [0, 1] \mid n \geq 1\}$ such that $\|z_n\|_2 \rightarrow \infty$, $\lambda_n \rightarrow 1$, and $y_n := \lambda_n x + (1 - \lambda_n) z_n \rightarrow y$ as $n \rightarrow \infty$. Using the convexity of $f(\cdot)$, we have
		\begin{align*}
			f(y_n) \leq \lambda_n f(x) + (1 - \lambda_n) f(z_n), \quad \text{for all } n \geq 1.
		\end{align*}
		By the continuity of $f(\cdot)$ and noting that $f(z_n) = \min_{x \in \mathbb{R}^d} f(x)$ for all $n \geq 1$, we can conclude that $f(y) \leq f(x)$. Similarly, we also have $f(x) \leq f(y)$. Therefore, $f(x) = f(y)$.\end{proof}
	
	As a result, the operator $T(\cdot)$ from Eq. (\ref{eq:GD}) is a contraction mapping with respect to $p(\cdot)$, which, in turn, guarantees the geometric convergence of the update in Eq. (\ref{eq:GD}). This is summarized in the following proposition.
	
	\begin{proposition}\label{le:seminorm_optimization}
		When $\alpha\in (0,2\mu/L^2)$, the operator $T(\cdot)$ is a contraction mapping with respect to the seminorm $p(\cdot)$, with contraction factor  $\gamma:=\sqrt{1 - 2\alpha \mu + \alpha^2 L^2}$. As a result, the sequence $\{x_t\}$ generated by Eq. (\ref{eq:GD}) satisfies $p(x_k-x^*)\leq \gamma^kp(x_0-x^*)$ for all $k\geq 0$.
	\end{proposition}
	
	\begin{proof}[Proof of Proposition \ref{le:seminorm_optimization}]
		For any $x, y \in \mathbb{R}^d$, we have
		\begin{align*}
			&p\left(T(x) - T(y) \right)^2 \\
			=\,& p\left(x-y - \alpha \left(\nabla f(x) - \nabla f(y)\right) \right)^2\\
			=\,& \| \Pi_{2,V^\bot}\left(x-y - \alpha \left(\nabla f(x) - \nabla f(y)\right)\right) \|^2_2\\
			=\,& \| \Pi_{2,V^\bot}\left(x-y\right) - \alpha \left(\nabla f(x) - \nabla f(y)\right) \|^2_2\\
			=\,& \| \Pi_{2,V^\bot}\left(x-y\right) \|^2_2 + 
			\alpha^2 \| \nabla f(x) - \nabla f(y)\|^2_2 -2\alpha \left(\nabla f(x) - \nabla f(y)\right)^\top \Pi_{2,V^\bot}\left(x-y\right)\\
			=\,& \| \Pi_{2,V^\bot}\left(x-y\right) \|^2_2 + 
			\alpha^2 \| \nabla f(x) - \nabla f(y)\|^2_2 -2\alpha \left(\nabla f(x) - \nabla f(y)\right)^\top \left(x-y\right).    
		\end{align*}
		By the smoothness assumption, we have 
		\begin{align*}
			\|  \nabla f(x) - \nabla f(y)\|^2_2 &= p^*\left(\nabla f(x) - \nabla f(y) \right)^2 \leq L^2 p\left(x-y \right)^2.
		\end{align*}
		By the strong convexity assumption, we have
		\begin{align*}
			\left(\nabla f(x) - \nabla f(y)\right)^\top \left(x-y\right) \geq \mu p(x-y)^2.
		\end{align*}
		Noting that $ \| \Pi_{2,V^\bot}\left(x-y\right) \|^2_2 = p(x-y)^2$, we have
		\begin{align*}
			p\left(T(x) - T(y) \right)^2 &\leq \left(1 - 2\alpha \mu + \alpha^2 L^2\right) p(x-y)^2.
		\end{align*}
		Thus, when $\alpha \in \left(0, 2\mu/L^2\right)$, $T$ is a $\sqrt{1 - 2\alpha \mu + \alpha^2 L^2}$-contraction mapping with respect to $p(\cdot)$.\end{proof}

	\subsubsection{Smooth Convex Function with the Quadratic Growth Property}\label{ap:QG}
	Consider the unconstrained minimization problem $\min_{x \in \mathbb{R}^d} f(x)$. We impose the following assumption on the objective function $f(\cdot)$.

	\begin{assumption}\label{as:quadratic_growth}
		The function $f(\cdot)$ is convex and $L$-smooth. The set of global minimizers of $f(\cdot)$, denoted by $\mathcal{X}^*$, is an affine subspace of $\mathbb{R}^d$. In addition, there exists $\mu>0$ such that
		\begin{align*}
			f(x) - f^* \geq \frac{\mu}{2} \| x - \Pi_{2,\mathcal{X}^*}(x) \|^2_2,
		\end{align*}
		where $f^*:=\min_{x\in\mathbb{R}^d}f(x)$ and $\Pi_{2,\mathcal{X}^*}$ is the orthogonal projection onto the linear subspace $\mathcal{X}^*$.
	\end{assumption}
	By defining $T: \mathbb{R}^d \to \mathbb{R}^d$ as $T(x)= x - \alpha \nabla f(x) $ for all $x\in\mathbb{R}^d$, the gradient descent method for minimizing $f(\cdot)$ can be written as
	\begin{align}\label{eq:QG}
		x_{k+1}= T(x_k), \quad \forall\,k\geq 0.
	\end{align}
	Let $p(\cdot)$ be a seminorm defined as $p(x) = \min_{y \in \mathcal{X^*}} \| x - y \|_2$ for all $x \in \mathbb{R}^d$. Then, we have the following result.
	
	\begin{lemma}\label{le:quadratic_growth}
		The operator $T(\cdot)$ is a contraction mapping with respect to the seminorm $p(\cdot)$.
	\end{lemma}
	\begin{proof}[Proof of Lemma \ref{le:quadratic_growth}]
		Under Assumption \ref{as:quadratic_growth}, we have 
		\begin{align*}
			p(x-x^*) = \min_{y \in V} \| x- (x^* + y) \|_2 = \min_{z \in \mathcal{X}^*} \| x-z \|_2 = \| x - \Pi_{2,\mathcal{X}^*}(x) \|_2.
		\end{align*}
		We now show that $T(\cdot)$ is a contraction mapping with respect to $p(\cdot)$. For any $x \in \mathbb{R}^d$, we have
		\begin{align*}
			p\left(T(x) - x^* \right)^2 &=\| x - \alpha\nabla f(x)  - \Pi_{2,\mathcal{X}^*}\left(x - \alpha\nabla f(x) \right) \|^2_2 \\
			&= \| x - \alpha \nabla f(x) - \Pi_{2,\mathcal{X}^*} (x)  \|^2_2 \\
			&= \| x - \Pi_{2,\mathcal{X}^*} (x) \|^2_2 + \alpha^2 \| \nabla f(x) \|^2_2 - 2\alpha \nabla f(x)^\top \left(x - \Pi_{2,\mathcal{X}^*} (x)\right)\\
			&\leq \| x - \Pi_{2,\mathcal{X}^*} (x) \|^2_2 + \alpha^2 L^2 \| x - \Pi_{2,\mathcal{X}^*} (x)  \|^2_2 - 2\alpha \left(f(x) - f^*\right)\\
			&\leq \| x - \Pi_{2,\mathcal{X}^*} (x) \|^2_2 + \alpha^2 L^2 \| x - \Pi_{2,\mathcal{X}^*} (x)  \|^2_2 - \alpha \mu \| x - \Pi_{2,\mathcal{X}^*} (x) \|^2_2\\
			&= \left(1 - \alpha \mu + \alpha^2 L^2 \right)\| x - \Pi_{2,\mathcal{X}^*} (x) \|^2_2\\
			&=\left(1 - \alpha \mu + \alpha^2 L^2 \right) p(x-x^*)^2.
		\end{align*}
		As a result, when the stepsize $\alpha$ is properly chosen such that $1 - \alpha \mu + \alpha^2 L^2 \in [0,1)$, then $T(\cdot)$ is a contraction mapping with respect to $p(\cdot)$. \end{proof}
	
	Applying Theorem \ref{theorem:Banach fixed-point theorem}, the iterate $x_k$ generated by the algorithm in Eq. (\ref{eq:QG}) converges to $x^*$ at a geometric rate in $p(\cdot)$, where $x^*$ is a global minimizer of $f(\cdot)$. 
	
	\section{Supplementary Results for Section \ref{sec:Lyapunov}}
	
	\subsection{Proof of Theorem \ref{theorem:Seminorm GAS and Lyapunov Equation}}
	\label{pf:theorem:Seminorm GAS and Lyapunov Equation}
	It is clear that $(3)\Leftrightarrow (2)$ (due to the equivalence between seminorms who share the same kernel spaces) and $(5) \Rightarrow (4)$. Therefore, to establish the equivalence among the five statements, it is enough to show $(1)\Leftrightarrow (2)$, $(2) \Rightarrow (5)$, and $(4)\Rightarrow (3)$.
	
	\begin{itemize}
		\item \textbf{(1) Implies (2):} Let the Jordan normal form of $A$ be given by $J = P^{-1} A P$, where
		\begin{equation}\label{eq:test}
			J=\begin{bmatrix}
				J_{<1} & 0\\
				0 & J_{\geq 1} 
			\end{bmatrix},~ P=\begin{bmatrix}
				P_{<1} & P_{\geq 1} 
			\end{bmatrix} \text{ and } P^{-1}=\begin{bmatrix}
				P^{-1}_{<1} \\
				P^{-1}_{\geq 1}
			\end{bmatrix}.
		\end{equation}
		Here, $ J_{<1} $ and $ J_{\geq 1} $ denote the Jordan blocks corresponding to eigenvalues with moduli strictly less than one and greater than or equal to one, respectively. Similarly, the columns of $ P_{<1} $ and $ P_{\geq 1} $ are the generalized eigenvectors associated with $ J_{<1} $ and $ J_{\geq 1} $, respectively. Let $ k $ denote the number of columns of $ P_{<1} $. In our expression for $ P^{-1} $, the submatrix $ P_{<1}^{-1} $ consists of the first $ k $ rows of $ P^{-1} $, while the submatrix $ P_{\geq 1}^{-1} $ consists of the last $ n-k $ rows.
		Note that, since
		\begin{align}\label{eq:gen_eig_prop}
			I_d=P^{-1}P=\begin{bmatrix}
				P^{-1}_{<1} P_{<1} & P^{-1}_{<1}P_{\geq 1}\\
				P^{-1}_{\geq 1} P_{<1}  & P^{-1}_{\geq 1} P_{\geq 1}
			\end{bmatrix},
		\end{align}
		we have $P^{-1}_{<1} P_{<1}=I_k$, $P^{-1}_{\geq 1} P_{\geq 1}=I_{d-k}$, $P^{-1}_{<1}P_{\geq 1}=0$, and $P^{-1}_{\geq 1} P_{<1}=0$. 
		As a result, we also have
		\begin{align}\label{eq:Px=0}
			P_{\geq 1}^{-1}x=0 \Leftrightarrow  x\in E_{A,<1},\quad P_{< 1}^{-1}x=0 \Leftrightarrow  x\in E_{A,\geq 1}
		\end{align}

		Let $D,D'\in\mathbb{R}^{d\times d}$ be diagonal matrices defined as
		\begin{align*}
			D=\begin{bmatrix}
				\tilde{D} & 0_{k\times (d-k)}\\
				0_{(d-k)\times k} & 0_{(d-k)\times (d-k)} 
			\end{bmatrix},\quad  D'=\begin{bmatrix}
				\tilde{D}^{-1} & 0_{k\times (d-k)}\\
				0_{(d-k)\times k} & 0_{(d-k)\times (d-k)} 
			\end{bmatrix},
		\end{align*} 
		where $\tilde{D}\in\mathbb{R}^{k\times k}$ is also a diagonal matrix with entries $\delta^0, \delta^1, \cdots, \delta^{k-1}$ along diagonal positions for some $\delta > 0$. 
		
		Define a seminorm $p(\cdot)$ as
		\begin{align}\label{eq:def:norm}
			p(x)= \| D' P^{-1} x \|_2, \quad \forall\, x \in  \mathbb{R}^d.
		\end{align}
		Since $p(x)=0$ $\Leftrightarrow P_{<1}^{-1}x=0$ $\Leftrightarrow x\in E_{A,\geq 1}$ (cf. Eq. (\ref{eq:Px=0})), we have $\text{ker}(p)=E_{A,\geq 1}$. 
		Next, we show that, with appropriately chosen $\delta$, we have $p(Ax)\leq \gamma p(x)$ for some $\gamma\in [0,1)$. We begin by observing that
		\begin{align}\label{eq:Ax_form}
			p(Ax)=p(PJP^{-1}x)=\|D'JP^{-1}x\|_2,\quad \forall\, x\in \mathbb{R}^d,
		\end{align}
		Using the definitions of $D$ and $D'$, we have
		\begin{align*}
			D'J P^{-1} x=D'JDD'P^{-1}x
			=\begin{bmatrix}
				\tilde{D}^{-1} J_{<1} \tilde{D}&  0\\
				0 & 0
			\end{bmatrix}D'P^{-1}x=\begin{bmatrix}
				\tilde{J}_{<1}+\delta U_k&  0\\
				0 & 0
			\end{bmatrix}D'P^{-1}x,
		\end{align*} 
		where $\tilde{J}_{<1}$ is a diagonal matrix with diagonal components identical to that of $J_{<1}$ and $U_k$ is the $k$-dimensional upper shift matrix (which is a binary matrix with ones only on the superdiagonal and zeros elsewhere else). Combining the previous equation with Eq. (\ref{eq:Ax_form}), we have
		\begin{align*}
			p(Ax)=\|D'JP^{-1}x\|_2\leq (\rho(J_{<1})+\delta)\|D'P^{-1}x\|_2=(\rho(J_{<1})+\delta) p(x),
		\end{align*}
		where $\rho(J_{<1})$ denotes the spectral radius of $J_{<1}$.
		By choosing $\delta=(1-\rho(J_{<1}))/2$ and defining $\gamma=(1+\rho(J_{<1}))/2$, we have  
		\begin{align*}
			p(Ax)\leq \gamma p(x),\quad \forall\,x\in \mathbb{R}^d.
		\end{align*}
		It follows that 
		\begin{align}\label{eq:geo_p_A>1}
			p(x_k)=p(A^kx_0)\leq \gamma^kp(x_0)=e^{-\log(1/\gamma)k}p(x_0),\quad \forall\,k\geq 0.
		\end{align}
		Note that $\text{ker}(p)=E_{A,\geq 1}\subseteq E$, but our goal is to find a seminorm $p_E(\cdot)$ with $\text{ker}(p_E)=E$ such that $x_k$ converges geometrically with respect to $p_E(\cdot)$. 
		
		To fix this issue,
		Let $p_E(x)=\min_{x'\in E}p(x-x')$. It is clear that $p_E(x)\leq p(x)$ for all $x\in\mathbb{R}^d$. In addition, we have $\text{ker}(p_E)=E$. To see this, let $ x \in E $ be arbitrary. Then, we have $ p_E(x) = \min_{x' \in E} p(x - x') \leq p(0) = 0 $, which implies $ x \in \ker(p_E) $. Hence, $ E \subseteq \ker(p_E) $.  To show that $ \ker(p_E) \subseteq E $, let $ x \in \ker(p_E) $ and let $ \|\cdot\| $ be a norm satisfying $ p(x) = \min_{x' \in E_{A,\geq 1}} \|x - x'\| $ (cf. Proposition \ref{prop:seminorm_properties}). Then, we have  
		\[
		0 = p_E(x) = \min_{x' \in E} \min_{x'' \in E_{A,\geq 1}} \|x - x' - x''\| \geq \min_{x' \in E} \min_{x'' \in E} \|x - x' - x''\| = \min_{z \in E} \|x - z\|,
		\]
		which implies $ x \in E $. Hence, $ \ker(p_E) \subseteq E $. 
		
		Next, we will use Eq. (\ref{eq:geo_p_A>1}) and the properties of $p_E(\cdot)$ to show that $x_k$ converges geometically with respect to $p_E(\cdot)$. For any $x_0\in\mathbb{R}^d$, there exists a unique pair $x_0^E\in E$ and $x_0^{E^\perp}\in E^\perp$ such that $x_0=x_0^E+x_0^{E^\perp}$. Since $E$ is invariant under $A$, we must have $A^kx_0^E\in E$. Therefore, we have
		\begin{align}
			p_E(x_k)=\,&p_E(A^kx_0)\nonumber\\
			=\,&p_E(A^kx_0^E+A^kx_0^{E^{\perp}})\nonumber\\
			=\,&p_E(A^kx_0^{E^{\perp}})\nonumber\\
			\leq\,& p(A^kx_0^{E^{\perp}})\nonumber\\
			\leq \,&p(x_0^{E^{\perp}})e^{-\log(1/\gamma)k}\tag{Eq. (\ref{eq:geo_p_A>1})}\nonumber\\
			\leq \,& \left(\sup_{x\in E^\perp,\|x\|_2=1}\frac{p(x)}{p_E(x)}\right) p_E(x_0^{E^{\perp}})e^{-\log(1/\gamma)k}.\label{eq:P_to_PE}
		\end{align}
		For simplicity of notation, denote $\mathcal{X}=\{x\in E^\perp\mid \|x\|_2=1\}\subseteq E^\perp$ and $C=\sup_{x\in \mathcal{X}}p(x)/p_E(x)$. Next, we show that $C$ is finite. Observe that $p(\cdot)$ and $p_E(\cdot)$ are continuous functions and $\mathcal{X}$ is a compact set. Moreover, since $E\cap \mathcal{X}=\emptyset$ and $E_{A,\geq 1}\cap \mathcal{X}=\emptyset$ (which follows from $E_{A,\geq 1}\subseteq E$), both $p(\cdot)$ and $p_E(\cdot)$ are strictly positive on $\mathcal{X}$. Therefore, by the Weierstrass extreme value theorem, we have $\sup_{x\in \mathcal{X}}p(x)<\infty$ and  $\inf_{x\in \mathcal{X}}p_E(x)>0$. As a result, 
		\begin{align*}
			C=\sup_{x\in \mathcal{X}}\frac{p(x)}{p_E(x)}\leq \frac{\sup_{x\in \mathcal{X}}p(x)}{\inf_{x'\in \mathcal{X}}p_E(x')}<\infty.
		\end{align*}
		Finally, using the finiteness of $C$ in Eq. (\ref{eq:P_to_PE}), we have
		\begin{align*}
			p_E(x_k)\leq\,& \left(\sup_{x\in E^\perp,\|x\|_2=1}\frac{p(x)}{p_E(x)}\right) p_E(x_0^{E^{\perp}})e^{-\log(1/\gamma)k}\\
			\leq\,& C p_E(x_0^{E^{\perp}})e^{-\log(1/\gamma)k}\\
			=\,& C p_E(x_0)e^{-\log(1/\gamma)k}.
		\end{align*}
		The proof is complete.

		\item \textbf{(2) Implies (1):} For any $x\in E$, since $p(Ax)\leq \alpha p(x)e^{-\beta}=0$, we must have $Ax\in E$. Therefore, $E$ is invariant under $A$. To show that $E_{A,\geq 1}\subseteq E$, let $\lambda$ be an eigenvalue of $A$ such that $|\lambda|\geq 1$. We will show by induction that all generalized eigenvectors $s_1,s_2,\cdots,s_m$ (suppose there are $m$ of them) associated with $\lambda$ are contained in $E$, where $s_1$ satisfies $As_1=\lambda s_1$ and $s_j=(A-\lambda I)s_{j+1}$ for all $j=1,2,\cdots,m-1$. This is sufficient to conclude that $E_{A,\geq 1}\subseteq E$.
		
		\textit{The Base Case:} Since $As_1=\lambda s_1$, we have $p(A^ks_1)=|\lambda|^kp(s_1)$ for all $k\geq 0$. However, we know that $\lim_{k\rightarrow\infty}p(A^ks_1)=0$. Therefore, since $|\lambda|\geq 1$, we must have $p(s_1)=0$, i.e., $s_1\in E$.
		
		\textit{The Induction Step:} Suppose that $s_j\in E$ for some $j\in \{1,2,\cdots,m-1\}$. Since $p(A^ks_j)\leq \alpha p(s_j)e^{-\beta k}=0$ for all $k\geq 1$, we must have $A^ks_j\in E$. As a result, we have $A^k(A-\lambda I)s_{j+1}=A^{k+1}s_{j+1}-\lambda A^ks_{j+1}=A^ks_j\in E$, which implies $p(A^{k+1}s_{j+1})=|\lambda| p(A^ks_{j+1})$. Since $k$ is arbitrary, we have by telescoping that $p(A^ks_{j+1})=|\lambda|^kp(s_{j+1})$ for all $k\geq 0$. However, we know that $p(A^ks_{j+1})\leq \alpha p(s_{j+1})e^{-\beta k}\rightarrow 0$. Combining both the relations implies $p(s_{j+1})=0$, i.e., $s_{j+1}\in E$, because $|\lambda|\geq 1$.
		\item \textbf{(2) Implies (5):} Let $P = \sum^\infty_{k=0} (A^k )^\top Q A^k$. We will verify that $P$ is the unique solution to the Lyapunov equation (\ref{eq:discrete Lyapunov equation}).
		\begin{itemize}
			\item \textbf{Finiteness of $P$:} We first show that $P$ is finite. Let $p_Q:\mathbb{R}^d\to \mathbb{R}$ be a seminorm defined as $p_Q(x)=\sqrt{x^\top Qx}$. It is clear that $\text{ker}(p_Q)=E$. Therefore, we have by Proposition \ref{prop:seminorm_properties} (2) that there exists $C>0$ such that $p_Q(x)\leq Cp(x)$ for all $x\in\mathbb{R}^d$. Now, for any $x\in\mathbb{R}^d$, we have
			\begin{align*}
				x^\top P x =\,& \sum^{\infty}_{k=0} x^\top\left(A^k \right)^\top Q A^k x \\
				= \,&\sum^{\infty}_{k=0} p_Q^2(A^kx)\\
				\leq \,&C^2\sum^{\infty}_{k=0} p^2(A^kx)\\
				\leq \,&C^2\alpha^2p^2(x)\sum^{\infty}_{k=0} e^{-2\beta k}\\
				\leq \,&\infty.
			\end{align*}
			Since $x$ is arbitrary, $P$ must be finite.
			\item \textbf{$P$ Solves Eq. (\ref{eq:discrete Lyapunov equation}):} Observe that
			\begin{align*}
				A^\top P A - P =\,& A^\top \left[\sum^\infty_{k=0} \left(A^k \right)^\top Q A^k \right] A - \sum^\infty_{k=0} \left(A^k \right)^\top Q A^k\\
				=\,& \sum^\infty_{j=1} \left(A^j \right)^\top Q A^j - \sum^\infty_{k=0} \left(A^k \right)^\top Q A^k\\
				=\,& -Q.
			\end{align*}
			Therefore, the matrix $P$ is a solution to Eq. (\ref{eq:discrete Lyapunov equation}).
			\item \textbf{$P$ is Positive Semi-definite with $\text{ker}(P)=E$:} 
			To show that $P \in \mathcal{S}^{d}_{+,E}$, we first note that $P$ is symmetric and positive semi-definite because for any $x \in \mathbb{R}^d$, we have
			\begin{align}\label{eq:PQA}
				x^\top P x = \sum^{\infty}_{k=0} \left(A^k x\right)^\top Q \left(A^k x\right) \geq 0.
			\end{align}
			It remains to show that $\text{ker}(P) = E$. On the one hand, if $x \in \text{ker}(P)$, we must have $A^kx\in \text{ker}(Q)=E$ for any $k\geq 0$. Setting $k=0$ implies $x\in E$. On the other hand, if $x\in E$, since $p(A^kx)\leq \alpha p(x)e^{-\beta k}=0$ for all $k\geq 0$, we must have $A^kx\in E=\text{ker}(Q)$ for all $k\geq 0$, which implies $x\in \text{ker}(P)$ via Eq. (\ref{eq:PQA}).
			\item \textbf{Uniqueness of $P$:} Suppose that there exists $P_1 \in \mathcal{S}^{d}_{+,E}$ (different from $P$) satisfying Eq. (\ref{eq:discrete Lyapunov equation}). Then, we must have
			\begin{align*}
				P &= \sum^\infty_{k=0} \left(A^k \right)^\top Q A^k\\
				&= \sum^\infty_{k=0} \left(A^k \right)^\top \left(P_1 - A^\top P_1A\right) A^k\\
				&= \sum^\infty_{k=0} \left(A^k \right)^\top P_1 A^k - \sum^\infty_{k=0} \left(A^{k+1} \right)^\top P_1 A^{k+1}\\
				&= P_1,
			\end{align*}
			which is a contradiction.
		\end{itemize}
		\item \textbf{(4) Implies (3):} Consider the seminorm $p(\cdot)$ defined as $p(x) := \sqrt{x^\top P x}$. It is clear that $\text{ker}(p)=\text{ker}(P)=E$. Moreover, for any $k \geq 0$, we have
		\begin{align*}
			p^2(x_{k+1}) - p^2(x_k)=\,& \left(Ax_k\right)^\top P \left(Ax_k\right) - x^\top_k P x_k \\
			=\,& x^\top_k \left(A^\top P A - P\right) x_k \\
			=\,& - x^\top_k Q x_k\\
			=\,&-p_{Q}^2(x_k),
		\end{align*}
		where $p_{Q}(x)=\sqrt{x^\top Qx}$ is also a seminorm with kernel space $E$. Since all seminorms sharing the same kernel space are equivalent (cf. Proposition \ref{prop:seminorm_properties} (2)), there exist $C_1\in (0,1)$ and $C_2\in (1,\infty)$ such that $C_1p_Q(x)\leq p(x)\leq C_2p_Q(x)$ for all $x\in\mathbb{R}^d$. 
		Therefore, we have
		\begin{align*}
			p^2(x_{k+1}) - p^2(x_k)\leq -p_{Q}^2(x_k)\leq -\frac{1}{C_2^2}p^2(x_k).
		\end{align*}
		Rearranging terms, we haves $p(x_{k+1})\leq \gamma p(x_k)$ for all $k\geq 0$, where $\gamma=\sqrt{1-1/C_2^2}$. Repeatedly using the previous inequality, we have
		\begin{align*}
			p(x_k)\leq \gamma^k p(x_0)=e^{-k\log(1/\gamma)}p(x_0),\quad \forall\,k\geq 0.
		\end{align*}
		Finally, again using the fact that all seminorms sharing the same kernel space are equivalent, for any seminorm $p'(\cdot)$ with $\text{ker}(p')=E$, we have 
		\begin{align*}
			p'(x_k)\leq C e^{-k\log(1/\gamma)}p'(x_0),\quad \forall\,k\geq 0,
		\end{align*}
		for some constant $C>0$.
	\end{itemize}

	\subsection{Proof of Theorem \ref{thm:Lyapunov_Continuous}}\label{pf:thm:Lyapunov_Continuous}
	
	Similarly to the proof of Theorem \ref{theorem:Seminorm GAS and Lyapunov Equation}, since (2) $\Leftrightarrow$ (3) and (5) $\Rightarrow$ (4) are straightforward, we only need to show (1) $\Leftrightarrow$ (2), (4) $\Rightarrow$ (2), and (2) $\Rightarrow$ (5).
	
	\begin{itemize}
		\item \textbf{(1) Implies (2)}: Since $E$ is invariant under $A$, it is also invariant under $e^{At}$ for any $t\geq 0$, which follows from the definition of matrix exponential. To proceed, we list the following facts from linear algebra, the proofs of which can be found in standard linear algebra textbooks.
		\begin{fact}\label{Fact1}
			The following statements hold:
			\begin{enumerate}[(1)]
				\item  $\lambda$ is an eigenvalue of $A$ if and only if $e^\lambda$ is an eigenvalue of $e^A$.
				\item  For each eigenvalue $\lambda$ of $A$, the algebraic and geometric multiplicities of $\lambda$ coincide with those of $e^\lambda$ for $e^A$.
				\item  An $x\in\mathbb{R}^d$ is a generalized eigenvector of order $k$ for $A$ (associated with $\lambda$) if and only if $x$ is a generalized eigenvector of order $k$ for $e^A$ (associated with $e^\lambda$).  
			\end{enumerate}
		\end{fact}
		
		Since $E$ is an invariant subspace of $e^A$ and $E_{A,\geq 0}=E_{e^{A},\geq 1}\subseteq E$, by Theorem \ref{theorem:Seminorm GAS and Lyapunov Equation}, there exists a seminorm $p(\cdot)$ (defined in terms of a positive semi-definite matrix) with $\text{ker}(p)=E$ and a constant $\gamma\in [0,1)$ such that
		\begin{align*}
			p(e^{A}x)\leq \gamma p(x),\quad\forall\,x\in\mathbb{R}^d.
		\end{align*}
		Therefore, for any $t\geq 0$, let $n_t$ be the largest integer smaller than $t$ and let $\|\cdot\|$ be a norm such that $p(x)=\min_{x'\in E}\|x-x'\|$. Then, we have
		\begin{align*}
			p(e^{At}x)=\,&p(e^{A n_t} \cdot e^{A(t-n_t)}x)\\
			\leq \,& \gamma^{n_t}p(e^{A(t-n_t)}x)\\
			\leq \,&\gamma^{n_t}\|e^{A(t-n_t)}\|p(x)\tag{Lemma \ref{le:Lipschitz_lsa}}\\
			\leq \,&\gamma^{n_t}\sup_{s \in [0,1]}\|e^{As}\|p(x)\\
			\leq \,&e^{-\log(1/\gamma) t}\frac{\sup_{s \in [0,1]}\|e^{As}\|}{\gamma }p(x)
		\end{align*}
		for all $t\geq 0$. Since $\|e^{As}\|$ as a function $s$ is continuous and $[0,1]$ is a compact set, by Weierstrass extreme value theorem, we have $\sup_{s \in [0,1]}\|e^{As}\|<\infty$. Therefore, the ODE (\ref{eq:ODE}) is globally exponentially stable with respect to $p(\cdot)$.
		
		\item \textbf{(2) Implies (1)}: Since there exists $p(\cdot)$ with $\text{ker}(p)=E$ such that $p(x(t))\leq \alpha p(x(0))e^{-\beta t}$ for some $\alpha,\beta>0$, we must have $e^{At}x\in E$ for any $t\geq 0$. This also implies $d e^{At} e/dt=Ae^{At}x\in E$ for all $t\geq 0$. Setting $t=0$ implies that $E$ is invariant under $A$. 
		
		Again since there exists $p(\cdot)$ with $\text{ker}(p)=E$ such that $p(x(t))\leq \alpha p(x(0))e^{-\beta t}$ for some $\alpha,\beta>0$, for any non-negative integer $k$, we must have $p(x(k))\leq \alpha p(x(0))e^{-\beta k}$. The result then follows from Fact \ref{Fact1} and Theorem \ref{theorem:Seminorm GAS and Lyapunov Equation} (2) $\Rightarrow$ (1).

		\item \textbf{(4) Implies (2)} Let $p(x)=\sqrt{x^\top Px}$, which is clearly a seminorm with $\text{ker}(p)=E$. Moreover, we have
		\begin{align*}
			\frac{d}{dt}p(x(t))=\,&\frac{d}{dt}x(t)^\top Px(t)\\
			=\,&\dot{x}(t)^\top Px+x^\top P\dot{x}(t)\\
			=\,&x(t)^\top (A^\top P+PA)x(t)\\
			=\,&-x(t)^\top Qx(t)\\
			=\,&-p_Q^2(x(t))\\
			\leq \,&-C^2p^2(x(t)),
		\end{align*}
		where $C>0$ satisfies $P_Q(\cdot)\geq Cp(\cdot)$ (cf. Proposition \ref{prop:seminorm_properties} (3)). 
		By Gr\"{o}nwall's inequality, we have 
		\begin{align*}
			p(x(t))\leq p(x(0))\exp\left(-C^2 t\right),\quad \forall\,t\geq 0.
		\end{align*}
		\item \textbf{(2) Implies (5):} To begin with, note that the solution to ODE (\ref{eq:ODE}) is explicitly given by $x(t)=e^{At}x(0)$. We will show that the unique solution to the Lyapunov equation is given by
		\begin{align}\label{eq:solution_CLE}
			P=\int_0^\infty e^{A^\top t}Qe^{At}dt.
		\end{align}
		\begin{itemize}
			\item \textbf{Finiteness of $P$:} Let $p_Q(x):=\sqrt{x^\top Qx}$ for all $x\in\mathbb{R}^d$. It is clear that $p_Q(\cdot)$ is a seminorm with $\text{ker}(p_Q)=E$. In addition, since seminorms who share the same kernel space are equivalent, there exists $C>0$ such that $p_Q(\cdot)\leq Cp(\cdot)$. As a result, for any $x\in\mathbb{R}^d$, we have
			\begin{align*}
				\int_0^\infty x^\top e^{A^\top t}Qe^{At} xdt
				=\,&\int_0^\infty p_{Q}^2(e^{At} x) dt\\
				\leq \,&C^2\int_0^\infty p^2(e^{At} x) dt\\
				\leq \,&C^2 \alpha^2 p^2(x)\int_0^\infty e^{-2\beta t} dt\\
				<\,&\infty.
			\end{align*}
			Therefore, $P$ is finite.
			\item \textbf{$P$ is Positive Semi-definite with $\text{ker}(P)=E$:} It is clear that $P$ is a positive semi-definite matrix. Therefore, we only need to show that $\text{ker}(P)=E$. On the one hand, let $x\in E$. Then, we have $p(x)=0$. It follows that
			\begin{align*}
				x^\top P x\leq C^2 \alpha^2 p^2(x)\int_0^\infty e^{-2\beta t} dt=0,
			\end{align*}
			implying $x\in\text{ker}(P)$. On the other hand, let $x\in \text{ker}(P)$. Then, by Eq. (\ref{eq:solution_CLE}), we have
			\begin{align*}
				0=x^\top Px=\int_0^\infty x^\top e^{A^\top t}Qe^{At}x dt.
			\end{align*}
			Since $x^\top e^{A^\top t}Qe^{At}x\geq 0$ and is a continuous function of $t$, we must have $x^\top e^{A^\top t}Qe^{At}x=0$ for all $t\geq 0$. Setting $t=0$ gives us $x^\top Qx=0$, implying $x\in E$.

			\item \textbf{$P$ Solves Eq. (\ref{eq:Lyapunov_Continuous}):} To verify that $P$ is a solution to the Lyapunov equation, observe that
			\begin{align*}
				A^\top P+PA=\,&\int_0^\infty \left[A^\top e^{A^\top t}Qe^{At}+e^{A^\top t}Qe^{At} A\right] dt\\
				=\,&\int_0^\infty \frac{d}{dt}\left[e^{A^\top t}Qe^{At}\right] dt\\
				=\,&e^{A^\top t}Qe^{At}\big|_0^\infty\\
				=\,&Q,
			\end{align*}
			where the last line follows from 
			\begin{align*}
				0\leq \lim_{t\rightarrow\infty}x^\top e^{A^\top t}Qe^{At} x= \lim_{t\rightarrow\infty}p_Q^2(e^{At}x)=0.
			\end{align*}
			Therefore, we have $A^\top P+PA+Q=0$.
			\item \textbf{Uniqueness of $P$:} Suppose that there exists $P_1\in\mathcal{S}_{+,E}^d$ (different from $P$) satisfying $A^\top P_1+P_1A+Q=0$. Then, we have
			\begin{align*}
				P=\,&\int_0^\infty e^{A^\top t}Qe^{At}dt\\
				=\,&-\int_0^\infty e^{A^\top t}(A^\top P_1+P_1A)e^{At}dt\\
				=\,&-\int_0^\infty \left[A^\top e^{A^\top t} P_1e^{At}+e^{A^\top t}P_1e^{At}A\right]dt\tag{$A$ and $e^{At}$ commute}\\
				=\,&-\int_0^\infty\frac{d}{dt}\left[e^{A^\top t} P_1e^{At}\right]dt\\
				=\,&-e^{A^\top t} P_1e^{At}\big|_0^\infty.
			\end{align*}
			To proceed, we need to evaluate the limit $\lim_{t\rightarrow\infty}e^{A^\top t} P_1e^{At}$. For any $x\in\mathbb{R}^d$, we have
			\begin{align*}
				0\leq\,& \lim_{t\rightarrow\infty}xe^{A^\top t} P_1e^{At}x\\
				\leq\,& C\lim_{t\rightarrow\infty}x^\top e^{A^\top t} Qe^{At}x\tag{Proposition \ref{prop:seminorm_properties} (3)}\\
				=\,&C\lim_{t\rightarrow\infty}p_Q(e^{At}x)^2\\
				=\,&0.
			\end{align*}
			It follows that $\lim_{t\rightarrow\infty}e^{A^\top t} P_1e^{At}=0$. As a result, we have
			\begin{align*}
				P=-e^{A^\top t} P_1e^{At}\big|_0^\infty=P_1,
			\end{align*}    
			which contradicts to $P_1\neq P$. 
		\end{itemize}
	\end{itemize}
	
	\section{Supplementary Results for Section \ref{sec:SA}}
	\subsection{Proof of Lemma \ref{le:box}} \label{pf:le:box}
	\label{proof: properties of infimal convolution with an indicator function}
	\begin{enumerate}[(1)]
		\item Since $f \geq g$, we have for any $x\in\mathbb{R}^d:$
		\begin{align*}
			(f ~\square~ \delta_E)(x) =\inf_{y \in\mathbb{R}^d}\left\{f(y)+\delta_E(x-y)\right\} \geq  \inf_{y \in\mathbb{R}^d}\left\{g(y)+\delta_E(x-y)\right\} =(g ~\square~ \delta_E)(x).
		\end{align*}
		\item If $\beta = 0$, the property holds automatically. For any scalar $\beta > 0$, we have for any $x\in\mathbb{R}^d$:
		\begin{align*}
			(\beta f ~\square~ \delta_E)(x) &=\inf_{y \in\mathbb{R}^d}\left\{\beta f(y)+\delta_E(x-y)\right\}\\
			&=\inf_{y \in\mathbb{R}^d}\left\{\beta f(y)+\beta \delta_E(x-y)\right\}\tag{$\delta_E = \beta \delta_E$}\\
			&=\beta \inf_{y \in\mathbb{R}^d}\left\{ f(y)+ \delta_E(x-y)\right\}\\
			&=\beta(f ~\square~ \delta_E)(x).
		\end{align*}
		\item For any $x \in \mathbb{R}^d$, we have
		\begin{align*}
			(f ~\square~ g)(x)&=\inf_{y \in\mathbb{R}^d}\left\{f(y)+g(x-y)\right\}\\
			&=\inf_{z \in\mathbb{R}^d}\left\{f(x-z)+g(z)\right\}\tag{Change of variable: $z=x-y$}\\
			&=(g ~\square~ f)(x).
		\end{align*}
		\item For any $x\in\mathbb{R}^d$, we have 
		\begin{align*}
			\left[(f ~\square~ g) ~\square~ h\right](x) &=\inf_{y\in\mathbb{R}^d}\left\{\inf_{z \in\mathbb{R}^d}\left\{f(z)+g(y-z)\right\}+h(x-y)\right\}\\
			&=\inf_{z\in\mathbb{R}^d}\left\{f(z) + \inf_{y \in\mathbb{R}^d}\left\{g(y-z)+h(x-y)\right\}\right\}\\
			&=\inf_{z \in\mathbb{R}^d}\left\{f(z)+\inf_{y\in\mathbb{R}^d}\left\{g(y-z)+h((x-z)-(y-z))\right\}\right\}\\
			&=\inf_{z \in\mathbb{R}^d}\left\{f(z)+\inf_{u\in\mathbb{R}^d}\left\{g(u)+h((x-z)-u)\right\}\right\}\tag{Change of variable: $u=y-z$}\\
			&=\left[f ~\square~ (g ~\square~ h)\right](x).
		\end{align*}
		\item For any $x\in\mathbb{R}^d$, we have
		\begin{align*}
			(\delta_E ~\square~ \delta_E)(x)=\inf_{y \in \mathbb{R}^d}\left\{\delta_E(y)+\delta_E(x-y)\right\}
			=\begin{dcases}
				0 &x\in E,\\
				+\infty &x\notin E.
			\end{dcases}
		\end{align*}
		\item Since $E$ is a convex set, $\delta_E$ is a proper convex function. Thus, by Theorem 2.19 of \cite{beck2017first}, $f ~\square~ \delta_E$ is convex. Moreover, if $f$ is also $L$-smooth, then by Theorem 5.30 of \cite{beck2017first}, we have $f ~\square~ \delta_E$ is $L$-smooth. 
	\end{enumerate}
	
	\subsection{Proof of Proposition \ref{prop:Moreau}}
	\label{pf:prop:Moreau}
	\begin{enumerate}[(1)]
		\item First, note that for any $x \in \mathbb{R}^d$, we have
		\begin{align*}
			\frac{1}{2}p^2_{c,E}(x) =\,&\min_{y\in E}\frac{1}{2}\|x-y\|_c^2\\
			=\,&\min_{y\in \mathbb{R}^d}\left\{\frac{1}{2}\|x-y\|_c^2+\delta_E(y)\right\}\\
			=\,&\min_{z\in \mathbb{R}^d}\left\{\frac{1}{2}\|z\|_c^2+\delta_E(x-z)\right\}\tag{Change of variable: $z=x-y$}\\
			=\,&\left(\frac{1}{2}\|\cdot\|_c^2 ~\square~ \delta_E\right)(x).
		\end{align*}
		Using Lemma \ref{le:box} (3) and (4), we have
		\begin{align*}
			M_E(x)=\,& \left[\frac{1}{2}p^2_{c,E}(\cdot) ~\Box~ \frac{1}{2\theta} \| \cdot \|^2_s\right](x)\\
			=\,& \left[\left(\frac{1}{2}\|\cdot\|_c^2 ~\square~ \delta_E\right) ~\Box~ \frac{1}{2\theta} \| \cdot \|^2_s\right](x) \\
			=\,& \left[\frac{1}{2} \| \cdot \|^2_c ~\square~ \left(\frac{1}{2\theta} \| \cdot \|^2_s ~\Box~ \delta_E\right)\right](x).
		\end{align*}
		Since $\| \cdot \|^2_s/$ is convex and $(L/\theta)$-smooth with respect to $\| \cdot \|_s$, by Lemma \ref{le:box} (6), the function $\| \cdot \|^2_s/(2\theta) ~\Box~ \delta_E$ is also convex and $(L/\theta)$-smooth with respect to $\| \cdot \|_s$. Using \cite[Theorem 2.19 and Theorem 5.30]{beck2017first}, we know that $M_E(\cdot)$ is also convex and $(L/\theta)$-smooth with respect to $\| \cdot \|_s$. This implies for any $x,y\in\mathbb{R}^d$:
		\begin{align*}
			M_E(y) \leq M_E(x) + \langle \nabla M_E(x), y-x \rangle + \frac{L}{2\theta} \| y-x\|_s^2. 
		\end{align*}
		To finish the proof, we next argue that the $\|\cdot\|_s^2$ on the right-hand side of the previous inequality can be replaced by $p_{s,E}(\cdot)^2$.
		
		For any $x \in \mathbb{R}^d$ and $z \in E$, we have
		\begin{align*}
			M_E(x+z) =\,& \min_{u \in \mathbb{R}^d} \left\{\frac{1}{2} p^2_{c,E}(x+z-u) + \frac{1}{2\theta} \| u \|^2_s \right\}\\
			=\,& \min_{u \in \mathbb{R}^d} \left\{\frac{1}{2} p^2_{c,E}(x-u) + \frac{1}{2\theta} \| u \|^2_s \right\} \\
			=\,& M_E(x).
		\end{align*}
		By convexity of $M_E(\cdot)$, we have
		\begin{align*}
			\langle\nabla M_E(x),z\rangle\leq\,& M_E(x+z)- M_E(x)=0  \\
			\langle \nabla M_E(x),-z\rangle=\,&\langle \nabla M_E(x+z),-z\rangle \leq M_E(x)-M_E(x+z)=0.
		\end{align*}
		Therefore, we must have $\langle \nabla M_E(x), z \rangle = 0$ for all $x \in \mathbb{R}^d$ and $z \in E$. Using the above result together with the smoothness of $M_E(\cdot)$, we have for any $x,y\in\mathbb{R}^d$ that
		\begin{align*}
			M_E(y) &= M_E(y-z^*) \tag{$z^* := \arg\min_{z \in E} \| (y - x) - z \|_s \in E$}\\
			&\leq M_E(x) + \langle \nabla M_E(x), y-z^*-x \rangle + \frac{L}{2\theta} \| y+z^*-x\|_s^2\\
			&= M_E(x) + \langle \nabla M_E(x), y-x \rangle + \frac{L}{2\theta} p^2_{s,E}(y-x),
		\end{align*}
		where the last line follows from $\|y-x-z^*\|_s=p_{s,E}(y-x)$.
		
		\item By Lemma \ref{le:box} (3) and (4), we can rewrite $M_E(\cdot)$ equivalently as follows:
		\begin{align*}
			M_E(x) =\,& \left(\frac{1}{2}p^2_{c,E}(\cdot) ~\Box~ \frac{1}{2\theta} \| \cdot \|^2_s\right)(x)  \\
			=\,&\left(\left(\frac{1}{2} \| \cdot \|^2_c ~\square~ \delta_E \right) ~\Box~ \frac{1}{2\theta} \| \cdot \|^2_s\right)(x)  \\
			=\,&\left(\left(\frac{1}{2} \| \cdot \|^2_c ~\square~ \frac{1}{2\theta} \| \cdot \|^2_s\right) ~\Box~ \delta_E\right)(x)\\
			=\,&\left(M(\cdot)~\Box~ \delta_E\right)(x),
		\end{align*}
		where $M:\mathbb{R}^d\to \mathbb{R}$ is defined as $M(x)=\left(\frac{1}{2} \| \cdot \|^2_c ~\square~ \frac{1}{2\theta} \| \cdot \|^2_s\right)(x)$. It was shown in \cite{chen2021lyapunov} that
		\begin{align*}
			\ell^2_{cm}M(x) \leq \frac{1}{2}\| x \|^2_c \leq u^2_{cm}M(x),\quad \forall\, x \in \mathbb{R}^d.
		\end{align*}
		Therefore, using Lemma \ref{le:box} (1) and (2), we have
		\begin{align*}
			\ell^2_{cm}M_E(x) \leq \frac{1}{2}p^2_{c,E}(x) \leq u^2_{cm}M_E(x),\quad\forall\, x \in \mathbb{R}^d.
		\end{align*}
		\item[(3)] We have shown in Part (1) of this proposition that $M_E(\cdot)$ is $(L/\theta)$-smooth with respect to $\| \cdot \|_s$. Given $x,y\in\mathbb{R}^d$, let $v^* = \arg\min_{v \in E} \| (x - y) - v \|_s \in E$, which implies $\| (x-y)- v^* \|_s = p_{s,E}(x-y)$. Since $\langle \nabla M_E(x),z\rangle=0$ for any $x\in\mathbb{R}^d$ and $z\in E$, we have by an equivalent definition of smoothness that
		\begin{align}
			\frac{\theta}{L}\| \nabla M_E(x) - \nabla M_E(y) \|^2_{s,*} &\leq \langle \nabla M_E(x) - \nabla M_E(y), x - y \rangle \nonumber\\
			&=\langle \nabla M_E(x) - \nabla M_E(y), x - y - v^*\rangle\nonumber\\
			&\leq \| \nabla M_E(x) - \nabla M_E(y)\|_{s,*} \| x - y - v^*\|_s\nonumber\\
			&= \| \nabla M_E(x) - \nabla M_E(y)\|_{s,*} p_{s,E}(x - y). \label{eq:substitute_back}
		\end{align}
		Rearranging terms, we obtain
		\begin{align*}
			\| \nabla M_E(x) - \nabla M_E(y) \|_{s,*} \leq \frac{L}{\theta} p_{s,E}(x-y),\quad \forall\,x,y\in\mathbb{R}^d,
		\end{align*}
		which is the claimed inequality. As a corollary of the above inequality, for any $x,y,z\in\mathbb{R}^d$, we have
		\begin{align}
			\langle \nabla M_E(x) - \nabla M_E(y), z \rangle=\,&\langle \nabla M_E(x) - \nabla M_E(y), z-z^* \rangle\tag{$z^*:=\arg\min_{z'\in E}\|z-z'\|_s$}\nonumber\\
			\leq\,& \|\nabla M_E(x) - \nabla M_E(y)\|_{s,^*}\|z-z^*\|_s\nonumber\\
			\leq \,&\frac{L}{\theta} p_{s,E}(x-y)p_{s,E}(z),\label{eq:moreau_extra}
		\end{align}
		which will also be used in our analysis.
	\end{enumerate}

	\subsection{Proof of Theorem \ref{thm:SA_finite}} 
	\label{proof: general finite-sample bound}
	Using the smoothness of $M_E(\cdot)$ (cf. Proposition \ref{prop:Moreau}) and the update equation \eqref{algo:SA}, we have for any $k\geq 0$ that
	\begin{align}\label{eq:error decomposition}
		&\mathbb{E}[M_E(x_{k+1}-x^*)]\nonumber\\
		\leq\;& \mathbb{E}\left[M_E(x_k-x^*)\right]+\mathbb{E}\left[\langle\nabla M_E(x_k-x^*),x_{k+1}-x_k\rangle\right]+\frac{L}{2\theta}\mathbb{E}\left[p_{s,E}(x_{k+1}-x_k)^2\right]\nonumber\\
		=\;& \mathbb{E}\left[M_E(x_k-x^*)\right]+\alpha_k\mathbb{E}\left[\langle\nabla M_E(x_k-x^*),F(x_k,Y_k)-x_k+\omega_k\rangle\right]\nonumber\\
		&+ \frac{L\alpha_k^2}{2\theta}\mathbb{E}[p_{s,E}(F(x_k,Y_k)-x_k+\omega_k)^2]\nonumber\\
		=\;& \mathbb{E}\left[M_E(x_k-x^*)\right]+\underbrace{\alpha_k\mathbb{E}\left[\langle\nabla M_E(x_k-x^*),\Bar{F}(x_k)-x_k\rangle\right]}_{T_1}\nonumber\\
		&+\underbrace{\alpha_k\mathbb{E}\left[\langle\nabla M_E(x_k-x^*),F(x_k,Y_k)-\Bar{F}(x_k)\rangle\right]}_{T_2} + \underbrace{\alpha_k\mathbb{E}\left[\langle\nabla M_E(x_k-x^*),\omega_k\rangle\right]}_{T_3}\nonumber\\
		&+\underbrace{\frac{L\alpha_k^2}{2\theta}\mathbb{E}\left[p_{s,E}(F(x_k,Y_k)-x_k+\omega_k)^2\right]}_{T_4}.
	\end{align}
	Now, we bound terms $T_1$ -- $T_4$ in the following sequence of lemmas. We begin with the term $T_1$.
	\begin{lemma}\label{le:T_1}
		It holds for any $k\geq 0$ that $$T_1\leq -\left(1-\gamma^2 \frac{u^2_{cm}}{\ell^2_{cm}}\right)\alpha_k\mathbb{E}[M_E(x_k-x^*)].$$
	\end{lemma}
	\begin{proof}[Proof of Lemma \ref{le:T_1}]
		For any $k \geq 0$, we have
		\begin{align*}
			\langle\nabla M_E(x_k-x^*),\bar{F}(x_k)-x_k\rangle &= \langle\nabla M_E(x_k-x^*), \bar{F}(x_k) - x^* + x^* - x_k\rangle\\
			&\leq M_E(\bar{F}(x_k) - x^*) - M_E(x_k - x^*) \tag{Convexity of $M_E(\cdot)$}\\
			&\leq \frac{1}{2\ell^2_{cm}}p^2_{c,E}(\bar{F}(x_k) - x^*) - M_E(x_k - x^*) \tag{Proposition \ref{prop:Moreau}}\\
			&= \frac{1}{2\ell^2_{cm}}p^2_{c,E}(\bar{F}(x_k) - \bar{F}(x^*)) - M_E(x_k - x^*) \tag{$\bar{F}(x^*)-x^*\in E$}\\
			&\leq \frac{\gamma^2}{2\ell^2_{cm}}p^2_{c,E}(x_k - x^*) - M_E(x_k - x^*) \tag{Assumption \ref{as:seminorm_contraction}}\\
			&\leq \frac{\gamma^2 u^2_{cm}}{\ell^2_{cm}}M_{E}(x_k - x^*) - M_E(x_k - x^*) \tag{Proposition \ref{prop:Moreau}}\\
			&= -\left(1 - \frac{\gamma^2 u^2_{cm}}{\ell^2_{cm}} \right) M_E(x_k - x^*).
		\end{align*}
	\end{proof}

	To bound the error term $T_2$, we need the following two lemmas.
	
	\begin{lemma}\label{le:Lipschitz}
		The following two inequalities hold: 
		\begin{enumerate}[(1)]
			\item  $p_{c,E}(F(x,y))\leq A_1p_{c,E}(x)+B_1$ for all $x\in\mathbb{R}^d$ and $y\in\mathcal{Y}$.
			\item $p_{c,E}(\Bar{F}(x))\leq A_1p_{c,E}(x)+B_1$ for all $x\in\mathbb{R}^d$.
		\end{enumerate}
	\end{lemma}
	\begin{proof}[Proof of Lemma \ref{le:Lipschitz}]
		For any $x\in\mathbb{R}^d$ and $y\in\mathcal{Y}$, using Assumption \ref{as:sa_all}, we have
		\begin{align*}
			p_{c,E}(F(x,y))\leq p_{c,E}(F(x,y)-F(0,y))+p_{c,E}(F(0,y))\leq A_1p_{c,E}(x)+B_1.
		\end{align*}
		Since $p_{c,E}(\cdot)$ is convex, using Jensen's inequality, we have for all $x\in\mathbb{R}^d$:
		\begin{align*}
			p_{c,E}(\Bar{F}(x))=p_{c,E}(\mathbb{E}_{Y\sim \mu}[F(x,Y)])\leq \mathbb{E}_{Y\sim \mu}[p_{c,E}(F(x,Y))]\leq A_1p_{c,E}(x)+B_1.
		\end{align*}
	\end{proof}

	\begin{lemma}\label{le:difference}
		Let non-negative integers $k_1 \leq k_2$ be such that $\alpha_{k_1,k_2-1}\leq 1/(4A)$. Then, we have for all $k\in [k_1,k_2]$ that
		\begin{align*}
			p_{c,E}(x_k-x_{k_1})\leq\,& 2\alpha_{k_1,k_2-1}(Ap_{c,E}(x_{k_1})+B)\leq  \frac{1}{2}(p_{c,E}(x_{k_1})+B/A),\\
			p_{c,E}(x_k-x_{k_1})\leq\,& 4\alpha_{k_1,k_2-1}(Ap_{c,E}(x_{k_2})+B)\leq p_{c,E}(x_{k_2})+B/A.
		\end{align*}
	\end{lemma}
	\begin{proof}[Proof of Lemma \ref{le:difference}]
		Using the update equation (\ref{algo:SA}), we have for any $k\geq 0$:
		\begin{align*}
			p_{c,E}(x_{k+1})-p_{c,E}(x_k)&\leq p_{c,E}(x_{k+1}-x_k)\\
			&= \alpha_k p_{c,E}(F(x_k,Y_k)-x_k+w_k)\\
			&\leq  \alpha_k \left(p_{c,E}(F(x_k,Y_k))+p_{c,E}(x_k)+p_{c,E}(w_k)\right)\\
			&\leq  \alpha_k \left(A_1p_{c,E}(x_k)+B_1+p_{c,E}(x_k)+A_2p_{c,E}(x_k)+B_2\right)\tag{Lemma \ref{le:Lipschitz} and Assumption \ref{as:sa_all} (3)}\\
			&\leq \alpha_k (A p_{c,E}(x_k)+B),
		\end{align*}
		where we recall that $A=A_1+A_2+1$ and $B=B_1+B_2$. Rearranging terms, we obtain
		\begin{align*}
			p_{c,E}(x_{k+1})+\frac{B}{A}\leq (1+\alpha_k A)\left(p_{c,E}(x_k)+\frac{B}{A}\right).
		\end{align*}
		Therefore, we have for all $k\in [k_1,k_2]$:
		\begin{align*}
			p_{c,E}(x_k)&\leq \prod_{j=k_1}^{k-1}(1+\alpha_j A)\left(p_{c,E}(x_{k_1})+\frac{B}{A}\right)-\frac{B}{A}\nonumber\\
			&\leq \prod_{j=k_1}^{k-1}e^{\alpha_j A}\left(p_{c,E}(x_{k_1})+\frac{B}{A}\right)-\frac{B}{A}\nonumber\\
			&= e^{\alpha_{k_1,k-1} A}\left(p_{c,E}(x_{k_1})+\frac{B}{A}\right)-\frac{B}{A}\nonumber\\
			&\leq  \left(1+2\alpha_{k_1,k-1} A\right)\left(p_{c,E}(x_{k_1})+\frac{B}{A}\right)-\frac{B}{A}\tag{$*$}\\
			&=\left(1+2\alpha_{k_1,k-1} A\right)p_{c,E}(x_{k_1})+2\alpha_{k_1,k-1} B,\nonumber
		\end{align*}
		where Inequality $(*)$ follows from the numerical inequality $e^z\leq 1+2z$ for all $z\in [0, 1/2]$ and our assumption that $\alpha_{k_1,k_2-1}\leq 1/(4A)$. Therefore, we have for any $k\in [k_1,k_2-1]$ that
		\begin{align*}
			p_{c,E}(x_{k+1}-x_k)&\leq \alpha_k (A p_{c,E}(x_k)+B)\\
			&\leq \alpha_k \left[A \left(1+2\alpha_{k_1,k-1} A\right)p_{c,E}(x_{k_1})+2\alpha_{k_1,k-1} AB+B\right]\\
			&\leq 2\alpha_k (A p_{c,E}(x_{k_1})+B),
		\end{align*}
		where the last inequality follows from $\alpha_{k_1,k-1}\leq \alpha_{k_1,k_2-1}\leq 1/(4A)$. Now that we have a bound for the incremental error, using telescoping and triangle inequality, we have for any $k\in [k_1,k_2]$ that
		\begin{align*}
			p_{c,E}(x_k-x_{k_1})&\leq \sum_{i=k_1}^{k-1}p_{c,E}(x_{i+1}-x_i)\\
			&\leq \sum_{i=k_1}^{k-1}2\alpha_i (A p_{c,E}(x_{k_1})+B)\\
			&= 2\alpha_{k_1,k-1} (A p_{c,E}(x_{k_1})+B)\\
			&\leq 2\alpha_{k_1,k_2-1} (A p_{c,E}(x_{k_1})+B).
		\end{align*}
		To establish the second claimed inequality, note that the previous inequality implies
		\begin{align*}
			p_{c,E}(x_{k_2}-x_{k_1})&\leq 2\alpha_{k_1,k_2-1}(Ap_{c,E}(x_{k_1})+B)\\
			&\leq 2\alpha_{k_1,k_2-1}(Ap_{c,E}(x_{k_1}-x_{k_2})+Ap_{c,E}(x_{k_2})+B)\\
			&\leq \frac{1}{2}p_{c,E}(x_{k_2}-x_{k_1})+ 2\alpha_{k_1,k_2-1}(Ap_{c,E}(x_{k_2})+B),
		\end{align*}
		where in the last inequality we again used $\alpha_{k_1,k_2-1}\leq 1/(4A)$.
		Rearranging terms, we obtain $p_{c,E}(x_{k_2}-x_{k_1})\leq   4\alpha_{k_1,k_2-1}(Ap_{c,E}(x_{k_2})+B)$. Therefore, we have for any $k\in [k_1,k_2]$ that
		\begin{align*}
			p_{c,E}(x_k-x_{k_1})&\leq 2\alpha_{k_1,k_2-1}(Ap_{c,E}(x_{k_1})+B)\\
			&\leq 2\alpha_{k_1,k_2-1}(Ap_{c,E}(x_{k_1}-x_{k_2})+Ap_{c,E}(x_{k_2})+B)\\
			&\leq 2\alpha_{k_1,k_2-1}(4A\alpha_{k_1,k_2-1}(Ap_{c,E}(x_{k_2})+B)+Ap_{c,E}(x_{k_2})+B)\\\
			&\leq 4\alpha_{k_1,k_2-1}(Ap_{c,E}(x_{k_2})+B). 
		\end{align*}
	\end{proof}

	Using the previous two lemmas, we next bound the term $T_2$ in the following lemma.
	\begin{lemma}\label{le:T_2}
		It holds for any $k\geq t_k$ that $$T_2\leq \frac{80LA^2u_{cm}^2\alpha_k\alpha_{k-t_k,k-1}}{\theta\ell_{cs}^2}\mathbb{E}\left[M_E(x_k-x^*)\right]+\frac{40L\alpha_k\alpha_{k-t_k,k-1}}{\theta\ell_{cs}^2}\left(Ap_{c,E}(x^*)+B\right)^2.$$
	\end{lemma}
	\begin{proof}[Proof of Lemma \ref{le:T_2}]
		We begin by decomposing $T_2$ as follows:
		\begin{align}\label{eq:error decomposition T_2}
			T_2  =\,& \alpha_k\underbrace{\mathbb{E}\left[\langle\nabla M_E(x_k-x^*) - \nabla M_E(x_{k-t_k}-x^*),F(x_k,Y_k)-\Bar{F}(x_k)\rangle\right]}_{T_{21}}\nonumber\\
			&+ \alpha_k \underbrace{\mathbb{E}\left[\langle\nabla M_E(x_{k-t_k}-x^*),F(x_k,Y_k)-F(x_{k-t_k}, Y_k)+\Bar{F}(x_{k-t_k})-\Bar{F}(x_{k})\rangle\right]}_{T_{22}}\nonumber\\
			&+ \alpha_k\underbrace{\mathbb{E}\left[\langle\nabla M_E(x_{k-t_k}-x^*),F(x_{k-t_k},Y_k)-\Bar{F}(x_{k-t_k})\rangle\right]}_{T_{23}}.
		\end{align}
		Next, we bound terms $T_{21}$ -- $T_{23}$.

		For the term  $T_{21}$, we have
		\begin{align*}
			T_{21}=\,&\mathbb{E}[\langle\nabla M_E(x_k-x^*) - \nabla M_E(x_{k-t_k}-x^*),F(x_k,Y_k)-\Bar{F}(x_k)\rangle] \\
			\leq \,&\frac{L}{\theta} \mathbb{E}\left[p_{s,E}\left(x_k-x_{k-t_k}\right) p_{s,E}\left(F(x_k,Y_k)-\Bar{F}(x_k)\right)\right]\tag{This follows from Eq. (\ref{eq:moreau_extra})}\\
			\leq \,& \frac{L}{\theta \ell^2_{cs}} \mathbb{E}\left[p_{c,E}\left(x_k-x_{k-t_k}\right) p_{c,E}\left(F(x_k,Y_k)-\Bar{F}(x_k)\right)\right].
		\end{align*}
		Since $\alpha_{k-t_k, k-1} \leq 1/(4A)$, using Lemma \ref{le:difference}, we have
		\begin{align*}
			p_{c,E}\left(x_k-x_{k-t_k}\right) \leq\,& 4\alpha_{k-t_k, k-1}\left(Ap_{c,E}(x_{k})+B\right)\\
			\leq \,&4\alpha_{k-t_k, k-1}\left(Ap_{c,E}(x_{k} - x^*)+Ap_{c,E}(x^*)+B\right). 
		\end{align*}
		Using Assumption \ref{as:seminorm_contraction}, Lemma \ref{le:Lipschitz} and the fact that $\bar{F}(x^*) x^*\in E$, we have
		\begin{align*}
			p_{c,E}\left(F(x_k,Y_k)-\Bar{F}(x_k)\right) &= p_{c,E}\left(F(x_k,Y_k)-\Bar{F}(x_k)+ \bar{F}(x^*) - x^*\right)\\
			&\leq p_{c,E}\left(F(x_k,Y_k)\right) + p_{c,E}\left(\Bar{F}(x_k) - \bar{F}(x^*)\right) + p_{c,E}\left(x^*\right)\\
			&\leq A_1p_{c,E}(x_k)+B_1 + \gamma p_{c,E}\left(x_k-x^*\right) + p_{c,E}\left(x^*\right)\\
			&\leq (A_1+\gamma)p_{c,E}(x_k - x^*) + (A_1+1)p_{c,E}(x^*) +B_1\\
			&\leq A p_{c,E}(x_k - x^*) + A p_{c,E}(x^*) + B.
		\end{align*}
		Combining the previous three inequalities together, we obtain 
		\begin{align}\label{eq:T_{21} bound}
			T_{21} &\leq \frac{4L\alpha_{k-t_k, k-1}}{\theta \ell^2_{cs}}\mathbb{E}\left[\left(Ap_{c,E}(x_{k} - x^*)+Ap_{c,E}(x^*)+B\right)^2\right]\nonumber\\  
			&\leq \frac{8LA^2\alpha_{k-t_k, k-1}}{\theta \ell^2_{cs}}\mathbb{E}\left[p^2_{c,E}(x_{k} - x^*)\right] + \frac{8L\alpha_{k-t_k, k-1}}{\theta \ell^2_{cs}}\left(Ap_{c,E}(x^*)+B\right)^2\nonumber\\
			&\leq \frac{16LA^2u^2_{cm}\alpha_{k-t_k, k-1}}{\theta \ell^2_{cs}}\mathbb{E}\left[M_{E}(x_{k} - x^*)\right] + \frac{8L\alpha_{k-t_k, k-1}}{\theta \ell^2_{cs}}\left(Ap_{c,E}(x^*)+B\right)^2,
		\end{align}
		where the last line follows from Proposition \ref{prop:Moreau} (2).
		
		Next, we consider the term $T_{22}$ from Eq. (\ref{eq:error decomposition T_2}). Using Proposition \ref{prop:Moreau}, $\nabla M_E(0) = 0$ (since $0 \in \arg\min_{x \in \mathbb{R}^d} M_E(x)$), we have
		\begin{align*}
			T_{22}=&\mathbb{E}\left[\langle\nabla M_E(x_{k-t_k}-x^*),F(x_k,Y_k)-F(x_{k-t_k}, Y_k)+\Bar{F}(x_{k-t_k})-\Bar{F}(x_{k})\rangle\right] \\
			\leq &\frac{L}{\theta} \mathbb{E}\left[p_{s,E}\left(x_{k-t_k}-x^*\right) p_{s,E}\left(F(x_k,Y_k)-F(x_{k-t_k}, Y_k)+\Bar{F}(x_{k-t_k})-\Bar{F}(x_{k})\right)\right]\tag{Eq. (\ref{eq:moreau_extra})}\\
			\leq & \frac{L}{\theta \ell^2_{cs}} \mathbb{E}\left[p_{c,E}\left(x_{k-t_k}-x^*\right) p_{c,E}\left(F(x_k,Y_k)-F(x_{k-t_k}, Y_k)+\Bar{F}(x_{k-t_k})-\Bar{F}(x_{k})\right)\right].
		\end{align*}
		Using Lemma \ref{le:difference}, we have
		\begin{align*}
			p_{c,E}\left(x_{k-t_k}-x^*\right) &\leq p_{c,E}\left(x_k - x_{k-t_k}\right) + p_{c,E}\left(x_k-x^*\right)\\
			&\leq p_{c,E}(x_{k})+\frac{B}{A} + p_{c,E}\left(x_k-x^*\right)\\
			&\leq 2 p_{c,E}(x_{k} - x^*) + p_{c,E}(x^*) + \frac{B}{A} \\
			&\leq 2\left(p_{c,E}(x_{k} - x^*) + p_{c,E}(x^*) + \frac{B}{A}\right).
		\end{align*}
		Using Assumption \ref{as:seminorm_contraction}, Lemma \ref{le:Lipschitz}, and Lemma \ref{le:difference}, we have
		\begin{align*}
			& p_{c,E}\left(F(x_k,Y_k)-F(x_{k-t_k}, Y_k)+\Bar{F}(x_{k-t_k})-\Bar{F}(x_{k})\right)\\
			\leq & p_{c,E}\left(F(x_k,Y_k)-F(x_{k-t_k}, Y_k)\right) + p_{c,E}\left(\Bar{F}(x_{k}) - \Bar{F}(x_{k-t_k})\right)\\
			\leq & (A_1+\gamma) p_{c,E}(x_k-x_{k-t_k})\\
			\leq & A p_{c,E}(x_k-x_{k-t_k})\\
			\leq & 4A\alpha_{k-t_k, k-1}\left(Ap_{c,E}(x_{k} - x^*)+Ap_{c,E}(x^*)+B\right)
		\end{align*}
		Combining the previous three inequalities together, we obtain
		\begin{align}\label{eq:T_{22} bound}
			T_{22} &\leq \frac{8L\alpha_{k-t_k, k-1}}{\theta \ell^2_{cs}}\mathbb{E}\left[\left(Ap_{c,E}(x_{k} - x^*)+Ap_{c,E}(x^*)+B\right)^2\right]\nonumber\\  
			&\leq \frac{16LA^2\alpha_{k-t_k, k-1}}{\theta \ell^2_{cs}}\mathbb{E}\left[p^2_{c,E}(x_{k} - x^*)\right] + \frac{16L\alpha_{k-t_k, k-1}}{\theta \ell^2_{cs}}\left(Ap_{c,E}(x^*)+B\right)^2\nonumber\\
			&\leq \frac{32LA^2u^2_{cm}\alpha_{k-t_k, k-1}}{\theta \ell^2_{cs}}\mathbb{E}\left[M_{E}(x_{k} - x^*)\right] + \frac{16L\alpha_{k-t_k, k-1}}{\theta \ell^2_{cs}}\left(Ap_{c,E}(x^*)+B\right)^2.
		\end{align}
		
		Next, we consider the term $T_{23}$ from Eq. (\ref{eq:error decomposition T_2}). Using Proposition \ref{prop:Moreau} and $\nabla M_E(0) = 0$, we have
		\begin{align*}
			T_{23}
			= & \mathbb{E}\left[\langle\nabla M_E(x_{k-t_k}-x^*), \mathbb{E}\left[F(x_{k-t_k},Y_k) \mid \mathcal{F}_{k-t_k} \right]-\Bar{F}(x_{k-t_k}) \rangle\right] \\
			\leq &\frac{L}{\theta} \mathbb{E}\left[p_{s,E}\left(x_{k-t_k}-x^*\right) p_{s,E}\left(\mathbb{E}\left[F(x_{k-t_k},Y_k)\mid \mathcal{F}_{k-t_k}\right]-\Bar{F}(x_{k-t_k})\right)\right]\\
			\leq &\frac{L}{\theta \ell_{cs}^2} \mathbb{E}\left[p_{c,E}\left(x_{k-t_k}-x^*\right) p_{c,E}\left(\mathbb{E}\left[F(x_{k-t_k},Y_k)\mid \mathcal{F}_{k-t_k}\right]-\Bar{F}(x_{k-t_k})\right)\right].
		\end{align*}
		Using Lemma \ref{le:difference}, we have
		\begin{align*}
			p_{c,E}(x_{k-t_k}-x^*)=\,&p_{c,E}(x_k-x_{k-t_k})+p_{c,E}(x_k-x^*)\\
			\leq \,&p_{c,E}(x_k)+\frac{B}{A}+p_{c,E}(x_k-x^*)\\
			\leq \,&2p_{c,E}(x_k-x^*)+p_{c,E}(x^*)+\frac{B}{A}\\
			\leq \,&2\left(p_{c,E}(x_k-x^*)+p_{c,E}(x^*)+\frac{B}{A}\right)\\
			\leq \,&2\left(Ap_{c,E}(x_k-x^*)+Ap_{c,E}(x^*)+B\right),
		\end{align*}
		where the last line follows from $A\geq 1$.
		
		Using Assumption \ref{as:sa_all}, Lemma \ref{le:Lipschitz}, and Lemma \ref{le:difference}, we have
		\begin{align*}
			& p_{c,E}\left(\mathbb{E}\left[F(x_{k-t_k},Y_k)\,\middle|\,\mathcal{F}_{k-t_k} \right]-\Bar{F}(x_{k-t_k})\right)\\
			=\, & p_{c,E}\left(\sum_{y \in \mathcal{Y}} \left(\mathbb{P}\left(Y_k=y \mid  Y_{k-t_k}\right) - \mu(y)\right)F(x_{k-t_k},y)\right)\\
			\leq \,& \sum_{y \in \mathcal{Y}} \left |\mathbb{P}\left(Y_k=y | Y_{k-t_k}\right) - \mu(y)\right | p_{c,E}\left(F(x_{k-t_k},y)\right)\\
			\leq \,& 2\sup_{y\in\mathcal{Y}}\left\{d_{\text{TV}}\left(\mathbb{P}\left(Y_{t_k}=\cdot | Y_0=y\right), \mu(\cdot)\right)\right\} \left(A_1p_{c,E}(x_{k-t_k})+B_1\right)\\
			\leq \,& 2 \alpha_k \left(A_1p_{c,E}(x_k - x_{k-t_k}) + A_1p_{c,E}(x_k) +B_1\right)\\
			\leq \,& 2 \alpha_k \left(A_1(p_{c,E}(x_k)+B/A) + A_1p_{c,E}(x_k) +B_1\right)\\
			\leq \,& 4 \alpha_k \left(Ap_{c,E}(x_{k})+ B\right)\\
			\leq \,& 4 \alpha_k \left(Ap_{c,E}(x_{k}-x^*) + Ap_{c,E}(x^*) + B\right).
		\end{align*}
		Combining the previous three inequalities together, we have
		\begin{align}
			T_{23} 
			&\leq \frac{8L\alpha_{k}}{\theta \ell^2_{cs}}\mathbb{E}\left[\left(Ap_{c,E}(x_{k} - x^*)+Ap_{c,E}(x^*)+B\right)^2\right]\nonumber\\  
			&\leq \frac{16LA^2\alpha_{k}}{\theta \ell^2_{cs}}\mathbb{E}\left[p^2_{c,E}(x_{k} - x^*)\right] + \frac{16L\alpha_{k}}{\theta \ell^2_{cs}}\left(Ap_{c,E}(x^*)+B\right)^2\nonumber\\
			&\leq \frac{32LA^2u^2_{cm}\alpha_{k}}{\theta \ell^2_{cs}}\mathbb{E}\left[M_{E}(x_{k} - x^*)\right] + \frac{16L\alpha_{k}}{\theta \ell^2_{cs}}\left(Ap_{c,E}(x^*)+B\right)^2.\label{eq:T_{23} bound}
		\end{align}
		Finally, using Eqs. \eqref{eq:T_{21} bound}, \eqref{eq:T_{22} bound} and \eqref{eq:T_{23} bound} in Eq. \eqref{eq:error decomposition T_2}, we have
		$$T_2\leq \frac{80LA^2u_{cm}^2\alpha_k\alpha_{k-t_k,k-1}}{\theta\ell_{cs}^2}\mathbb{E}\left[M_E(x_k-x^*)\right]+\frac{40L\alpha_k\alpha_{k-t_k,k-1}}{\theta\ell_{cs}^2}\left(Ap_{c,E}(x^*)+B\right)^2.$$
	\end{proof}
	
	Next, we bound the error term $T_3$ in the following lemma.
	\begin{lemma}\label{le:T_3}
		It holds for any $k\geq 0$ that $T_3=0$.
	\end{lemma}
	\begin{proof}[Proof of Lemma \ref{le:T_3}]
		Since $x_k$ is measurable with respect to the $\sigma$-algebra $\mathcal{F}_k$ and $\{\omega_k\}$ is a martingale difference sequence with respect to $\mathcal{F}_k$, we have by the tower property of conditional expectations that
		\begin{align*}
			T_3&= \alpha_k \mathbb{E}\left[\langle\nabla M(x_k-x^*),w_k\rangle\right]\\
			&= \alpha_k \mathbb{E}\left[\mathbb{E}[\langle\nabla M_E(x_k-x^*),w_k\rangle\mid \mathcal{F}_k]\right]\\
			&= \alpha_k \mathbb{E}\left[\langle\nabla M_E(x_k-x^*),\mathbb{E}[w_k\mid \mathcal{F}_k]\rangle\right]\\
			&=0.
		\end{align*}
	\end{proof}

	Next, we bound the error term $T_4$ in the following lemma.
	\begin{lemma}\label{le:T_4}
		It holds for any $k\geq 0$ that $$T_4 \leq \frac{2LA^2u_{cm}^2\alpha_k^2}{\theta\ell_{cs}^2}\mathbb{E}\left[M_E(x_k-x^*)\right]+\frac{L\alpha_k^2}{\theta\ell_{cs}^2}\left(Ap_{c,E}(x^*)+B\right)^2.$$
	\end{lemma}
	\begin{proof}[Proof of Lemma \ref{le:T_4}]
		For any $k\geq 0$, we have
		\begin{align*}
			&p_{s,E}\left(F(x_k,Y_k)-x_k+w_k\right)\\
			\leq\;& \frac{1}{\ell_{cs}}p_{c,E}(F(x_k,Y_k)-x_k+w_k) \\
			\leq\;& \frac{1}{\ell_{cs}}[p_{c,E}(F(x_k,Y_k))+p_{c,E}(x_k)+p_{c,E}(w_k)]\\
			\leq\;& \frac{1}{\ell_{cs}}[A_1p_{c,E}(x_k)+B_1+p_{c,E}(x_k)+A_2p_{c,E}(x_k)+B_2]\tag{Assumption \ref{as:sa_all} and Lemma \ref{le:Lipschitz}}\\
			=\;& \frac{1}{\ell_{cs}}[Ap_{c,E}(x_k)+B]\\
			\leq\;& \frac{1}{\ell_{cs}}[Ap_{c,E}(x_k-x^*)+Ap_{c,E}(x^*)+B].
		\end{align*}
		It follows that
		\begin{align*}
			T_4&=\frac{L\alpha_k^2}{2\theta}\mathbb{E}\left[p_{s,E}(F(x_k,Y_k)-x_k+w_k)^2\right]\\
			&\leq \frac{L\alpha_k^2}{2\theta\ell_{cs}^2}\mathbb{E}\left[(Ap_{c,E}(x_k-x^*)+Ap_{c,E}(x^*)+B)^2\right]\\
			&\leq \frac{LA^2\alpha_k^2}{\theta\ell_{cs}^2}\mathbb{E}\left[p^2_{c,E}(x_k-x^*)\right]+\frac{L\alpha_k^2}{\theta\ell_{cs}^2}\left(Ap_{c,E}(x^*)+B\right)^2\\
			&\leq \frac{2LA^2u_{cm}^2\alpha_k^2}{\theta\ell_{cs}^2}\mathbb{E}\left[M_E(x_k-x^*)\right]+\frac{L\alpha_k^2}{\theta\ell_{cs}^2}\left(Ap_{c,E}(x^*)+B\right)^2,
		\end{align*}
		where the last line follows from Proposition \ref{prop:Moreau} (2).
	\end{proof}
	
	Now that we have controlled the terms $T_1-T_4$ on the right-hand side of Eq. \eqref{eq:error decomposition}, using them altogether, we have the desired one-step recursive inequality stated in the following lemma.
	\begin{lemma}\label{le:recursion_ap}
		It holds  for all $k\geq t_k$ that
		\begin{align*}
			\mathbb{E}\left[M_E(x_{k+1}-x^*)\right]\leq
			\left(1-\varphi_2 \alpha_k\right)\mathbb{E}[M_E(x_k-x^*)] +\frac{\varphi_3\alpha_k\alpha_{k-t_k,k-1}}{2u_{cm}^2}\left(Ap_{c,E}(x^*)+B\right)^2.
		\end{align*}
	\end{lemma}
	\begin{proof}[Proof of Lemma \ref{le:recursion_ap}]
		Using the bounds for the terms $T_1-T_4$ in Eq. \eqref{eq:error decomposition}, we have for all $k\geq t_k$ that
		\begin{align*}
			\mathbb{E}[M_E(x_{k+1}-x^*)]\leq\,&
			\left(1-\left(1-\gamma^2 \frac{u^2_{cm}}{\ell^2_{cm}}\right)\alpha_k+\frac{82LA^2u_{cm}^2\alpha_k\alpha_{k-t_k,k-1}}{\theta\ell_{cs}^2}\right)\mathbb{E}[M_E(x_k-x^*)]\\
			&+\frac{41L\alpha_k\alpha_{k-t_k,k-1}}{\theta\ell_{cs}^2}\left(Ap_{c,E}(x^*)+B\right)^2\\
			\leq \,&\left(1-2\varphi_2 \alpha_k+ \varphi_3 A^2\alpha_k\alpha_{k-t_k,k-1}\right)\mathbb{E}[M_E(x_k-x^*)]\\ &+\frac{\varphi_3\alpha_k\alpha_{k-t_k,k-1}}{2u_{cm}^2}\left(Ap_{c,E}(x^*)+B\right)^2,
		\end{align*}
		where we recall the definition of the constants $\{\varphi_i\}^3_{i=1}$ in Eq. \eqref{eq:def:constants}. Since $\alpha_{k-t_k,k-1} \leq \varphi_2/(\varphi_3A^2)$ for all $k \geq t_k$ (cf. Condition \ref{condition: stepsizes requirement}), we have
		\begin{align*}
			\mathbb{E}\left[M_E(x_{k+1}-x^*)\right]\leq
			\left(1-\varphi_2 \alpha_k\right)\mathbb{E}[M_E(x_k-x^*)] +\frac{\varphi_3\alpha_k\alpha_{k-t_k,k-1}}{2u_{cm}^2}\left(Ap_{c,E}(x^*)+B\right)^2,\quad \forall\,k\geq t_k.
		\end{align*}
	\end{proof}

	Next, we repeatedly apply the previous lemma to obtain:
	\begin{align}
		&\mathbb{E}[p_{c,E}(x_{k}-x^*)^2] \nonumber\\
		\leq\,& 2u_{cm}^2\mathbb{E}[M_E(x_{k}-x^*)]\nonumber\\
		\leq\,& 2u_{cm}^2 \mathbb{E}[M_E(x_{K}-x^*)] \prod^{k-1}_{j=K} (1 - \varphi_2 \alpha_j)\nonumber\\
		& + \varphi_3 (Ap_{c,E}(x^*)+B)^2\sum^{k-1}_{i=K}\alpha_i \alpha_{i-t_i, i-1} \prod^{k-1}_{j=i+1} (1-\varphi_2\alpha_j)\nonumber\\
		\leq \,&\frac{u_{cm}^2}{\ell^2_{cm}} \mathbb{E}[p_{c,E}(x_{K}-x^*)] \prod^{k-1}_{j=K} (1 - \varphi_2 \alpha_j)\nonumber\\
		& + \varphi_3 (Ap_{c,E}(x^*)+B)^2\sum^{k-1}_{i=K}\alpha_i \alpha_{i-t_i, i-1} \prod^{k-1}_{j=i+1} (1-\varphi_2\alpha_j)\nonumber\\
		=\,& \varphi_1 \mathbb{E}[p_{c,E}(x_{K}-x^*)] \prod^{k-1}_{j=K} (1 - \varphi_2 \alpha_j) + \varphi_3c_2 \sum^{k-1}_{i=K}\alpha_i \alpha_{i-t_i, i-1} \prod^{k-1}_{j=i+1} (1-\varphi_2\alpha_j).\label{eq:recursion 4}
	\end{align}
	According to Condition \ref{condition: stepsizes requirement}, we have $\alpha_{0,K-1} \leq 1/(4A)$. Therefore, by Lemma \ref{le:difference}, we have
	\begin{align}\label{eq:almost sure bound}
		\mathbb{E}[p_{c,E}(x_{K}-x^*)^2] &\leq \mathbb{E}\left[\left(p_{c,E}(x_{K}-x_0) + p_{c,E}(x_0 - x^*)\right)^2\right]\nonumber\\
		&\leq \left(p_{c,E}(x_0) + p_{c,E}(x_{0}-x^*) +\frac{B}{A}\right)^2 \nonumber\\
		&= c_1.
	\end{align}
	Finally, by combining the previous two inequalities together, we obtain
	\begin{align*}
		\mathbb{E}[p_{c,E}(x_{k}-x^*)^2] \leq \varphi_1 c_1 \prod^{k-1}_{j=K} (1 - \varphi_2 \alpha_j) + \varphi_3c_2 \sum^{k-1}_{i=K}\alpha_i \alpha_{i-t_i, i-1} \prod^{k-1}_{j=i+1} (1-\varphi_2\alpha_j), \quad \forall\, k \geq K.
	\end{align*}
	Upon obtaining the general finite-sample bound, we can derive the finite-sample convergence bounds for three common choices of stepsizes. The proof is identical to that of \cite[Theorem 2.1]{chen2021lyapunov}, and therefore is omitted.
	
	\subsection{Proof of Theorem \ref{thm:linear_SA}}\label{sec:lsa_pf}
	There are two approaches to prove this theorem: one is to reformulate it as a seminorm-contractive SA, verify the required assumptions, and then apply Theorem \ref{thm:SA_finite}; the other is to directly prove it using the \textit{continuous-time} Lyapunov equation (cf. Theorem \ref{thm:Lyapunov_Continuous}). We present the second approach here and defer the first approach to Appendix \ref{ap:alternative}.
	
	For simplicity of notation, denote $G(x,y)=A(y)x+b(y)$ for any $x\in\mathbb{R}^d$ and $y\in\mathcal{Y}$ and $\bar{G}(x)=\bar{A}x+\bar{b}$ for any $x\in\mathbb{R}^d$. Then, the linear SA algorithm described in Eq. (\ref{algo:linear_SA}) can be equivalently written as
	\begin{align}\label{eq:linear_SA_equivalent}
		x_{k+1}=x_k+\alpha_k G(x_k,Y_k),\quad \forall\,k\geq 0.
	\end{align}
	Let $p(\cdot)$ be a seminorm defined as $p(x)=\sqrt{x^\top Px}$, where $P$ is defined in Assumption \ref{as:lsa_Lyapunov}. Since $p(\cdot)$ is defined in terms of a positive semi-definite matrix, the norm-square function $p^2(x)/2$ is $1$-smooth with respect to $p(x)$. Therefore, we can directly use $M_E(x)=p^2(x)/2$ as the Lyapunov function. 
	
	For any $k\geq 0$, using the definition of $M_E(\cdot)$ and Eq. (\ref{eq:linear_SA_equivalent}), we have for all $k\geq 0$ that
	\begin{align}\label{eq:Thm:lsa_proof_decomposition}
		\mathbb{E}[M_E(x_{k+1}-x^*)]=\,&\mathbb{E}[M_E(x_k-x^*)]+\underbrace{\alpha_k\mathbb{E}[\nabla M_E(x_k-x^*)^\top \bar{G}(x_k)]}_{:=T_1}\nonumber\\
		&+\underbrace{\alpha_k\mathbb{E}[\nabla M_E(x_k-x^*)^\top (G(x_k,Y_k)-\bar{G}(x_k))]}_{:=T_2}+\underbrace{\frac{\alpha_k^2}{2}\mathbb{E}[p^2(G(x_k,Y_k)]}_{:=T_3}.
	\end{align}
	Next, we bound the terms $T_1, T_2$, and $T_3$ in the following three lemmas.
	\begin{lemma}\label{le:LSA_T1}
		It holds for all $k\geq 0$ that
		\begin{align*}
			T_1 \leq -c_2'\alpha_k\mathbb{E}[M_E(x_k-x^*)].
		\end{align*}
	\end{lemma}
	\begin{proof}[Proof of Lemma \ref{le:LSA_T1}] For any $k\geq 0$, we have
		\begin{align*}
			\nabla M_E(x_k-x^*)^\top \bar{G}(x_k)
			=\,& (x_k-x^*)^\top P (\bar{A}x_k+\bar{b})\\
			=\,& (x_k-x^*)^\top P \bar{A}(x_k-x^*)\tag{$\bar{A}x^*+\bar{b}\in E=\text{ker}(P)$}\\
			=\,& (x_k-x^*)^\top \left(\frac{\bar{A}^\top P+PA}{2}\right)(x_k-x^*)\\
			=\,&-\frac{1}{2} (x_k-x^*)^\top Q(x_k-x^*)\\
			\leq \,&-\frac{c_2'}{2}(x_k-x^*)^\top P(x_k-x^*)\tag{$Q\geq c_2'P$}\\
			= \,&-c_2'M_E(x_k-x^*).
		\end{align*}
		It follows that 
		\begin{align*}
			T_1=\alpha_k\mathbb{E}[\nabla M_E(x_k-x^*)^\top \bar{G}(x_k)]\leq -\alpha_k c_2'\mathbb{E}[M_E(x_k-x^*)].
		\end{align*}
	\end{proof}

	The following lemma will be useful in controlling the terms $T_2$ and $T_3$.
	\begin{lemma}\label{le:Lipschitz_lsa}
		For any real-valued matrix $A\in\mathbb{R}^{d\times d}$ such that $E$ is an invariant subspace of $A$, i.e., $x\in E \Rightarrow Ax\in E$, we have $p(Ax)\leq \|A\|_cp(x)$ for all $x\in\mathbb{R}^d$, where $\|A\|_c:=\max_{x:\|x\|_c=1}\|Ax\|_c$. As a result, the following two statements hold.
		\begin{enumerate}[(1)] 
			\item For any $x\in\mathbb{R}^d$ and $y\in\mathcal{Y}$, we have $p(A(y)x)\leq L_1p(x)$ for all $x\in\mathbb{R}^d$, which also implies $p(G(x,y))\leq L_1p(x)+L_2$.
			\item For any $x\in\mathbb{R}^d$, we have $p(\bar{A}x)\leq L_1p(x)$, which also implies $p(\bar{G}(x))\leq L_1p(x)+L_2$.
		\end{enumerate}
	\end{lemma}
	\begin{proof}[Proof of Lemma \ref{le:Lipschitz_lsa}]
		For any $x\in\mathbb{R}^d$, we have
		\begin{align*}
			p(Ax)=\,&\min_{y\in E}\|Ax-y\|_c\\
			\leq \,&\min_{z\in E}\|Ax-Az\|_c\tag{This follows from $z\in E\Rightarrow Az\in E$}\\
			\leq \,&\|A\|_c\min_{z\in E}\|x-z\|_c\\
			=\,&\|A\|_cp(x).
		\end{align*}
		Statements (1) and (2) follow from the above result and the definitions of $G(x,y)$ and $\bar{G}(x)$.
	\end{proof}

	Next, we bound the term $T_2$ in the following lemma.
	
	\begin{lemma}\label{le:LSA_T2}
		It holds for all $k\geq t_k$ that
		\begin{align*}
			T_2\leq 112\alpha_k\alpha_{k-t_k,k-1}L_1^2\mathbb{E}[M_E(x_k-x^*)]+56\alpha_k\alpha_{k-t_k,k-1}(L_1p(x^*)+L_2)^2.
		\end{align*}
	\end{lemma}
	\begin{proof}[Proof of Lemma \ref{le:LSA_T2}]
		For any $k\geq 0$, we have
		\begin{align}\label{eq:thm:lsa_proof_T2}
			&\mathbb{E}[\nabla M_E(x_k-x^*)^\top (G(x_k,Y_k)-\bar{G}(x_k))]\nonumber\\
			=\,&\mathbb{E}[(x_k-x^*)^\top P (G(x_k,Y_k)-\bar{G}(x_k))]\nonumber\\
			=\,&\underbrace{\mathbb{E}[(x_{k-t_k}-x^*)^\top P (G(x_{k-t_k},Y_k)-\bar{G}(x_{k-t_k}))]}_{:=T_{2,1}}\nonumber\\
			&+\underbrace{\mathbb{E}[(x_{k-t_k}-x^*)^\top P (G(x_k,Y_k)-G(x_{k-t_k},Y_k)+\bar{G}(x_{k-t_k})-\bar{G}(x_k))]}_{T_{2,2}}\nonumber\\
			&+\underbrace{\mathbb{E}[(x_k-x_{k-t_k})^\top P (G(x_k,Y_k)-\bar{G}(x_k))]}_{T_{2,3}}.
		\end{align}
		To control the terms $T_{2,1}$, $T_{2,2}$, and $T_{2,3}$, we require the following two lemmas.
		
		\begin{lemma}\label{le:difference_lsa}
			Let $k_1,k_2$ be non-negative integers 
			satisfying $k_1<k_2$ and $\alpha_{k_1,k_2-1}\leq 1/(4L_1)$. Then, we have for all $k\in [k_1,k_2]$ that
			\begin{align*}
				p(x_k-x_{k_1})\leq\,& 2\alpha_{k_1,k_2-1}(L_1p(x_{k_1})+L_2)\leq  \frac{1}{2}(p(x_{k_1})+L_2/L_1),\\
				p(x_k-x_{k_1})\leq\,& 4\alpha_{k_1,k_2-1}(L_1p(x_{k_2})+L_2)\leq p(x_{k_2})+L_2/L_1.
			\end{align*}
		\end{lemma}
		The proof of Lemma \ref{le:difference_lsa} is identical to that of Lemma \ref{le:difference}, and is therefore omitted.

		Now, we proceed to control the terms $T_{2,1}$, $T_{2,2}$, and $T_{2,3}$ from Eq. (\ref{eq:thm:lsa_proof_T2}) in the following. For the term $T_{2,1}$, we have
		\begin{align*}
			T_{2,1}=\,&\mathbb{E}[(x_{k-t_k}-x^*)^\top P (G(x_{k-t_k},Y_k)-\bar{G}(x_{k-t_k}))]\\
			=\,&\mathbb{E}[(x_{k-t_k}-x^*)^\top P (\mathbb{E}[G(x_{k-t_k},Y_k)\mid \mathcal{F}_{k-t_k}]-\bar{G}(x_{k-t_k}))]\\
			\leq \,&\mathbb{E}[p(x_{k-t_k}-x^*) p(\mathbb{E}[G(x_{k-t_k},Y_k)\mid \mathcal{F}_{k-t_k}]-\bar{G}(x_{k-t_k}))]\tag{Cauchy–Schwarz inequality}\\
			= \,&\mathbb{E}[p(x_{k-t_k}-x^*) p(\left(\mathbb{E}[A(Y_k)\mid \mathcal{F}_{k-t_k}]-\bar{A}\right)x_{k-t_k}+\mathbb{E}[b(Y_k)\mid \mathcal{F}_{k-t_k}]-\bar{b})]\\
			\leq \,&\mathbb{E}[p(x_{k-t_k}-x^*) (\|\mathbb{E}[A(Y_k)\mid \mathcal{F}_{k-t_k}]-\bar{A}\|_cp(x_{k-t_k})+\|\mathbb{E}[b(Y_k)\mid \mathcal{F}_{k-t_k}]-\bar{b}\|_c)]\tag{Lemma \ref{le:Lipschitz_lsa}}\\
			\leq \,&\alpha_k\mathbb{E}[p(x_{k-t_k}-x^*) (L_1p(x_{k-t_k})+L_2)]\tag{Assumption \ref{as:lsa_all} (4)}\\
			\leq \,&\alpha_k\mathbb{E}[(p(x_{k-t_k}-x_k)+p(x_k-x^*)) (L_1p(x_{k-t_k}-x_k)+L_1p(x_k)+L_2)]\\
			\leq \,&2\alpha_k\mathbb{E}[(p(x_k)+L_2/L_1+p(x_k-x^*)) (L_1p(x_k)+L_2)]\tag{Lemma \ref{le:difference_lsa}}\\
			\leq \,&2\alpha_k\mathbb{E}[(L_1p(x_k)+L_2+L_1p(x_k-x^*)) (L_1p(x_k)+L_2)]\tag{$L_1\geq 1$}\\
			\leq \,&2\alpha_k\mathbb{E}[(2L_1p(x_k-x^*)+L_1p(x^*)+L_2) (L_1p(x_k-x^*)+L_1p(x^*)+L_2)]\\
			\leq \,&4\alpha_k\mathbb{E}[(L_1p(x_k-x^*)+L_1p(x^*)+L_2)^2]\\
			\leq \,&8\alpha_kL_1^2\mathbb{E}[p^2(x_k-x^*)]+8\alpha_k(L_1p(x^*)+L_2)^2\tag{$(a+b)^2\leq 2(a^2+b^2)$}\\
			\leq \,&16\alpha_kL_1^2\mathbb{E}[M_E(x_k-x^*)]+8\alpha_k(L_1p(x^*)+L_2)^2,
		\end{align*}
		where the last line follows from $M_E(x)=p^2(x)/2$. 
		
		For the term $T_{2,2}$, we have
		\begin{align*}
			T_{2,2}=\,&\mathbb{E}[(x_{k-t_k}-x^*)^\top P (G(x_k,Y_k)-G(x_{k-t_k},Y_k)+\bar{G}(x_{k-t_k})-\bar{G}(x_k))]\\
			\leq \,&\mathbb{E}[(x_{k-t_k}-x^*)^\top P (A(Y_k)-\bar{A})(x_k-x_{k-t_k})]\\
			\leq \,&\mathbb{E}[p(x_{k-t_k}-x^*) p((A(Y_k)-\bar{A})(x_k-x_{k-t_k}))]\tag{Cauchy–Schwarz inequality}\\
			\leq \,&2L_1\mathbb{E}[p(x_{k-t_k}-x^*)p(x_k-x_{k-t_k})]\tag{Lemma \ref{le:Lipschitz_lsa}}\\
			\leq \,&2L_1\mathbb{E}[(p(x_k-x_{k-t_k})+p(x_k-x^*))p(x_k-x_{k-t_k})]\tag{Triangle inequality}\\
			\leq \,&8\alpha_{k-t_k,k-1}L_1\mathbb{E}[(p(x_k)+L_2/L_1+p(x_k-x^*))(L_1p(x_k)+L_2)]\tag{Lemma \ref{le:difference_lsa}}\\
			\leq \,&8\alpha_{k-t_k,k-1}\mathbb{E}[(2L_1p(x_k-x^*)+L_1p(x^*)+L_2)(L_1p(x_k-x^*)+L_1p(x^*)+L_2)]\\
			\leq \,&16\alpha_{k-t_k,k-1}\mathbb{E}[(L_1p(x_k-x^*)+L_1p(x^*)+L_2)^2]\\
			\leq \,&64\alpha_{k-t_k,k-1}L_1^2\mathbb{E}[M_E(x_k-x^*)]+32\alpha_{k-t_k,k-1}(L_1p(x^*)+L_2)^2.
		\end{align*}
		For the term $T_{2,3}$, we have
		\begin{align*}
			T_{2,3}=\,&\mathbb{E}[(x_k-x_{k-t_k})^\top P (G(x_k,Y_k)-\bar{G}(x_k))]\\
			\leq \,&\mathbb{E}[p(x_k-x_{k-t_k})p (G(x_k,Y_k)-\bar{G}(x_k))]\tag{Cauchy–Schwarz inequality}\\
			\leq \,&\mathbb{E}[p(x_k-x_{k-t_k})(p (G(x_k,Y_k))+p(\bar{G}(x_k)))]\tag{Triangle inequality}\\
			\leq \,&4\alpha_{k-t_k,k-1}\mathbb{E}[(L_1p(x_k)+L_2)(p (G(x_k,Y_k))+p(\bar{G}(x_k)))]\tag{Lemma \ref{le:difference_lsa}}\\
			\leq \,&8\alpha_{k-t_k,k-1}\mathbb{E}[(L_1p(x_k)+L_2)^2]\tag{Lemma \ref{le:Lipschitz_lsa}}\\
			\leq \,&8\alpha_{k-t_k,k-1}\mathbb{E}[(L_1p(x_k-x^*)+L_1p(x^*)+L_2)^2]\\
			\leq \,&32\alpha_{k-t_k,k-1}L_1^2\mathbb{E}[M_E(x_k-x^*)]+16\alpha_{k-t_k,k-1}(L_1p(x^*)+L_2)^2.
		\end{align*}
		Combining the previous three inequalities together, we obtain
		\begin{align*}
			T_2=\,&\alpha_k \mathbb{E}[\nabla M_E(x_k-x^*)^\top (G(x_k,Y_k)-\bar{G}(x_k))]\\
			=\,&\alpha_k (T_{2,1}+T_{2,2}+T_{2,3})\\
			\leq \,&112\alpha_k\alpha_{k-t_k,k-1}L_1^2\mathbb{E}[M_E(x_k-x^*)]+56\alpha_k\alpha_{k-t_k,k-1}(L_1p(x^*)+L_2)^2.
		\end{align*}
	\end{proof}
	
	Next, we bound the term $T_3$ in the following lemma.
	
	\begin{lemma}\label{le:LSA_T3}
		It holds for all $k\geq 0$ that
		\begin{align*}
			T_3 \leq 2\alpha_k^2L_1^2\mathbb{E}[M_E(x_k-x^*)]+\alpha_k^2(L_1p(x^*)+L_2)^2.
		\end{align*}
	\end{lemma}
	\begin{proof}[Proof of Lemma \ref{le:LSA_T3}]
		For any $k\geq 0$, we have
		\begin{align*}
			p^2\left(G(x_k,Y_k)\right)
			\leq \;& \left(p(A(Y_k)x_k)+p(b(Y_k))\right)^2\\
			\leq \,&\left(\|A(Y_k)\|_cp(x_k)+\|b(Y_k)\|_c\right)^2\tag{Lemma \ref{le:Lipschitz_lsa}}\\
			\leq \,&\left(L_1p(x_k)+L_2\right)^2\tag{Assumption \ref{as:lsa_all} (2)}\\
			\leq \,&\left(L_1p(x_k-x^*)+L_1p(x^*)+L_2\right)^2\\
			\leq \,&2L_1^2p^2(x_k-x^*)+2(L_1p(x^*)+L_2)^2\tag{$2(a^2+b^2)\geq (a+b)^2$}\\
			\leq \,&4L_1^2M_E(x_k-x^*)+2(L_1p(x^*)+L_2)^2.\tag{$M_E(x)=\frac{1}{2}p^2(x)$}
		\end{align*}
		It follows that 
		\begin{align*}
			T_3=\frac{\alpha_k^2}{2}\mathbb{E}[p^2(G(x_k,Y_k)]\leq 2\alpha_k^2L_1^2\mathbb{E}[M_E(x_k-x^*)]+\alpha_k^2(L_1p(x^*)+L_2)^2.
		\end{align*}
	\end{proof}

	Now that we have control for all the terms on the right-hand side of Eq. (\ref{eq:Thm:lsa_proof_decomposition}), combining the bounds altogether, we obtain for all $k\geq t_k$ that
	\begin{align*}
		\mathbb{E}[M_E(x_{k+1}-x^*)]\leq \,&\left(1-c_2'\alpha_k+114\alpha_k\alpha_{k-t_k,k-1}L_1^2\right)\mathbb{E}[M_E(x_k-x^*)]\\
		&+57\alpha_k\alpha_{k-t_k,k-1}(L_1p(x^*)+L_2)^2\\
		\leq \,&\left(1-\frac{c_2'\alpha_k}{2}\right)\mathbb{E}[M_E(x_k-x^*)]+57\alpha_k\alpha_{k-t_k,k-1}(L_1p(x^*)+L_2)^2\tag{$\alpha_{k-t_k,k-1}\leq \frac{c_2'}{228L_1^2}$}.
	\end{align*}
	Repeatedly using the previous inequality, we have for all $k\geq K$ that
	\begin{align*}
		\mathbb{E}[p^2(x_k-x^*)]\leq\,& \prod_{j=K}^{k-1}\left(1-\frac{c_2'\alpha_k}{2}\right)p^2(x_K-x^*)\\
		&+114(L_1p(x^*)+L_2)^2\sum_{i=K}^{k-1}\alpha_i\alpha_{i-t_i,i-1}\prod_{j=i+1}^{k-1}\left(1-\frac{\alpha_j\lambda_{\min}(Q)}{2\lambda_{\max}(P)}\right)\\
		\leq \,&c_1'\prod_{j=K}^{k-1}\left(1-c_2'\alpha_j\right)+c_3'\sum_{i=K}^{k-1}\alpha_i\alpha_{i-t_i,i-1}\prod_{j=i+1}^{k-1}\left(1-c_2'\alpha_j\right),
	\end{align*}
	where the last inequality follows from 
	\begin{align*}
		p^2(x_K-x^*)\leq (p(x_K-x_0)+p(x_0-x^*))^2
		\leq(p(x_0)+L_2/L_1+p(x_0-x^*))^2
		=c_1',
	\end{align*}
	and $c_3'=114(L_1p(x^*)+L_2)^2$.

	\subsubsection{A Alternative Approach to Prove Theorem \ref{thm:linear_SA}}\label{ap:alternative}
	Another way to prove Theorem \ref{thm:linear_SA} is to reformulate its update equation as a seminorm-contractive SA algorithm described in Eq.~\eqref{algo:SA}, and then verify the assumptions needed to apply Theorem \ref{thm:SA_finite}. Specifically, for any $\eta > 0$, Eq.~\eqref{algo:linear_SA} is equivalent to  
	\begin{align}\label{eq:lsa_reform}  
		x_{k+1} = x_k + \beta_k \big(F(x_k, Y_k) - x_k\big),  
	\end{align}  
	where $F(x, y) = (\eta A(y) + I)x + \eta b(y)$ for any $x \in \mathbb{R}^d, y \in \mathcal{Y}$, and $\beta_k = \alpha_k / \eta$. All the assumptions of Theorem \ref{thm:SA_finite} can be easily verified, except for the requirement that $\bar{F}(\cdot)$ be a seminorm contraction mapping, which we focus on next.
	
	Let $\text{spec}^-(\bar{A})$ denote the set of eigenvalues of $\bar{A}$ with strictly negative real parts, and let $\text{spec}^+(\bar{A}) = \text{spec}(\bar{A}) \setminus \text{spec}^-(\bar{A})$. Let $E$ be the linear subspace of $\mathbb{R}^d$ spanned by all the generalized eigenvectors of $\bar{A}$ corresponding to eigenvalues in $\text{spec}^+(\bar{A})$. The following lemma leverages Theorem~\ref{theorem:Seminorm GAS and Lyapunov Equation} to explicitly construct a seminorm $p_{c,E}(\cdot)$ with kernel space $E$, such that $\bar{F}(\cdot)$ is a contraction mapping with respect to $p_{c,E}(\cdot)$.

	\begin{lemma}\label{le:lsa_contraction}
		Let $\eta$ be chosen such that $|\eta \lambda+1|<1$ for any $\lambda\in \text{spec}^-(\bar{A})$ and let $Q\in\mathbb{R}^{d\times d}$ be such that $x^\top Qx = \min_{y \in E} \|x - y\|_2^2$. Then, there exists a unique $P\in \mathcal{S}^{d}_{+,E}$ such that 
		\begin{align}\label{eq:lyap_eq}
			(\eta \bar{A}+I)^\top P(\eta \bar{A}+I)-P+Q=0,
		\end{align}
		where we recall that $\mathcal{S}^d_{+,E}$ denotes the set of positive semi-definite matrices whose null space is precisely $E$.
		Moreover, letting $p_{c,E}:\mathbb{R}^d\to \mathbb{R}$ be a seminorm defined as $p_{c,E}(x)=\sqrt{x^\top Px}$, we have
		\begin{align*}
			p_{c,E}(\bar{F}(x_1)-\bar{F}(x_2))\leq \sqrt{1-1/\lambda_{\max}^2(P)}p_{c,E}(x_1-x_2),\quad \forall\,x_1,x_2\in\mathbb{R}^d,
		\end{align*}
		where $\lambda_{\max}(P)$ denotes the maximum eigenvalue of $P$. 
	\end{lemma}
	\begin{proof}[Proof of Lemma \ref{le:lsa_contraction}]
		The Lyapunov equation is a direct consequence of Theorem \ref{theorem:Seminorm GAS and Lyapunov Equation}.
		For the second part of the lemma, for any $x_1,x_2\in\mathbb{R}^d$, we have
		\begin{align*}
			p_{c,E}(\bar{F}(x_1)-\bar{F}(x_2))^2=\,&(x_1-x_2)^\top (\eta \bar{A}+I)^\top P(\eta \bar{A}+I) (x_1-x_2)\\
			=\,&(x_1-x_2)^\top (P-Q) (x_1-x_2)\tag{Eq. (\ref{eq:lyap_eq})}\\
			=\,&p_{c,E}(x_1-x_2)^2-p_{2,E}(x_1-x_2)^2,
		\end{align*}
		where $p_{2,E}(x)=\min_{y\in E}\|x-y\|_2^2=\sqrt{x^\top Qx}$.
		Since all seminorms sharing the same kernel space are equivalent (cf. Proposition \ref{prop:seminorm_properties} (3)), there exists $C_1\in (0,1)$ and $C_2\in (1,+\infty)$ such that $C_1p_{2,E}(\cdot)\leq p_{c,E}(\cdot)\leq C_2p_{2,E}(\cdot)$. In our case, since $p_{c,E}(\cdot)$ and $p_{2,E}(\cdot)$ are both defined in terms of positive semi-definite matrices, i.e., $P$ and $Q$, respectively, the constant $C_1$ is the minimum non-zero eigenvalue of $P$ and $C_2$ is the maximum eigenvalue of $P$. Therefore, we have
		\begin{align*}
			p_{c,E}(\bar{F}(x_1)-\bar{F}(x_2))^2\leq \,&p_{c,E}^2(x_1-x_2)-p_{2,E}(x_1-x_2)^2\\
			= \,&\left(1-\frac{1}{\lambda_{\max}(P)^2}\right)p_{c,E}(x_1-x_2)^2.
		\end{align*}
	\end{proof}
	
	With Lemma \ref{le:lsa_contraction} at hand, we can apply Theorem \ref{thm:SA_finite} to obtain the finite-sample bounds of the linear SA in Eq. (\ref{algo:linear_SA}). The results will be identical to the ones we proved in the previous section, modulo constants.

	\section{Supplementary Results for Section \ref{sec:average-reward-RL}}
	\subsection{Proof of Lemma \ref{le:TD fixed points}}\label{pf:le:TD fixed points} 
	We consider two cases: $e \not \in W_{\Phi}$ and $e \not \notin W_{\Phi}$.
	\textbf{Case 1.} 
	Suppose that $ e \not\in W_{\Phi} =\{\Phi\theta\mid \theta\in\mathbb{R}^d\}$. Then, we have $ S_{\Phi, e} = \{0\} $ (cf. Eq. (\ref{def:theta_e})), and $ E_{\Phi, e} $, as the orthogonal complement of $ S_{\Phi, e} $, is $\mathbb{R}^d$. This implies that $ W_{E_{\Phi, e}} = \{\Phi \theta \mid \theta \in E_{\Phi, e}\} = W_{\Phi} $. By Theorem 1 in \cite{tsitsiklis1999average}, we know that the projected Bellman equation (\ref{projected Bellman equation}) has a unique fixed point $\theta^*$. Thus, $\mathcal{L} = \{\theta^*\}$.
	
	\textbf{Case 2.} Suppose that $e \in W_{\Phi} =\{\Phi\theta\mid \theta\in\mathbb{R}^d\}$. In this case, we have denoted $\theta_e$ as the unique solution to $\Phi \theta_e = e$, where the uniqueness follows from $\Phi$ being full column rank. Observe that, using the explicit definition of $\mathcal{T}^{(\lambda)}(\cdot)$, the projected Bellman equation (\ref{projected Bellman equation}) is equivalent to 
	\begin{align}
		0 =\,&\Phi^\top D \left(-\frac{r(\pi)}{1-\lambda}e+\mathcal{R}^{(\lambda)}+P^{(\lambda)} \Phi \theta -\Phi\theta\right)\nonumber\\
		=\,&\Phi^\top D(P^{(\lambda)}-I)\Phi\theta+\Phi^\top D\mathcal{R}^{(\lambda)}-\frac{r(\pi)}{1-\lambda}\Phi^\top \mu,\label{equation:PBE_Equivalent}
	\end{align}
	where $\mathcal{R}^{(\lambda)}:=\sum_{m=0}^\infty(\lambda P^\pi)^m\mathcal{R}^\pi$.
	Next, we characterize the kernel space and the image space of the matrix $\Phi^\top D(P^{(\lambda)}-I)\Phi$. 
	
	\begin{itemize}
		\item \textbf{The Kernel Space of $\Phi^\top D(P^{(\lambda)}-I)\Phi$:} On the one hand, for any $\theta \in S_{\Phi,e}$, there exists some constant $c\in\mathbb{R}$ such that $\theta  =c\theta_e$ (cf. Eq. (\ref{def:theta_e})). Therefore, we have
		\begin{align*}
			\Phi^\top D(P^{(\lambda)}-I)\Phi \theta=\,&c \Phi^\top D(P^{(\lambda)}-I)\Phi  \theta_e\\
			=\,&c \Phi^\top D(P^{(\lambda)}-I)e\\
			=\,&0,
		\end{align*}
		which implies that $S_{\Phi,e}$ is a subset of the kernel space of $\Phi^\top D(P^{(\lambda)}-I)\Phi$. On the other hand, since $P^{(\lambda)}$ is irreducible and aperiodic (which trivially follows from the definition of $P^{(\lambda)}$ and Assumption \ref{assumption:ergodic MC}), we have $\theta^\top \Phi^\top  D(P^{(\lambda)}-I)\Phi \theta<0$ for any $\theta\notin S_{\Phi,e}$ \cite[Lemma 7]{tsitsiklis1999average}. Therefore, the kernel space of $\Phi^\top D(P^{(\lambda)}-I)\Phi$ is exactly $S_{\Phi,e}$.
		\item \textbf{The Image Space of $\Phi^\top D(P^{(\lambda)}-I)\Phi$:} Since $S_{\Phi,e}$ (which is the kernel space of $\Phi^\top D(P^{(\lambda)}-I)\Phi$) is a one-dimensional linear subspace of $\mathbb{R}^d$, by the Rank–nullity theorem, the dimension of the image space of $\Phi^\top D(P^{(\lambda)}-I)\Phi$ must be $d-1$. Since for any $\theta\in\mathbb{R}^d$, we have
		\begin{align*}
			\langle \theta_e, \Phi^\top D(P^{(\lambda)}-I)\Phi\theta \rangle=\,&\theta_e^\top  \Phi^\top D(P^{(\lambda)}-I)\Phi\theta\\
			=\,&e^\top D(P^{(\lambda)}-I)\Phi\theta\\
			=\,&\mu^\top (P^{(\lambda)}-I)\Phi\theta\\
			=\,&(\mu^\top -\mu^\top)\Phi\theta\tag{$\mu^\top P^{\lambda}=\mu^\top$ }\\
			=\,&0,
		\end{align*}
		the image space of $\Phi^\top D(P^{(\lambda)}-I)\Phi$ is orthogonal to $S_{\Phi,e}$. Therefore, the image space of $\Phi^\top D(P^{(\lambda)}-I)\Phi$ must be the orthogonal complement of $S_{\Phi, e}$, i.e, $E_{\Phi, e}$.
	\end{itemize}
	
	To show that Eq. (\ref{equation:PBE_Equivalent}) has a solution, it is enough to show that the vector $\Phi^\top D\mathcal{R}^{(\lambda)}-\bar{r}\Phi^\top \mu/(1-\lambda)$ belongs to the image space of $\Phi^\top D(P^{(\lambda)}-I)\Phi$, or equivalently, the vector $\Phi^\top D\mathcal{R}^{(\lambda)}-\bar{r}\Phi^\top \mu/(1-\lambda)$ is orthogonal to $\theta_e$, which spans $S_{\Phi,e}$. This follows by observing that
	\begin{align*}
		\left\langle \theta_e,\Phi^\top D\mathcal{R}^{(\lambda)}-\frac{\bar{r}}{1-\lambda}\Phi^\top \mu\right\rangle=\mu^\top \mathcal{R}^{(\lambda)}-\frac{r(\pi)}{1-\lambda}=0.
	\end{align*}
	
	To this end, we have shown that the set of solutions to Eq. (\ref{equation:PBE_Equivalent}) must be $\mathcal{L}_{\Phi,e}=\tilde{\theta}+S_{\Phi,e}$, where $\tilde{\theta}$ is a particular solution to Eq. (\ref{equation:PBE_Equivalent}) and $S_{\Phi,e}$ is the kernel space of $\Phi^\top D(P^{(\lambda)}-I)\Phi$. It remains to show that the equation
	\begin{align}\label{equation:PBE_second}
		\Phi \theta = \Pi_{D,W_{E_{\Phi, e}}} \mathcal{T}^{(\lambda)} \left(\Phi \theta\right)
	\end{align}
	has a unique solution $\theta^*$ and $\theta^*\in\mathcal{L}_{\Phi,e}$,
	where $\Pi_{D,W_{E_{\Phi, e}}}$ denotes the projection operator onto the linear subspace $W_{E_{\Phi, e}}=\{\Phi\theta \mid \theta\in E_{\Phi,e}=S_{\Phi,e}^\perp\}$ with respect to the weighted $\ell_2$-norm $\|\cdot\|_D$. 
	\begin{itemize}
		\item \textbf{Existence:} Let $\tilde{\theta}=\theta^*+(\theta^*)^\perp$, where $\theta^*\in E_{\Phi,e}=S_{\Phi,e}^\perp$ and $(\theta^*)^\perp \in S_{\Phi,e}$. It is clear that $\theta^*$ solves Eq. (\ref{equation:PBE_Equivalent}) because $\tilde{\theta}-\theta^*\in S_{\Phi,e}$. Moreover, since $\theta^*\in E_{\Phi,e}$, we have $\Phi\theta^* \in W_{E_{\Phi,e}}$. Combining these two observations with the fact that $W_{E_{\Phi,e}}\subseteq W_{\Phi,e}$, we have
		\begin{align*}
			\Pi_{D,W_{E_{\Phi, e}}}  \mathcal{T}^{(\lambda)} \left(\Phi \theta^*\right)=\,&\Pi_{D,W_{E_{\Phi, e}}} \Pi_{D,W_\Phi} \mathcal{T}^{(\lambda)} \left(\Phi \theta^*\right)\\
			=\,&\Pi_{D,W_{E_{\Phi, e}}}\Phi \theta^*\\
			=\,&\Phi \theta^*,
		\end{align*}
		which implies that $\theta^*$ is a solution to Eq. (\ref{equation:PBE_second}).
		
		\item \textbf{Uniqueness:} Since a solution to Eq. (\ref{equation:PBE_second}) must be in $W_{E_{\Phi,e}}=\{\Phi\theta \mid \theta\in E_{\Phi,e}=S_{\Phi,e}^\perp\}=\{\Phi\tilde{\Pi}\theta \mid \theta\in \mathbb{R}^d\}$, where $\tilde{\Pi}$ denotes the projection matrix onto $E_{\Phi,e}$ with respect to the $\ell_2$-norm $\|\cdot\|_2$, Eq. (\ref{equation:PBE_second}) can be equivalently written as
		\begin{align}\label{equation:PBE_third}
			\tilde{\Pi}^\top \Phi^\top D(P^{(\lambda)}-I)\Phi\tilde{\Pi}\theta+\tilde{\Pi}^\top \Phi^\top D\mathcal{R}^{(\lambda)}-\frac{r(\pi)}{1-\lambda}\tilde{\Pi}^\top \Phi^\top \mu=0.
		\end{align}
		We have shown that Eq. (\ref{equation:PBE_third}) (or Eq. (\ref{equation:PBE_second})) has a solution $\theta^*$. Therefore, to verify that $\theta^*$ is indeed the unique solution, it is enough to show that the kernel space of $\tilde{\Pi}^\top \Phi^\top D(P^{(\lambda)}-I)\Phi \tilde{\Pi}$ is contained in $S_{\Phi,e}$. Suppose that this is not true, i.e., there exists $\theta\notin S_{\Phi,e}$ such that
		\begin{align*}
			\theta^\top \tilde{\Pi}^\top \Phi^\top D(P^{(\lambda)}-I)\Phi\tilde{\Pi}\theta=0.
		\end{align*}
		Since $P^{(\lambda)}$ is irreducible and aperiodic, we have $\theta^\top \tilde{\Pi}^\top \Phi^\top D(P^{(\lambda)}-I)\Phi\tilde{\Pi}\theta=0$ only if $\tilde{\Pi}\theta \in  S_{\Phi,e} $ \cite[Lemma 7]{tsitsiklis1999average}. However, we know that $\tilde{\Pi}\theta \in E_{\Phi,e}=S_{\Phi,e}^\perp$, which implies $\tilde{\Pi} \theta=0$, i.e., $\theta \in S_{\Phi,e}$. This is a contradiction. Therefore, $\theta^*$ is the unique solution to Eq. (\ref{equation:PBE_second}).
	\end{itemize}
	
	\subsection{Proof of Lemma \ref{le:TDlambda_assumptions}}\label{pf:le:TDlambda_assumptions}
	\begin{enumerate}
		\item [(1)] For any $\Theta=[r,\theta^\top ]^\top \in E=\{0\}\times S_{\Phi,e}$, we have $r=0$ and $\phi(s)^\top \theta=\phi(s')^\top \theta$ for all $s,s'\in\mathcal{S}$ (cf. Eq. (\ref{def:theta_e})). Therefore, using the explicit expression of $A(y)$, we have
		\begin{align*}
			A(y)\Theta=\begin{bmatrix}
				0\\
				z(\phi(s')^\top -\phi(s)^\top )\theta
			\end{bmatrix}=0\in E.
		\end{align*}
		\item [(2)] Since $(\Phi \theta)^\top D (P^{(\lambda)} - I )\Phi \theta<0$ \cite[Lemma 7]{tsitsiklis1999average} for any $\theta \notin S_{\Phi,e}$, we have by Weierstrass extreme value theorem that
		\begin{align*}
			\Delta=\min_{\|\theta\|_2=1, \theta\in E_{\Phi, e}} \theta^\top\Phi^\top D ( I - P^{(\lambda)} )\Phi\theta>0.
		\end{align*}
		Therefore, for any $\Theta=[r,\theta^\top]^\top  \in E^\perp$, we have
		\begin{align*}
			\Theta^\top \bar{A}\Theta=\,&-c_{\alpha}r^2-\frac{r}{1-\lambda}\theta^\top\Phi^\top De+\theta^\top \Phi^\top D(  P^{(\lambda)} -I)\Phi\theta\\
			\leq \,&-c_{\alpha}r^2-\Delta\|\theta\|_2^2+\frac{r}{1-\lambda}\theta^\top\Phi^\top De.
		\end{align*}
		Observe that 
		\begin{align*}
			\left|\frac{r}{1-\lambda}\theta^\top \Phi^\top De\right|
			&\leq \frac{|r|}{1-\lambda}\|\Phi\theta\|_{\infty}\|\mu\|_1\\
			&= \frac{|r|}{1-\lambda}\|\Phi\theta\|_{\infty}\\
			&\leq \frac{|r|}{1-\lambda}\max_{s\in \mathcal{S}}\|\phi(s)\|_2\|\theta\|_2\\
			&\leq \frac{|r|}{1-\lambda}\|\theta\|_2,
		\end{align*}
		where the last inequality follows from feature normalization $\max_{s\in\mathcal{S}}\|\phi(s)\|_2\leq 1$. It follows that
		\begin{align}
			\Theta^\top \bar{A}\Theta
			\leq \,&-c_{\alpha}r^2-\Delta\|\theta\|_2^2+\frac{|r|}{1-\lambda}\|\theta\|_2\nonumber\\
			\leq \,&-c_{\alpha}r^2-\Delta\|\theta\|_2^2+\frac{r^2}{2\Delta(1-\lambda)^2}+\frac{\Delta}{2}\|\theta\|_2^2\tag{$a^2+b^2\geq 2ab$ for any $a,b\in\mathbb{R}$}\nonumber\\
			=\,&-\left(c_\alpha-\frac{1}{2\Delta(1-\lambda)^2}\right)r^2-\frac{\Delta}{2}\|\theta\|_2^2\tag{Choosing $c_\alpha\geq \frac{\Delta}{2}+\frac{1}{2\Delta(1-\lambda)^2}$}\nonumber\\
			\leq \,&-\frac{\Delta}{2}(r^2+\|\theta\|_2^2)\nonumber\\
			=\,&-\frac{\Delta}{2}\|\Theta\|_2^2\label{eq:TDLFA_Delta}
		\end{align}
		for all $\Theta=[r,\theta^\top]^\top  \in E^\perp$. 
		
		Recall that $P$ is the projection matrix onto $E^\perp=\mathbb{R}\times E_{\Phi,e}$ with respect to $\|\cdot\|_2$, i.e., $\min_{\Theta'\in E^\perp}\|\Theta-\Theta'\|_2^2=\|\Theta-P\Theta\|_2^2$. Moreover, the matrix $P$ is symmetric, idempotent, and positive semi-definite. For any $\Theta \in \mathbb{R}^{d+1}$, there exists a unique pair $\Theta_1\in E$ and $\Theta_2\in E^\perp$ such that $\Theta=\Theta_1+\Theta_2$. Therefore, we have
		\begin{align*}
			\Theta^\top (\bar{A}^\top P +P\bar{A})\Theta=\,&(\Theta_1+\Theta_2)^\top (\bar{A}^\top P +P\bar{A})(\Theta_1+\Theta_2)\\
			=\,&\Theta_2^\top (\bar{A}^\top P +P\bar{A})\Theta_2\tag{$\Theta_1 \in E$ $\Rightarrow$ $P\Theta_1=0$ and $\bar{A}\Theta_1=0$}\\
			=\,&2\Theta_2^\top \bar{A}\Theta_2\tag{$\Theta_2 \in E^\perp\Rightarrow P\Theta_2 =\Theta_2$}\\
			=\,&2\Theta^\top P^\top \bar{A}P\Theta\tag{$\Theta_2=P\Theta$}\\
			\leq\,& -\Delta\|P\Theta\|_2^2\tag{Eq. (\ref{eq:TDLFA_Delta})}\\
			=\,& -\Delta\Theta^\top P\Theta\tag{$P$ is idempotent}.
		\end{align*}
		It follows that
		\begin{align*}
			\bar{A}^\top P+P\bar{A}+\Delta P\leq 0.
		\end{align*}
		\item [(3)] Observe that $\Theta^\top P\Theta=\Theta^\top P^2\Theta=\|P\Theta\|_2^2=\min_{\Theta'\in E}\|\Theta-\Theta'\|_2^2$. Recall the definitions of $A(y)$ and $b(y)$ from Section \ref{subsubsec:TD_Algorithm}. For any $y=(s,s',z)\in\mathcal{Y}$, we have
		\begin{align*}
			\|A(y)\Theta\|_2^2=\,&c_\alpha^2r^2+\|-zr+z\left(\phi(s')^\top - \phi(s)^\top \right)\theta\|_2^2\\
			\leq \,&c_\alpha^2r^2+2r^2\|z\|_2^2+2\|z\left(\phi(s')^\top - \phi(s)^\top \right)\theta\|_2^2\tag{$(a+b)^2\leq 2(a^2+b^2)$}\\
			=\,&c_\alpha^2r^2+2\left(r^2+|(\phi(s')^\top - \phi(s)^\top)\theta|^2\right)\|z\|_2^2.
		\end{align*}
		Since $\|z\|_2\leq \sum_{k=0}^\infty \lambda^k\|\phi(s_k)\|_2\leq 1/(1-\lambda)$, we have
		\begin{align*}
			\|z\|_2^2\leq \frac{1}{(1-\lambda)^2}.
		\end{align*}
		Moreover, using Cauchy–Schwarz inequality, we have
		\begin{align*}
			|(\phi(s')^\top - \phi(s)^\top)\theta|\leq \,&|\phi(s')^\top\theta| + \phi(s)^\top\theta|\\
			\leq \,&(\|\phi(s')\|_2+\|\phi(s)\|_2)\|\theta\|_2\\
			\leq \,&2\|\theta\|_2.
		\end{align*}
		It follows that
		\begin{align*}
			\|A(y)\Theta\|_2^2\leq \,&c_\alpha^2r^2+2\left(r^2+4\|\theta\|_2^2\right)\frac{1}{(1-\lambda)^2}\\
			=\,&\left(c_\alpha^2+\frac{2}{(1-\lambda)^2}\right)r^2+\frac{4}{(1-\lambda)^2}\|\theta\|_2^2\\
			\leq \,&4c_\alpha^2\|\Theta\|_2^2,
		\end{align*}
		where the last inequality follows from $c_\alpha \geq \Delta/2+1/(2\Delta(1-\lambda)^2)\geq 1/(1-\lambda)$ and $\|\Theta\|_2^2=r^2+\|\theta\|_2^2$. The previous inequality implies $\|A(y)\Theta\|_2\leq 2c_\alpha$. Similarly, we have
		\begin{align*}
			\|b(y)\|_2^2=\,&(\mathcal{R}^\pi(s))^2(c_\alpha^2+\|z\|_2^2)\\
			\leq \,&c_\alpha^2+\frac{1}{(1-\lambda)^2}\\
			\leq \,&2c_\alpha^2,
		\end{align*}
		which implies $\|b(y)\|_2\leq 2c_\alpha$.
		\item [(4)] The proof essentially follows from Lemma 6.7 in \cite{bertsekas1996neuro} and therefore is omitted.
	\end{enumerate}

\end{document}